\pdfoutput=1
\documentclass[11pt]{article}

\usepackage[margin=1in]{geometry}
\usepackage{times}
\usepackage{microtype}

\usepackage{amsmath,amsfonts,bm}

\def\eqref#1{equation~\ref{#1}}

\def\1{\bm{1}}

\DeclareMathAlphabet{\mathsfit}{\encodingdefault}{\sfdefault}{m}{sl}
\SetMathAlphabet{\mathsfit}{bold}{\encodingdefault}{\sfdefault}{bx}{n}

\usepackage{amsthm}
\usepackage{amsmath}
\usepackage{amssymb}
\usepackage{float}
\usepackage{caption}
\usepackage{hyperref}
\usepackage{subcaption}
\usepackage[english]{babel}
\usepackage{booktabs}
\usepackage{bbm}
\usepackage{enumitem}

\usepackage{xcolor}

\usepackage{blindtext}
\usepackage{multicol}
\usepackage{multirow}
\usepackage{wrapfig}
\usepackage{adjustbox} 
\renewcommand{\eqref}[1]{(\ref{#1})}

\usepackage{enumitem}  
\setlist[itemize]{label=--}  

\usepackage{algorithmic}
\usepackage[linesnumbered,ruled,vlined]{algorithm2e}
\newcommand\mycommfont[1]{\footnotesize\ttfamily\textcolor{blue}{#1}}
\SetCommentSty{mycommfont}

\SetKwInput{KwInput}{Input}                
\SetKwInput{KwOutput}{Output}              

\usepackage{graphicx} 

\usepackage{graphicx} %

\newtheorem{theorem}{Theorem}[section]
\newtheorem{proposition}[theorem]{Proposition}

\newtheorem{remark}[theorem]{Remark}

\newtheorem{lemma}[theorem]{Lemma}

\usepackage{hyperref}
\usepackage{url}
\usepackage{natbib}

\title{\textbf{Fused Partial Gromov--Wasserstein for Structured Objects}}

\author{
  Yikun Bai\textsuperscript{1} \quad
  Shuang Wang\textsuperscript{2} \quad
  Huy Tran\textsuperscript{1} \quad
  Hengrong Du\textsuperscript{4} \quad
  Juexin Wang\textsuperscript{3} \quad
  Soheil Kolouri\textsuperscript{1} \\
  \\
  \textsuperscript{1}Department of Computer Science, Vanderbilt University \\
  \textsuperscript{2}Department of Computer Science, Indiana University–Indianapolis \\
\textsuperscript{3}Department of BioHealth Informatics, Indiana University–Indianapolis \\
\textsuperscript{4}Department of Mathematics and Computer Science, Fisk University \\
  \texttt{\{yikun.bai, huy.tran, soheil.kolouri\}@vanderbilt.edu} \\
  \texttt{\{sw152,wangjuex\}@iu.edu} \\
  \texttt{\{hdu@fisk.edu\}@fisk.edu}
}

\date{} 

\begin{document}
\maketitle
\begin{abstract}
Structured data, such as graphs, is vital in machine learning due to its capacity to capture complex relationships and interactions. In recent years, the Fused Gromov-Wasserstein (FGW) distance has attracted growing interest because it enables the comparison of structured data by jointly accounting for feature similarity and geometric structure. However, as a variant of optimal transport (OT), classical FGW assumes an equal mass constraint on the compared data. In this work, we relax this mass constraint and propose the Fused Partial Gromov-Wasserstein (FPGW) framework, which extends FGW to accommodate unbalanced data. Theoretically, we establish the relationship between FPGW and FGW and prove the metric properties of FPGW. Numerically, we introduce Frank-Wolfe solvers and Sinkhorn solvers for the proposed FPGW framework.  Finally, we evaluate the FPGW distance through graph matching, graph classification and graph clustering experiments, demonstrating its robust performance. 
\end{abstract}

\section{Introduction}
\begin{wrapfigure}{r}{0.5\textwidth}
\centering
\vspace{-0.2in}
    \includegraphics[width=\linewidth]{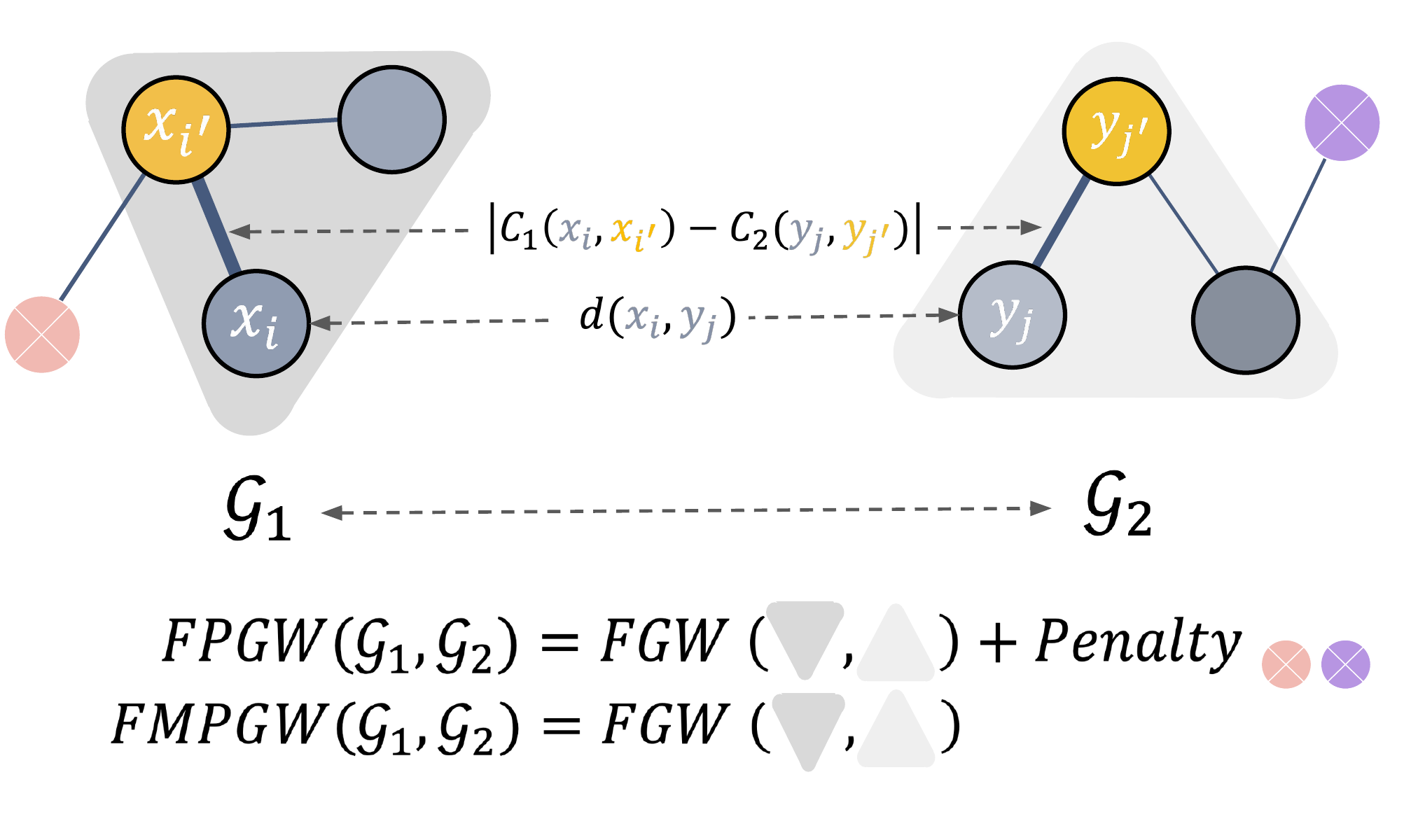}

    \caption{An intuitive understanding of the fused-PGW problems \eqref{eq:fpgw_orig}. The distance between node features is modeled as \( d(\cdot,\cdot) \). The structural information of each graph is represented by their shortest path distances, \( C_1(\cdot,\cdot) \) and \( C_2(\cdot,\cdot) \).}
    \label{fig:fig1}
\end{wrapfigure}
Analyzing structured data, which combines feature-based and relational information, is a longstanding challenge in machine learning, data science, and statistics. One classical type of structure-based data is the graph, where the nodes with attributes can model the data feature while the edges can describe the structure. Examples of such data structures are abundant, including molecular graphs for drug discovery \cite{ruddigkeit2012enumeration}, functional and structural brain networks \cite{bassett2017network}, and social network graphs \cite{hamilton2017inductive}. Beyond graphs, structured data encompasses a wide variety of domains, such as sequences \cite{graves2006connectionist}, hierarchical structures such as trees \cite{bille2008integrated}, and even pixel-based data, such as images \cite{wang2004integration}.  

In recent years, the Optimal Transport (OT) distance \cite{Villani2009Optimal} and its extensions, including unbalanced Optimal Transport \cite{chizat2018unbalanced,figalli2010optimal}, linear Optimal Transport \cite{wang2013linear,cai2022linearized,bai2023linear,martin2023lcot}, sliced optimal transport \cite{kolouri2019generalized,bonneel2015sliced,bai2022sliced}, and expected optimal transport \cite{rowland2019orthogonal} have been widely used in machine learning tasks due to their capacity to measure the similarity between datasets. Based on the classical OT,  Gromov-Wasserstein problem and its unbalanced extension \cite{memoli2011gromov,memoli2009spectral,sejourne2021unbalanced,chapel2020partial,bai2023partial,kong2024outlier} has been proposed, which can capture the inherent structure of the data. 

Classical OT can incorporate the features of data to measure the similarity and
Gromov-Wasserstein formulations capture the structure information. Inspired by
these works, fused Gromov-Wasserstein \cite{feydy2017optimal,vayer2020fused},
which can be treated as a ``linear combination'' of classical OT and
Gromov-Wasserstein has been proposed in recent years to analyze the structured
feature data. Despite its successful applications in graph data
analysis, similar to classical OT, fused GW formulation requires equal mass
constraint. To address this issue, in recent years, Fused Unbalanced Gromov
Wasserstein (FUGW) \cite{thual2022aligning,halmos2025dest} and related Sinkhorn solvers
have been proposed and applied in brain image analysis. However, FUGW
relies on the Sinkhorn solver, and its metric property is still unclear. To
address this limitation and complete the theoretical gap, in this paper, we
introduce the fused-partial Gromov-Wasserstein formulations: 

\begin{itemize}
    \item We introduce the fused partial Gromov Wasserstein formulations \eqref{eq:fpgw}, which allow the comparison of structured objects with unequal total mass. Theoretically, we demonstrate that the FPGW admits a (semi-)metric. 
    \item We propose the related Frank-Wolfe algorithms and Sinkhorn Algorithms to solve the FPGW problem. In addition, we present the FPGW barycenter and related computational solvers. 
    \item We applied FPGW in graph matching, clustering and classification experiments and demonstrated that the FPGW-based methods admit more robust performance. 
 \end{itemize}

\section{Background: Gromov-Wasserstein (GW) Problems}\label{subsec: GW}
\vspace{-0.5em}
Note, graph-structured data can be measured as a metric measure space (mm-space) consisting of a set $X$ endowed with a metric structure, that is, a notion of distance $d_X$ between its elements, and equipped with a Borel measure $\mu$.  
As in \cite[Ch. 5]{memoli2011gromov}, we will assume that $X$ is compact and that  $\operatorname{supp}(\mu)=X$.
Given two probability mm-spaces $\mathbb{X}=(X,d_X,\mu)$, $\mathbb{Y}=(Y,d_Y,\nu)$, with $\mu\in \mathcal{P}(X)$ and $\nu\in\mathcal{P}(Y)$, and a non-negative lower semi-continuous cost function  $L: \mathbb{R}^2\to \mathbb{R}_+$  (e.g., the Euclidean distance or the KL-loss),  the Gromov-Wasserstein (GW) matching problem is defined as:
\begin{equation}
GW_{r,L}(\mathbb{X},\mathbb{Y}):=\inf_{\gamma\in \Gamma(\mu,\nu)}\gamma^{\otimes 2}(L(d_X^r(\cdot,\cdot),d_Y^r(\cdot,\cdot))), 
\label{eq:gw}
\end{equation}
where $r\ge 1$ and 
\begin{align}
\Gamma(\mu,\nu):=\{\gamma\in\mathcal{P}(X\times Y):\gamma_1=\mu,\gamma_2=\nu\}\label{eq:Gamma_=}.
\end{align}
For brevity, we employ the notation $\gamma^{\otimes 2}$ for the product measure $d\gamma^{\otimes 2}((x,y),(x',y'))=d\gamma(x,y)d\gamma(x',y')$. 
If 
$L(a,b)=|a-b|^q$, 
for $1\leq q< \infty$, 
we denote $GW_{r,L}(\cdot,\cdot)$ by $d_{GW,r,q}^q(\cdot,\cdot)$.
In this case, the expression \eqref{eq:gw} defines an equivalence relation $\sim$ among probability mm-spaces, i.e.,  $\mathbb{X}\sim\mathbb{Y}$ if and only if  $d_{GW,r,q}(\mathbb{X},\mathbb{Y})=0$\footnote{Moreover, given two probability mm-spaces $\mathbb{X}$ and $\mathbb{Y}$, $d_{GW,r,q}(\mathbb{X},\mathbb{Y})=0$ if and only if there exists a bijective isometry $\phi:X\to Y$ such that $\phi_\#\mu=\nu$. 
In particular, the GW distance is invariant under rigid transformations (translations and rotations) of a given probability mm-space.}. A minimizer of the GW problem \eqref{eq:gw} always exists, and thus, we can replace $\inf$ by $\min$. Moreover,  similar to OT, the above GW problem defines a distance for probability mm-spaces after taking the quotient under $\sim$. For details, we refer to \cite[Ch. 5 and 10]{memoli2011gromov,bai2024efficient}.

Classical GW requires an equal mass assumption, i.e., $|\mu|=|\nu|$, which limits its application in many machine learning tasks, e.g., positive unsupervised learning \cite{chapel2020partial,sejourne2023unbalanced}. To address this issue, in recent years, the above formulation has been extended to the unbalanced setting \cite{chapel2020partial,sejourne2023unbalanced,bai2024efficient,bai2023linear}. 
In particular, two equivalent extensions of the Gromov-Wasserstein problem, named \textbf{Partial Gromov-Wasserstein problem} and \textbf{Mass-constrained Partial Gromov-Wasserstein problem} have been proposed: 
\begin{align}
&PGW_{r,L}(\mathbb{X},\mathbb{Y})=\inf_{\gamma\in \Gamma_\leq(\mu,\nu)}\gamma^{\otimes2}(L(d_X^r,d_Y^r))+\lambda(|\mu^{\otimes2}-\gamma_1^{\otimes2}|+|\nu^{\otimes2}-\gamma_2^{\otimes2}|)\label{eq:PGW}\\
&MPGW_{r,L}(\mathbb{X},\mathbb{Y})=\inf_{\gamma\in \Gamma_\leq^\rho(\mu,\nu)}\gamma^{\otimes2}(L(d_X^r,d_Y^r))\label{eq:MPGW}.
\end{align}
where 
\begin{align}
&\Gamma_\leq(\mu,\nu):=\{\gamma\in\mathcal{M}_+(X\times Y): \gamma_1\leq \mu,\gamma_2\leq \nu\}   \label{eq:Gamma_leq},\\
&\Gamma_\leq^\rho (\mu,\nu):=\{\gamma\in\Gamma_\leq(\mu,\nu):|\gamma|=\rho\}, \rho\in[0,\min(|\mu|,|\nu|)], \label{eq:Gamma_leq^rho}
\end{align}
and the notation $\gamma_1\leq \mu$ denotes that for each Borel set $B\subset X$, $\gamma_1(B)\leq \mu(B)$. 

Note that both GW and its unbalanced extension can measure the similarity of structure data by utilizing their in-structure distance $d_X$ and $d_Y$. The above formulation can naturally measure similarity between graph data $(X,E_X), (Y,E_Y)$, since we can define $d_X$ (and $d_Y$) by the (weights of) edges $E_X$ (and $E_Y$). 
However, suppose that nodes $X, Y$ contain features (e.g., attributed graphs), and thus we can define the distance between $(x,y)$ where for each pair of nodes $x\in X,y\in Y$,  classical GW/PGW can not incorporate this information. To address this limitation, \textbf{Fused Gromov-Wasserstein distance} has been proposed.

\subsection{Fused Gromov-Wasserstein problem}
Given two $mm-$spaces $\mathbb{X},\mathbb{Y},\omega_1,\omega_2\ge 0$ with $\omega_1+\omega_2=1$, a cost function $C:X\times Y\to \mathbb{R}_+$, 
the fused Gromov-Wasserstein problem is defined as: 
{\small
\begin{align}
FGW_{r,L}(\mathbb{X},\mathbb{Y}):=\inf_{\gamma\in\Gamma(\mu,\nu)}\omega_1\gamma(C)+\omega_2 \gamma^{\otimes2}(L(d_X^r,d_Y^r))\label{eq:fgw}. 
\end{align}}
Similar to the original GW problem, the above problem admits a minimizer. In addition, when $C(x,y)=\|x-y\|^q$, and $L(\cdot_1,\cdot_2)=|\cdot_1-\cdot_2|^q$, it defines a semi-metric \cite{titouan2019optimal}. 

\subsection{Fused Unbalanced Gromov-Wasserstein problem.}
The above formulation relies on the equal mass assumption, i.e., $|\mu|=|\nu|$. By relaxing this constraint, the authors in \cite{thual2022aligning} have proposed the following fused-UGW problem: 
{\footnotesize
\begin{align}
&FUGW_{r,L,\lambda}(\mathbb{X},\mathbb{Y}):=\inf_{\gamma\in\mathcal{M}_+(X\times Y)} \omega_1\gamma( C)+\omega_2\gamma^{\otimes2}(L(d_X^r,d_Y^r))+\lambda(D_{\phi_1}(\gamma_1^{\otimes2}\parallel \mu^{\otimes 2})+D_{\phi_2}(\gamma_2^{\otimes2}\parallel \nu^{\otimes2})) \label{eq:fugw},  
\end{align}}
where $\mathcal{M}_+(X\times Y)$ denotes the set of all positive Radon measures defined on $X\times Y$, $D_{\phi_i},i\in[1:2]$ are the f-divergence terms. 
In \cite{thual2022aligning}, the authors set 
$D_{\phi_i},i\in[1:2]$ as KL divergence. They adapt entropic regularization and propose the related Sinkhorn solver.

\section{Fused Partial Gromov-Wasserstein problem}
Inspired by these previous works, by setting the f-divergence terms $D_{\phi_i}$ to be the Total variation, we propose the following ``fused partial Gromov Wasserstein problem'', and the corresponding mass-constrained version: 
{\footnotesize
\begin{align}
&FPGW_{r,L,\lambda}(\mathbb{X},\mathbb{Y})=\inf_{\gamma\in\mathcal{M}_+(X\times Y)} \omega_1  \gamma(C)+\omega_2 \gamma^{\otimes2}(L(d_X^r,d_Y^r))+\lambda(|\mu^{\otimes 2}-\gamma_1^{\otimes2}|_{TV}+|\nu^{\otimes2}-\gamma_2^{\otimes2}|_{TV}),
\label{eq:fpgw_orig}\\
&FMPGW_{r,L,\rho}(\mathbb{X},\mathbb{Y})=\inf_{\gamma\in\Gamma^\rho_\leq(\mu,\nu)} \omega_1  \gamma(C)+\omega_2 \gamma^{\otimes2}(L(d_X^r,d_Y^r))
\label{eq:fmpgw}.
\end{align}
}
where $\lambda\ge 0$.\footnote{
The discrete version of formulation \ref{eq:fmpgw} has been discussed in \cite{liu2023partial} and one of its special case  \eqref{eq:fmpgw} has been introduced by \cite{pan2024subgraph}. 

However, to the best of our knowledge, its fundamental properties and computational methods are not formally discussed in the previous works. We present this formulation as a byproduct of the main contribution of this paper and serve it as a completion of the previous work.}

\begin{theorem}\label{thm:main} We have the following: 
\begin{enumerate}
\item [(1)] When $C(x,y)\ge 0,\forall x,y, L(\mathrm{r}_1,\mathrm{r}_2)\ge 0,\forall \mathrm{r}_1,\mathrm{r}_2\in \mathbb{R}$, the problem \eqref{eq:fpgw_orig} can be further simplified as 
{\small
\begin{align}
&FPGW_{r,L,\lambda}(\mathbb{X},\mathbb{Y})=\inf_{\gamma\in\Gamma_\leq(\mu,\nu)} \omega_1\gamma(C)+\omega_2\gamma^{\otimes2}(L(d_X^r,d_Y^r)-2\lambda)+\lambda(|\mu|^2+|\nu|^2)\label{eq:fpgw}
\end{align}
}
\item[(2)] The problems \eqref{eq:fpgw_orig}, \eqref{eq:fpgw} and \eqref{eq:fmpgw} admit minimizer $\gamma$.
\item[(3)] When $C(x,y)=|x-y|^q,L(\cdot_1,\cdot_2)=|\cdot_1-\cdot_2|^q$, $\omega_2,\lambda>0$, the above formulation  \eqref{eq:fpgw} admits a semi-metric. Furthermore, when $q=1$, \eqref{eq:fpgw} defines a metric. 
\end{enumerate}
\end{theorem}

\subsection{Algorithms for Discrete FPGW}
In discrete case, say $\mu=\sum_{i=1}^np_i\delta_{x_i},\nu=\sum_{i=1}^mq_j\delta_{y_j}$. Let $C^X=[d_X(x_i,x_{i'})]_{i,i'}\in \mathbb{R}^{n\times n},C^Y=[d_Y(y_j,y_{j'})]_{j,j'}$. Thus, $(C^X,\mu),(C^Y,\nu)$ can represent the mm-spaces $\mathbb{X}=(X,d_X,\mu),\mathbb{Y}=(Y,d_Y,\nu)$, respectively. We only discuss the Frank Wolfe solver and Sinkhorn solver for FPGW \ref{eq:fpgw} in the main text. The solvers for FMPGW \ref{eq:fmpgw} are discussed in the appendix.

\subsection{Frank-Wolfe Algorithm}
The above FPGW problem \eqref{eq:fpgw} becomes the following: 
\begin{align}
FPGW_\rho(\mathbb{X},\mathbb{Y})=&\min_{\gamma\in\gamma_\leq^\rho(p,q)}\underbrace{\omega_1\langle C,\gamma \rangle+\omega_2\langle (M-2\lambda)\circ \gamma,\gamma}_{\mathcal{L}_{C,M-2\lambda}} \rangle +\underbrace{\lambda(|\mu|^2+|\nu|^2)}_{\text{constant}} \rangle \label{eq:fmpgw_empirical_1} 
\end{align} 
where $C=[d(x_i,y_j)]_{i\in[1:n],j\in[1:m]}, M=[|d_X(x_i,x_{i'})-d_Y(y_j,y_{j'})|^2]_{i,i'\in[1:n],j,j'\in[1:m]}$ are defined in the previous subsection, $M-2\lambda$ denote the elementwise subtraction, and the constant term will be ignored in the remainder of the paper. 

Similarly to the Fused Gromov-Wasserstein problem, we propose the following Frank-Wolfe algorithm as a solver: The above problem will be solved iteratively. In every iteration, say $k$, we will adapt the following steps: 

\textbf{Step 1. Gradient computation.}

Suppose $\gamma^{(k-1)}$ is the transportation plan in the previous iteration; it is straightforward to verify: 
$$\nabla\mathcal{L}_{C,M-2\lambda}(\gamma)=\omega_1 C+\omega_2(M+M^\top-4\lambda)\circ\gamma.$$
Next, we aim to find the optimal  $\gamma\in\Gamma_\leq(\mu,\nu)$ for the following partial OT problem: 
\begin{align}
\gamma^{(k)}{'}:=\arg\min_{\gamma\in\Gamma_\leq(\mu,\nu)} \langle \nabla\mathcal{L}_{C,M-2\lambda}(\gamma^{(k-1)}),\gamma^{k}\rangle\label{eq:fmpgw_step1_0}. 
\end{align}

\textbf{Step 2. linear search algorithm.}
In this step, we aim to find the optimal step size $\alpha^*\in[0,1]$. In particular,  
\begin{align}
\alpha^*:=\arg\min_{\alpha\in[0,1]}\mathcal{L}_{C,M-2\lambda}((1-\alpha)\gamma^{(k-1)}+\alpha\gamma^{(k)}{'})\nonumber. 
\end{align}

and $\alpha^*$ is given by  the following: 
\begin{align}
\alpha^*=\begin{cases}
    1 &\text{if } a\leq 0, a+b\leq 0, \\
    0 &\text{if }a\leq 0, a+b>0\\
    \text{clip}(\frac{-b}{2a},[0,1]), &\text{if }a>0 
\end{cases}, \begin{cases}
a&=\omega_2\langle (M-2\lambda)\circ \delta\gamma,\delta\gamma\rangle\\
b&=\langle \omega_2(M+M^\top-4\lambda)\circ \gamma^{(k-1)}+\omega_2C,\delta\gamma\rangle \\
\delta \gamma&=\gamma^{(k)'}-\gamma^{(k-1)} 
\end{cases}\label{eq:line_search_sol_1}
\end{align}

\begin{algorithm}[bt]
   \caption{Frank-Wolfe Algorithm for FPGW}
   \label{alg:fpgw}
\begin{algorithmic}
\STATE 
  {\bfseries Input:} $C\in \mathbb{R}^{n\times m}, C^X\in \mathbb{R}^{n\times n},C^Y\in \mathbb{R}^{m\times m}, p\in \mathbb{R}^n_+, q\in\mathbb{R}^m_+$, $\omega_2\in[0,1],\lambda\ge 0$.
   \STATE {\bfseries Output:}
$\gamma^{(final)}$

   \FOR{$k=1,2,\ldots$}
   \STATE 
$G^{(k)}\gets \omega_1C+\omega_2 (M+M^\top-4\lambda)\circ \gamma^{(k)}$ // Compute gradient 
   \STATE $\gamma^{(k)}{'}\gets \arg\min_{\gamma\in \Gamma_\leq(\mathrm{p},\mathrm{q})}\langle G^{(k)}, \gamma\rangle_F$ // Solve the POT problem.  \\ 
   \STATE Compute $\alpha^{(k)}\in[0,1]$ via  \eqref{eq:line_search_sol_1} //  Line Search
   \STATE $\gamma^{(k+1)}\gets (1-\alpha^{(k)})\gamma^{(k)}+\alpha^{(k)} \gamma^{(k)'}$//  Update $\gamma$  
   \STATE if convergence, break
   \ENDFOR
\STATE $\gamma^{(final)}\gets \gamma^{(k)}$
\end{algorithmic}
\end{algorithm}

In algorithm  \eqref{alg:fpgw}, the computational complexity can be written as $\mathcal{O}(\mathrm{C}\cdot \mathrm{L})$, where $\mathrm{C}$ is the complexity of each iteration and $\mathcal{L}$ is number of iterations that the algorithms converge. 

When a linear programming solver \cite{bonneel2011displacement} for partial OT is adopted, $\mathcal{C}=(n+m)nm$. If we adapt Sinknorn algorithm \cite{cuturi2014fast}, $\mathcal{C}=\mathcal{O}(\frac{1}{\epsilon}\ln(n+m)nm)$. The number of iterations $\mathcal{L}$ refers to the convergence analysis of the FW algorithm. We refer to the Appendix \eqref{sec: convergence} for details. 

\subsection{Sinkhorn Algorithm}
Another popular solver for the Gromov Wasserstein problem is the Sinkhorn algorithm \cite{sejourne2023unbalanced}. In Fused-PGW setting, the problem is defined as: 

\begin{align}
EFPGW (\mathbb{X},\mathbb{Y})&:=\min_{\gamma\in \Gamma_\leq(\mu,\nu)}\mathcal{L}(\gamma)+\epsilon \overline{D}_{KL}(\gamma^{\otimes 2}\parallel (\mu\otimes \nu)^{\otimes2})\label{eq:entropy-fpgw-main}\\
&\overline{D}_{KL}(A\parallel B)=\int \frac{dA}{dB}dA ,\text{ for any postivie radon measures }A,B \nonumber \\
&\mathcal{L}(\gamma):=\mathcal{L}_{C,M-2\lambda}=\omega_1\langle c,\gamma \rangle+\omega_2 \langle L(d_X^r,d_Y^r), \gamma^{\otimes 2} \rangle+\lambda(|\mu|^2+|\nu|^2-2|\gamma|^2). \nonumber  
\end{align}

Problem \eqref{eq:entropy-fpgw-main} can be further relaxed as:

\begin{align}
&\min_{\gamma,\pi\in\Gamma_\leq(\mu,\nu)}\mathcal{F}(\gamma,\pi)+\epsilon \overline{D}_{KL}(\pi\otimes \gamma \parallel (\mu\otimes \nu)^{\otimes2})\label{eq:entropy-fpgw-2-main}\\
&\mathcal{F}(\mu,\nu):=\omega_1 \langle d(x,y),\frac{\gamma+\pi}{2}\rangle+\omega_2\langle L(d_X^r,d_Y^r),\gamma\otimes \pi\rangle+\lambda(|\mu|^2+|\nu|^2-2|\gamma||\pi|)  \nonumber 
\end{align}
It is clear $\mathcal{F}(\gamma,\gamma)=\mathcal{F}(\gamma)$. Thus, $\eqref{eq:entropy-fpgw-2-main}\leq \eqref{eq:entropy-fpgw-main}$, and we denote \eqref{eq:entropy-fpgw-2-main} as $LB-FPGW_\lambda(\mathbb{X},\mathbb{Y})$ (lower bound of Fused Partial Gromov Wasserstein). Essentially, the Sinkhorn algorithm aims to solve $LB-FPGW$. 

We first introduce the following fundamental proposition. Note, a similar version can be found in \cite[Proposition 4]{sejourne2021unbalanced}:  
\begin{proposition}\label{pro:sinkhorn-fmpgw}
Given a fixed $\pi\in \Gamma_\leq(\mu,\nu)$, considering the problem: 
$$\min_{\gamma\in\mathcal{M}_+(X\times Y)}\mathcal{F}(\pi,\gamma)+\epsilon \overline{D}_{KL}(\pi\otimes \gamma\parallel (\mu\otimes \nu)^{\otimes 2}),$$
it is equivalent to solve the following entropic optimal partial transport problem: 
\begin{align}
\min_{\gamma\in\Gamma_\leq(\mu,\nu)}\int_{X\times Y} c_{\pi}(x,y) d\gamma+\lambda |\pi| (|\mu|+|\nu|-2|\gamma|)+\epsilon |\pi|\overline{D}_{KL}(\gamma\parallel \mu\otimes \nu), \label{eq:entropic_pot}  
\end{align}
where
\begin{align}
c_\pi(x,y)&=\frac{1}{2}\omega_1d(x,y)+\omega_2 [L(d_X^r,d_Y^r)\circ \pi] (x,y)+\epsilon \overline{D}_{KL}(\pi\parallel \mu\otimes 
\nu).\nonumber
\end{align} 
\end{proposition}

Given these fundamental results, we can present the Sinkhorn algorithm \ref{alg:sink-fpgw}. 

\begin{algorithm}[bt]
   \caption{Sinkhorn Algorithm for FPGW}
   \label{alg:sink-fpgw}
\begin{algorithmic}
\STATE 
  {\bfseries Input:} $C\in \mathbb{R}^{n\times m}, C^X\in \mathbb{R}^{n\times n},C^Y\in \mathbb{R}^{m\times m}, p\in \mathbb{R}^n_+, q\in\mathbb{R}^m_+$, $\omega_2\in[0,1],\lambda\ge 0$.
   \STATE {\bfseries Output:}
$\gamma$

\FOR{$k=1,2,\ldots$}
\STATE $\pi \gets \gamma$
\STATE Solve the Sinkhorn partial OT problem \eqref{eq:entropic_pot}:
$$\gamma\gets \min_{\gamma\in\Gamma_\leq(\mu,\nu)}\int_{X\times Y} c_{\pi}(x,y) d\gamma+\lambda |\pi| (|\mu|+|\nu|-2|\gamma|)+\epsilon |\pi|D_{KL}(\gamma\parallel \mu\otimes \nu) $$

\STATE Fix $\gamma$ and solve the similar Sinkhorn partial OT problem \eqref{eq:entropic_pot}:
$$\pi\gets \min_{\pi\in\Gamma_\leq(\mu,\nu)}\int_{X\times Y} c_{\gamma}(x,y) d\gamma+\lambda |\gamma| (|\mu|+|\nu|-2|\pi|)+\epsilon |\pi|D_{KL}(\pi\parallel \mu\otimes \nu) $$

\STATE Rescale $\gamma\gets \sqrt{|\pi|/|\gamma|}\gamma$
\STATE Break if $\pi\approx \gamma$
\ENDFOR
\end{algorithmic}
\end{algorithm}

\section{Numerical Applications}

\begin{figure}[t]
\centering   
\begin{subfigure}[t]{0.45\columnwidth}
\centering
\includegraphics[width=\textwidth]{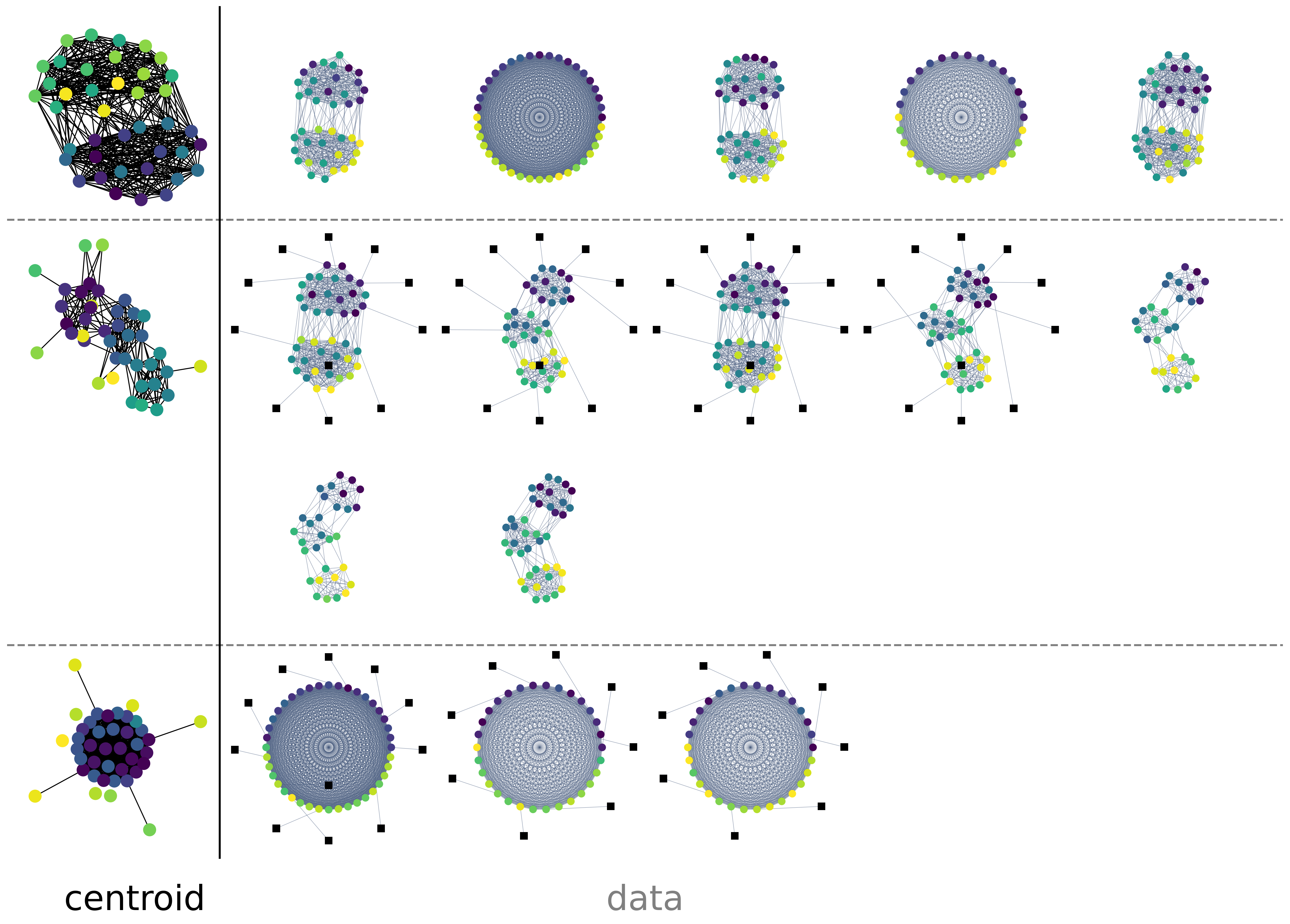}
\caption{FGW-kmeans}
\label{fig:fgw}
\end{subfigure}
\begin{subfigure}[t]{0.45\columnwidth}
        \centering
\includegraphics[width=\textwidth]{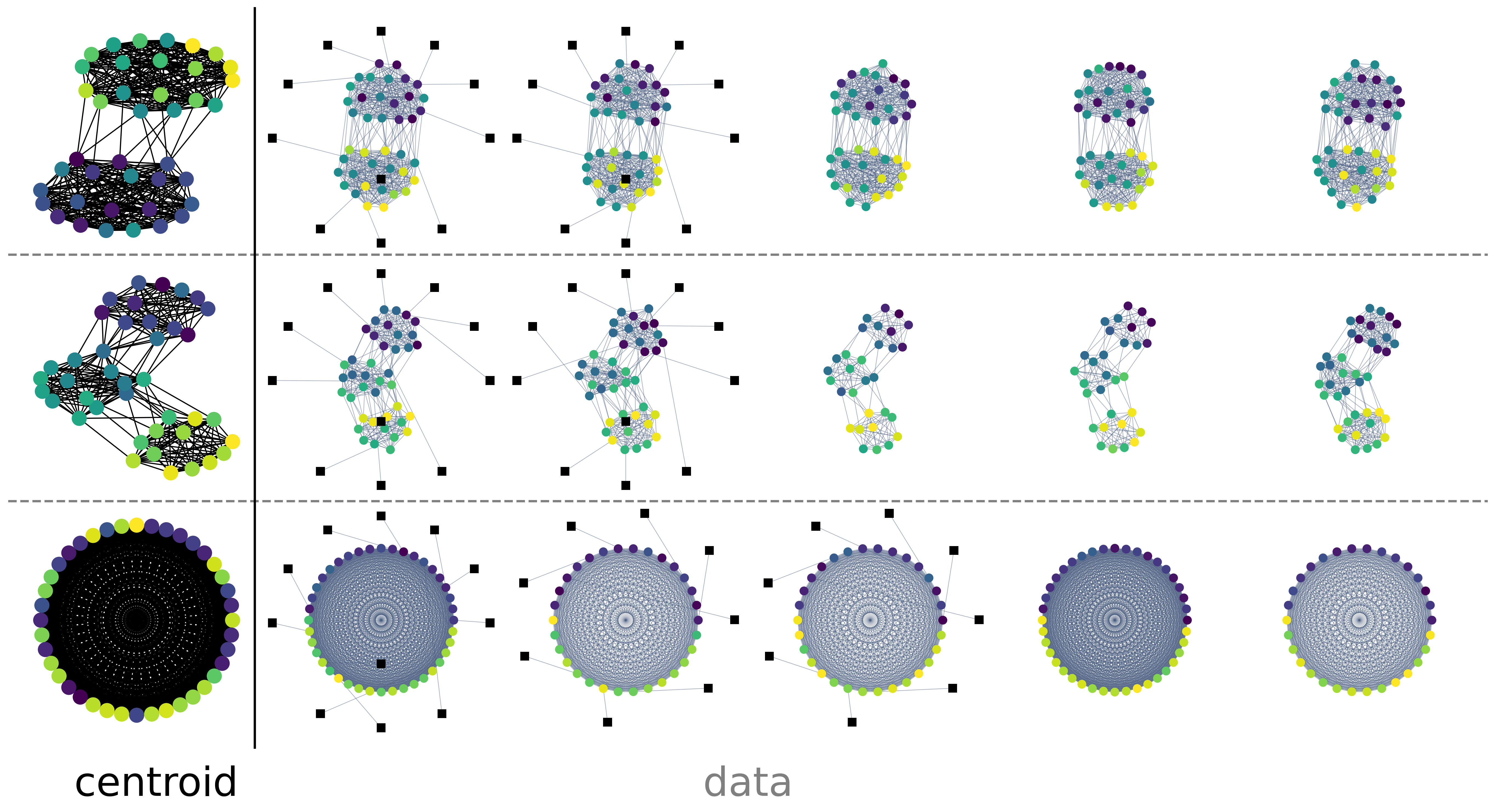}
\caption{FPGW-kmeans}
\label{fig:fmpgw}
\end{subfigure}
\caption{We present the clustering results of the FGW and FPGW k-means methods. \\  
In the first column, we visualize the centroids obtained using both methods. The centroids are represented by their features and structure (distance matrix). The edges of these graphs are either reconstructed or approximated based on the returned distance matrix. Additionally, the color of each node corresponds to its feature. \\  
The clustering results are shown in the remaining columns. For each graph, the color of regular nodes represents their features, while all outlier nodes are depicted as black squares.
}
\label{fig:clustering}
\end{figure}

\subsection{Toy example: Graph clustering}
Given a set of unlabeled graphs $\{G_1,\ldots G_K\}$, we compare the performance of FGW and FPGW in the graph clustering task. 

\textbf{Dataset setup.}
We adapt the dataset generated by \cite{feydy2017optimal}. Each graph follows a simple Stochastic Block Model, and the groups are defined w.r.t. the number of communities inside each graph and the distribution of their labels. The dataset contains three different types of graphs, and each type contains five graphs. Each graph contains 30 or 40 nodes, and the feature of each node is randomly selected in a compact set $[-2,2]$. 

In addition, we randomly generate outliers and add them to graphs. In particular, we select 50\% of graphs and add $30\%$ outliers. These outlier nodes are randomly connected or connected to the original nodes. For each outlier node, we define its feature as a random value in $[2,2+2\text{std}]$, where $\text{std}$ is the standard deviation of all features of nodes in the graph. 
Similar to the setting of graph classification, for each graph $G$, we suppose $N_G$ is the number of nodes that are not outliers. We assume $N_G$ is known in this experiment. We refer to figure \ref{fig:clustering} for the visualization of these graphs.

\textbf{Baseline: FGW K-means.}
We consider the FGW K-means method introduced in Section 4.3 of \cite{feydy2017optimal} as the baseline method. Specifically, given graphs $\{G_1, \ldots, G_K\}$ and the number of clusters $K' \leq K$, the method iteratively alternates between the following two steps:

\begin{itemize}
    \item \textbf{Step 1.} For each $i \in [1:K]$ and $j \in [1:K']$, compute the distance between each graph $G_i$ and centroid $G_j'$. Assign each graph $G_i$ to the closest centroid based on the computed distances.
    \item \textbf{Step 2.} Using the updated assignments, recompute each centroid $G_j'$, where $j'\in[1:K']$, as the center of the graphs assigned to it.
\end{itemize}

In the FGW K-means method, the Fused-GW distance is used to define the distance in Step 1, while the Fused-GW barycenter is used to compute the ``center'' in Step 2.

\textbf{Our method: FPGW K-means.}  
Inspired by the FGW K-means method, we introduce the FPGW K-means method. In summary, we adapt the FMPGW discrepancy to measure the distance in Step 1, and introduce the \textbf{FPGW barycenter} to define the ``center'' in Step 2.

First, we convert the graphs to mm-spaces using the formulation introduced in Section \ref{sec:graph_mm}. Specifically, given a graph $G = (V = \{v_1, \ldots, v_N\}, E)$, we define  
\[
\mathbb{X} = \left(V, C, \sum_{i=1}^N \frac{1}{N_G} \delta_{v_i}\right),
\]
where $C = [c(v_i, v_{i'})]_{i, i' \in [1:N]} \in \mathbb{R}^{N \times N}$, and $c(v_i, v_{i'})$ is the ``shortest path distance'' determined by $E$. Here, $N_G$ is the number of nodes that are not outliers, and its value is known based on the experiment setup. The feature distance between nodes $v_i$ and $v_j$ is defined as the Euclidean distance. We then use the Fused-PGW discrepancy \eqref{eq:fmpgw} to measure the distance between each pair of graphs and centroids in Step 1 of the K-means method.

For Step 2, suppose the set of graphs $\boldsymbol{G}^k \subset \{G_1, \ldots, G_K\}$ is assigned to cluster $k'$ for each $k' \in [1:K']$. We define the ``center'' using the following \textbf{Fused-PGW barycenter}:
\begin{align}
\min_{G} \sum_{G_k \in \boldsymbol{G}^k} \frac{1}{|\boldsymbol{G}^k|} \text{FMPGW}_{\rho_k}(G, G_k), \label{eq:barycenter_main}
\end{align}
where $\rho_k = 1$ for all $k$. The solution to this problem provides the updated centroid for cluster $k'$. Details of the formal formulation and solver are provided in Appendix \ref{sec:barycenter}.

We iteratively repeat the above two steps until all centroids/assignments converge.

\textbf{Other parameter setting.}
In this experiment, we set $\omega_2=0.999$ for both the FGW and FPGW methods. In addition, we set the number of clustering $K'=3$ for both methods. For each centroid, we initialize it as a random connected graph with $40$ nodes.

\textbf{Performance analysis.}
We present the clustering results of FGW K-means and FPGW K-means in Figure~\ref{fig:clustering}. The performance of the FGW method is significantly impacted due to the presence of outlier nodes in half of the graphs. In sharp contrast, FPGW demonstrates a more robust performance, with clustering results closely aligning with the ground truth. This robustness is attributed to the partial matching property of FPGW. 

From the centroids visualization for FGW/FPGW methods, it is evident that FPGW effectively excludes most of the information from outliers, whereas FGW incorporates it into the centroids.

Regarding the wall-clock time, FGW requires 41.9 seconds, while FPGW requires 82.7 seconds.

\subsection{Graph Matching}

\textbf{Dataset setup}. We apply our method on seven widely used graph datasets with continuous node attributes: \textit{Synthetic} \cite{feragen2013scalable}, \textit{Enzymes, Protein, AIDS} \cite{borgwardt2005shortest}, \textit{Cuneiform} \cite{kriege2016valid}, \textit{COX2, and BZR} \cite{sutherland2003spline}. Each dataset consists of hundreds or thousands of connected graphs. For each graph, we use its adjacency matrix as structural information and its attributes as node features. To create a partial matching task, we utilize the BFS (Breadth-first search
) method to randomly extract subgraphs containing 50\% of the original nodes and their corresponding edges. We also benchmark performance on the \textit{Douban dataset} \cite{8443159}, which provides a large online graph (3,906 nodes) and a smaller offline subgraph (1,118 nodes), using user locations as node features.

\textbf{Baselines.} We compare sink-FPGW against competitive methods including 
balanced GW methods, 
\textit{SpecGW} \cite{chowdhury2021generalized}, \textit{eBPG} \cite{Solomon2016Entropic}, \textit{BPG} \cite{Xu2019Gromov}, \textit{BAPG} \cite{Li2023Convergent}, \textit{srGW} \cite{VincentCuaz2022Semi}, and unbalanced GW methods: \textit{UGW} \cite{sejourne2021unbalanced}, \textit{PGW} \cite{chapel2020partial,bai2024efficient}, \textit{RGW} \cite{kong2024outlier}, 
\textit{FUGW} \cite{thual2022aligning}.

\textbf{Settings of GW methods and our method.} 
In all these method, we first convert grpahs into mm-spaces (see appendix \ref{sec:graph_mm}). In all these methods, we default the probability mass function of the query graph and original graph as $\mu=\sum_{i=1}^mp_i\delta_{v_i},\nu_{i=1}^nq_j\delta_{v_i}$, with $p_i=1/m,q_j=1/n$, where $m$ and $n$ denote the numbers of the source and target nodes. For the computation of the cost matrix C in \eqref{eq:fpgw}, we use the Euclidean distance among the continuous node attributes.

\textbf{Evaluation metric and performance analysis.} 
Accuracy is defined as the fraction of ground-truth elements correctly recovered in the predicted set, 
$\mathrm{Acc} = \frac{\lvert S_{\mathrm{gt}}\cap S_{\mathrm{pred}}\rvert}{\lvert S_{\mathrm{gt}}\rvert}\times 100\%$, following \cite{kong2024outlier}. 
As shown in Table~2, our proposed sink-FPGW achieves superior accuracy on most datasets while remaining highly efficient. 
The key advantage comes from its ability to leverage both node features (linear part) and structure information (quadratic part), unlike methods limited to structural cues. 
This design not only boosts accuracy but also yields significant computational efficiency, making sink-FPGW consistently faster than existing alternatives.  


\begin{table}[h!]
\centering
\caption{Comparison of methods on multiple datasets. Acc = accuracy, Time = runtime. For the first seven datasets (excluding Douban), Time corresponds to the total graph set matching time, while for Douban it reflects the runtime of a single matching instance.}
\begin{adjustbox}{width=\textwidth}
\begin{tabular}{l rr rr rr rr}
\toprule
& \multicolumn{2}{c}{\textbf{Synthetic}}
& \multicolumn{2}{c}{\textbf{Enzymes}}
& \multicolumn{2}{c}{\textbf{Cuneiform}}
& \multicolumn{2}{c}{\textbf{COX2}} \\
\cmidrule(lr){2-3} \cmidrule(lr){4-5} \cmidrule(lr){6-7} \cmidrule(lr){8-9}
\textbf{Method} & Acc & Time & Acc & Time & Acc & Time & Acc & Time \\
\midrule

SpecGW & $0.00^{\pm0.00}$ & 5.25 & 9.98$^{\pm 0.11}$ & 3.13 & 7.15$^{\pm0.12}$ & 1.06 & 5.27$^{\pm 0.06}$ & 2.82 \\
eBPG &  $2.00^{\pm 0.00}$ & 121.84 & 12.68$^{\pm 0.12}$ & 9950.78 & 5.99$^{\pm0.09}$ & 4310.76 & 8.99$^{\pm 0.06}$ & 5760.38 \\
BPG & $4.00^{\pm 0.00}$ & 21.12 & 28.18$^{\pm 0.20}$ & 90.71 & 5.81$^{\pm0.07}$ & 12.13 & 23.32$^{\pm 0.13}$ & 48.88 \\
BAPG &  $88.03^{\pm 0.05}$ & 45.85  & 62.07$^{\pm 0.24}$ & 14.34 & 72.05$^{\pm0.17}$ & 2.02 & 21.23$^{\pm 0.10}$ & 15.68 \\
srGW & $0.00^{\pm 0.00}$ & 60.27 & 15.22$^{\pm 0.24}$ & 64.24 & 19.94$^{\pm0.09}$ & 7.24 & 1.67$^{\pm 0.02}$ & 33.12 \\
PGW & $0.00^{\pm 0.00}$ & 6.19 & 9.74$^{\pm 0.13}$ & 17.26 & 9.94$^{\pm0.09}$ & 14.73 & 9.41$^{\pm 0.11}$ & 11.69 \\
UGW & $2.00^{\pm 0.00}$ & 12.84 & 17.83$^{\pm 0.26}$ & 732.21 & 04.17$^{\pm0.08}$ & 375.36 & 3.93$^{\pm 0.05}$ & 65.03 \\
RGW & $37.24^{\pm 0.23}$ & 424.68 & 77.20$^{\pm 0.28}$ & 137.70 & 85.33$^{\pm0.21}$ & 25.87 & 37.05$^{\pm 0.15}$ & 188.99 \\
\midrule
FGW & $46.68^{\pm 0.06}$ & 5.05 & 62.92$^{\pm 0.24}$ & 9.26 & 86.78$^{\pm0.08}$ & 1.20 & 73.28$^{\pm 0.22}$ & 5.32 \\
FUGW & \textbf{99.89}$^{\pm 0.01}$ & 70.76 & 90.91$^{\pm 0.19}$ & 335.71 & 96.88$^{\pm0.04}$ & 109.74 & 90.73$^{\pm 0.23}$ & 166.92 \\
sink-FPGW(ours) & 99.70$^{\pm 0.01}$ & 4.90 & \textbf{93.47}$^{\pm 0.18}$ & 1.81 & \textbf{99.96}$^{\pm0.01}$ & 0.3245 & \textbf{92.62}$^{\pm 0.23}$ & 1.38 \\
\bottomrule
\end{tabular}
\end{adjustbox}
\vspace{1em} 

\begin{adjustbox}{width=\textwidth}
\begin{tabular}{l rr rr rr rr}
\toprule
& \multicolumn{2}{c}{\textbf{BZR}}
& \multicolumn{2}{c}{\textbf{Protein}}
& \multicolumn{2}{c}{\textbf{AIDS}}
& \multicolumn{2}{c}{\textbf{Douban}} \\
\cmidrule(lr){2-3} \cmidrule(lr){4-5} \cmidrule(lr){6-7} \cmidrule(lr){8-9}
\textbf{Method} & Acc & Time & Acc & Time & Acc & Time & Acc & Time \\
\midrule

SpecGW &  9.51$^{\pm 0.07}$ & 2.08   &   12.24$^{\pm 0.15}$ & 9.29   &   23.89$^{\pm 0.19}$ & 7.36  &  0.00 & 65.11 \\
eBPG &   14.21$^{\pm 0.10}$ & 5540.20   &   14.88$^{\pm 0.16}$ & 15936.79   &   25.98$^{\pm 0.20}$ & 26182.34  & 0.09  &13.24 \\
BPG &    25.67$^{\pm 0.18}$ & 48.78   &   30.07$^{\pm 0.19}$ & 134.55  &   30.49$^{\pm 0.22}$ & 106.86  & 58.59  & 391.00 \\
BAPG &   34.54$^{\pm 0.15}$ & 11.11  &   25.52$^{\pm 0.22}$ & 66.90   &    50.08$^{\pm 0.25}$ & 15.42  & 53.94  & 2194.70 \\
srGW &   3.20$^{\pm 0.04}$ & 26.55   &   18.96$^{\pm 0.26}$ & 149.17   &   23.76$^{\pm 0.24}$ & 150.31  & 4.38  & 1584.68\\
PGW &   6.80$^{\pm 0.08}$ & 9.59   &   13.22$^{\pm 0.18}$ & 97.74   &   22.48$^{\pm 0.21}$ & 22.89  & 2.06  & 1363.59 \\
UGW &   6.76$^{\pm 0.07}$ & 152.98   &   15.57$^{\pm 0.22}$ & 1280.80   &   21.32$^{\pm 0.20}$ & 1152.91  & 0.09  & 702.53\\
RGW &   39.54$^{\pm 0.19}$ & 120.97   &   38.26$^{\pm 0.30}$ & 511.67   &   57.05$^{\pm 0.34}$ & 340.94  & 51.88  & 17784.09\\
\midrule
FGW &   74.10$^{\pm 0.27}$ & 8.09   &   63.95$^{\pm 0.24}$ & 29.29   &    84.31$^{\pm 0.21}$ & 6.76  & 24.60  & 367.15 \\
FUGW &  89.21$^{\pm 0.23}$ & 133.96   &   74.87$^{\pm 0.23}$ & 733.48   &    97.12$^{\pm 0.10}$ & 617.36  & 66.37  & 226.22\\
sink-FPGW(ours) &   \textbf{93.51}$^{\pm 0.24}$ & 0.97   &   \textbf{96.21}$^{\pm 0.15}$ & 7.17 &  \textbf{98.89}$^{\pm 0.09}$ & 2.33  & \textbf{66.99}  & 540.33\\
\bottomrule
\end{tabular}
\end{adjustbox}
\label{tab:results}
\end{table}

\subsection{Geometry Matching}

\textbf{Dataset setup}. We evaluate our method on a partial geometry matching task using the 'Victoria' model from the \textit{TOSCA} \cite{DBLP:journals/corr/RodolaCBTC15} 3D mesh dataset. This high-resolution mesh consists of 10,000 nodes, which are segmented into two parts: an upper body (4,691 nodes) and a lower body (5,436 nodes).

\textbf{Our methods.} For mesh objects, we encode them to mm-spaces by the approach described in section \ref{sec:graph_mm} and then apply \eqref{alg:sink-fpgw} (sink-FPGW) to solve the optimization problem \eqref{eq:fpgw} for the geometry matching task. In particular, we define the source and target mass distribution functions as $p_i = q_j= \frac{1}{\max(n,m)}$ for all the methods, where $m,n$ denotes the number of nodes of the source and target graphs. For both complete and partial geometries, we construct node features using the Euclidean distances from each node to a single randomly selected anchor node, while the mesh structure is encoded via the adjacency matrix.

\textbf{Baselines.}
We benchmark our proposed sink-FPGW method against two competitive baselines, RGW and FUGW. To ensure a fair comparison, RGW is initialized with a transport plan derived from node features, and the regularization hyperparameters for both FUGW and sink-FPGW are optimized via line search.

\textbf{Evaluation metric and performance analysis.}
Following the RGW evaluation protocol, sink-FPGW achieves matching accuracies of $99.47\%$ and $99.77\%$ on the two partial matching tasks ($43.47\%$ and $48.05\%$ in RGW, and $98.99\%$ and $56.09\%$ in FUGW), where $\mathrm{Acc} = \frac{\lvert S_{\mathrm{gt}}\cap S_{\mathrm{pred}}\rvert}{\lvert S_{\mathrm{gt}}\rvert}\times 100\%$. We visualize the one-hot transport plans as real shapes (Figure~\ref{fig:geometry_matching} (b)-(d)) and heatmaps (Figure~\ref{fig:geometry_matching} (e)-(h)), confirming that sink-FPGW yields the most accurate matches. The one-hot plans are obtained by taking the argmax match for each source node from the transport plans.

\begin{figure}
    \centering
\includegraphics[width=\textwidth]{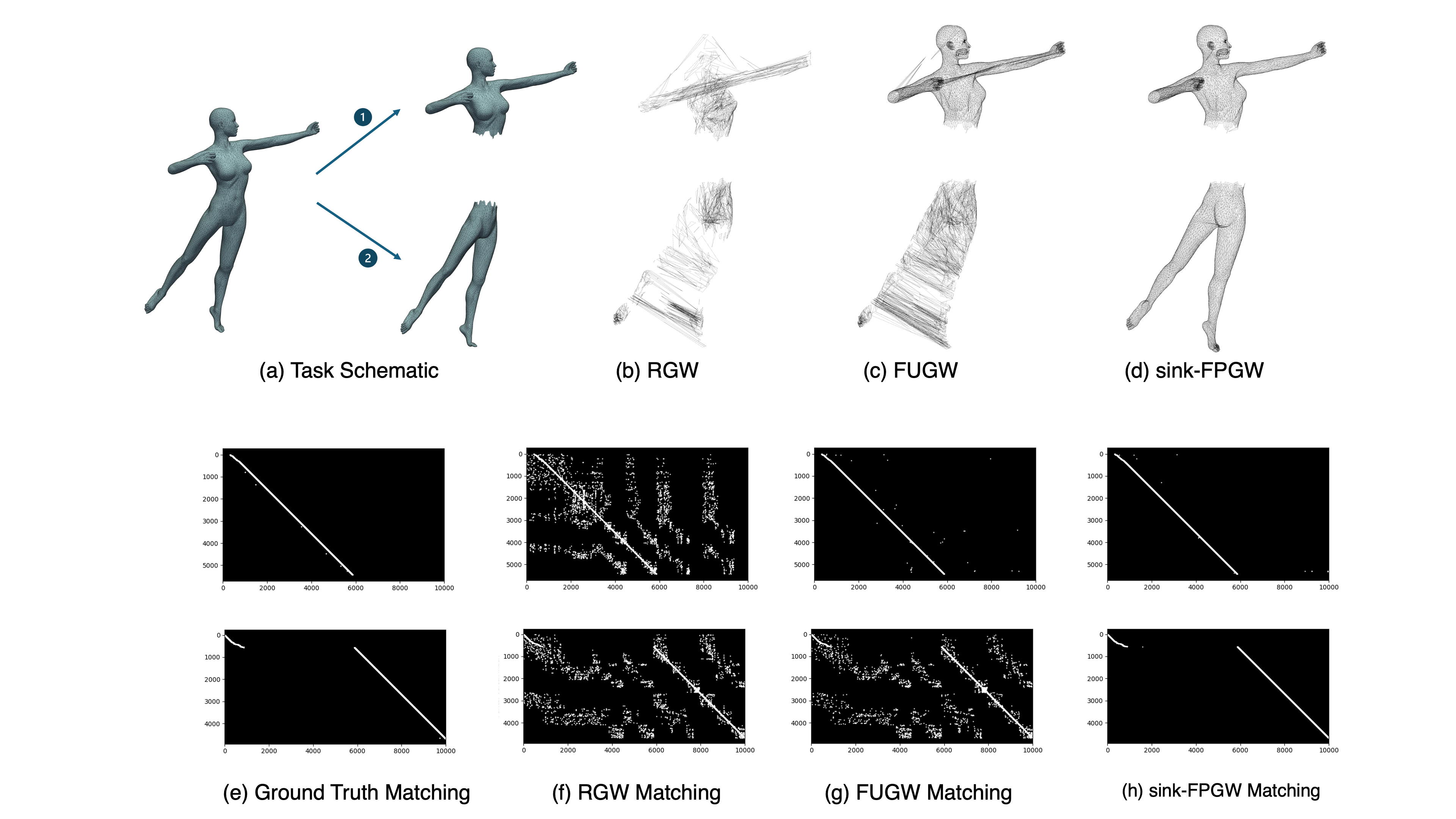}
    \vspace{+0.in}
    \caption{Qualitative results for the partial geometry matching experiment. (a): Schematic illustrating the two whole-to-partial matching tasks; (b)-(d): Visualizations of the partial mesh reconstructed from the whole source mesh using the transport plans computed by RGW, FUGW, and sink-FPGW.; (e)-(h) Corresponding heatmap visualizations for the ground truth, RGW, FUGW, and sink-FPGW, the top row corresponding to the first partial matching task (upper body) and the bottom row to the second (lower body). The x- and y-axes represent the vertex indices of the whole and partial meshes, respectively. }
    \label{fig:geometry_matching}
\end{figure}

\section{Summary.}
In this paper, we proposed a novel formulation called ``fused-partial Gromov-Wasserstein'' (fused-PGW) for comparing structured objects. Theoretically, we demonstrated the metric properties of fused-PGW, and numerically, we introduced the corresponding Frank-Wolfe solver, Sinkhorn Solver and barycenter algorithms. Finally, we applied fused-PGW to graph matching and clustering experiments, showing that it achieves more robust performance due to its partial matching property.

\section*{Acknowledgment} 
This research was partially supported by NSF CAREER Award No. 2339898.
The authors thank Dr. Rocio Martin (Florida State University) for valuable discussions and Xinran Liu (Vanderbilt University) for assistance with code testing.

\bibliography{iclr2026_conference,ref}

\begin{thebibliography}{10}

\bibitem{ruddigkeit2012enumeration}
Lars Ruddigkeit, Ruud Van~Deursen, Lorenz~C Blum, and Jean-Louis Reymond.
\newblock Enumeration of 166 billion organic small molecules in the chemical universe database gdb-17.
\newblock {\em Journal of chemical information and modeling}, 52(11):2864--2875, 2012.

\bibitem{bassett2017network}
Danielle~S Bassett and Olaf Sporns.
\newblock Network neuroscience.
\newblock {\em Nature neuroscience}, 20(3):353--364, 2017.

\bibitem{hamilton2017inductive}
Will Hamilton, Zhitao Ying, and Jure Leskovec.
\newblock Inductive representation learning on large graphs.
\newblock {\em Advances in neural information processing systems}, 30, 2017.

\bibitem{graves2006connectionist}
Alex Graves, Santiago Fern{\'a}ndez, Faustino Gomez, and J{\"u}rgen Schmidhuber.
\newblock Connectionist temporal classification: labelling unsegmented sequence data with recurrent neural networks.
\newblock In {\em Proceedings of the 23rd international conference on Machine learning}, pages 369--376, 2006.

\bibitem{bille2008integrated}
Rapha{\"e}l Bill{\'e}.
\newblock Integrated coastal zone management: four entrenched illusions.
\newblock {\em SAPI EN. S. Surveys and Perspectives Integrating Environment and Society}, (1.2), 2008.

\bibitem{wang2004integration}
L~Wang, WP~Sousa, and P~Gong.
\newblock Integration of object-based and pixel-based classification for mapping mangroves with ikonos imagery.
\newblock {\em International journal of remote sensing}, 25(24):5655--5668, 2004.

\bibitem{Villani2009Optimal}
Cedric Villani.
\newblock {\em Optimal transport: old and new}.
\newblock Springer, 2009.

\bibitem{chizat2018unbalanced}
Lenaic Chizat, Gabriel Peyr{\'e}, Bernhard Schmitzer, and Fran{\c{c}}ois-Xavier Vialard.
\newblock Unbalanced optimal transport: Dynamic and {Kantorovich} formulations.
\newblock {\em Journal of Functional Analysis}, 274(11):3090--3123, 2018.

\bibitem{figalli2010optimal}
Alessio Figalli.
\newblock The optimal partial transport problem.
\newblock {\em Archive for rational mechanics and analysis}, 195(2):533--560, 2010.

\bibitem{wang2013linear}
Wei Wang, Dejan Slep{\v{c}}ev, Saurav Basu, John~A Ozolek, and Gustavo~K Rohde.
\newblock A linear optimal transportation framework for quantifying and visualizing variations in sets of images.
\newblock {\em International journal of computer vision}, 101(2):254--269, 2013.

\bibitem{cai2022linearized}
Tianji Cai, Junyi Cheng, Bernhard Schmitzer, and Matthew Thorpe.
\newblock The linearized hellinger--kantorovich distance.
\newblock {\em SIAM Journal on Imaging Sciences}, 15(1):45--83, 2022.

\bibitem{bai2023linear}
Yikun Bai, Ivan~Vladimir Medri, Rocio~Diaz Martin, Rana Shahroz, and Soheil Kolouri.
\newblock Linear optimal partial transport embedding.
\newblock In {\em International Conference on Machine Learning}, pages 1492--1520. PMLR, 2023.

\bibitem{martin2023lcot}
Rocio~Diaz Martin, Ivan Medri, Yikun Bai, Xinran Liu, Kangbai Yan, Gustavo~K Rohde, and Soheil Kolouri.
\newblock Lcot: Linear circular optimal transport.
\newblock {\em arXiv preprint arXiv:2310.06002}, 2023.

\bibitem{kolouri2019generalized}
Soheil Kolouri, Kimia Nadjahi, Umut Simsekli, Roland Badeau, and Gustavo Rohde.
\newblock Generalized sliced {Wasserstein} distances.
\newblock {\em Advances in Neural Information Processing Systems}, 32, 2019.

\bibitem{bonneel2015sliced}
Nicolas Bonneel, Julien Rabin, Gabriel Peyr{\'e}, and Hanspeter Pfister.
\newblock Sliced and {Radon} {Wasserstein} barycenters of measures.
\newblock {\em Journal of Mathematical Imaging and Vision}, 51(1):22--45, 2015.

\bibitem{bai2022sliced}
Yikun Bai, Bernhard Schmitzer, Matthew Thorpe, and Soheil Kolouri.
\newblock Sliced optimal partial transport.
\newblock In {\em Proceedings of the IEEE/CVF Conference on Computer Vision and Pattern Recognition}, pages 13681--13690, 2023.

\bibitem{rowland2019orthogonal}
Mark Rowland, Jiri Hron, Yunhao Tang, Krzysztof Choromanski, Tamas Sarlos, and Adrian Weller.
\newblock Orthogonal estimation of wasserstein distances.
\newblock In {\em The 22nd International Conference on Artificial Intelligence and Statistics}, pages 186--195. PMLR, 2019.

\bibitem{memoli2011gromov}
Facundo M{\'e}moli.
\newblock Gromov--wasserstein distances and the metric approach to object matching.
\newblock {\em Foundations of computational mathematics}, 11:417--487, 2011.

\bibitem{memoli2009spectral}
Facundo M{\'e}moli.
\newblock Spectral gromov-wasserstein distances for shape matching.
\newblock In {\em 2009 IEEE 12th International Conference on Computer Vision Workshops, ICCV Workshops}, pages 256--263. IEEE, 2009.

\bibitem{sejourne2021unbalanced}
Thibault S{\'e}journ{\'e}, Fran{\c{c}}ois-Xavier Vialard, and Gabriel Peyr{\'e}.
\newblock The unbalanced gromov wasserstein distance: Conic formulation and relaxation.
\newblock {\em Advances in Neural Information Processing Systems}, 34:8766--8779, 2021.

\bibitem{chapel2020partial}
Laetitia Chapel, Mokhtar~Z Alaya, and Gilles Gasso.
\newblock Partial optimal tranport with applications on positive-unlabeled learning.
\newblock {\em Advances in Neural Information Processing Systems}, 33:2903--2913, 2020.

\bibitem{bai2023partial}
Yikun Bai, Huy Tran, Steven~B Damelin, and Soheil Kolouri.
\newblock Partial transport for point-cloud registration.
\newblock {\em arXiv preprint arXiv:2309.15787}, 2023.

\bibitem{kong2024outlier}
Lemin Kong, Jiajin Li, Jianheng Tang, and Anthony Man-Cho So.
\newblock Outlier-robust gromov-wasserstein for graph data.
\newblock {\em Advances in Neural Information Processing Systems}, 36, 2024.

\bibitem{feydy2017optimal}
Jean Feydy, Benjamin Charlier, Fran{\c{c}}ois-Xavier Vialard, and Gabriel Peyr{\'e}.
\newblock Optimal transport for diffeomorphic registration.
\newblock In {\em Medical Image Computing and Computer Assisted Intervention- MICCAI 2017: 20th International Conference, Quebec City, QC, Canada, September 11-13, 2017, Proceedings, Part I 20}, pages 291--299. Springer, 2017.

\bibitem{vayer2020fused}
Titouan Vayer, Laetitia Chapel, R{\'e}mi Flamary, Romain Tavenard, and Nicolas Courty.
\newblock Fused gromov-wasserstein distance for structured objects.
\newblock {\em Algorithms}, 13(9):212, 2020.

\bibitem{thual2022aligning}
Alexis Thual, Quang~Huy Tran, Tatiana Zemskova, Nicolas Courty, R{\'e}mi Flamary, Stanislas Dehaene, and Bertrand Thirion.
\newblock Aligning individual brains with fused unbalanced gromov wasserstein.
\newblock {\em Advances in Neural Information Processing Systems}, 35:21792--21804, 2022.

\bibitem{halmos2025dest}
Peter Halmos, Xinhao Liu, Julian Gold, Feng Chen, Li~Ding, and Benjamin~J Raphael.
\newblock Dest-ot: Alignment of spatiotemporal transcriptomics data.
\newblock {\em Cell Systems}, 16(2), 2025.

\bibitem{bai2024efficient}
Yikun Bai, Rocio Diaz~Martin, Abihith Kothapalli, Hengrong Du, Xinran Liu, and Soheil Kolouri.
\newblock Partial gromov-wasserstein metric.
\newblock {\em arXiv preprint arXiv:2402.03664}, 2024.

\bibitem{sejourne2023unbalanced}
Thibault S{\'e}journ{\'e}, Cl{\'e}ment Bonet, Kilian Fatras, Kimia Nadjahi, and Nicolas Courty.
\newblock Unbalanced optimal transport meets sliced-wasserstein.
\newblock {\em arXiv preprint arXiv:2306.07176}, 2023.

\bibitem{titouan2019optimal}
Vayer Titouan, Nicolas Courty, Romain Tavenard, and R{\'e}mi Flamary.
\newblock Optimal transport for structured data with application on graphs.
\newblock In {\em International Conference on Machine Learning}, pages 6275--6284. PMLR, 2019.

\bibitem{liu2023partial}
Xinhao Liu, Ron Zeira, and Benjamin~J Raphael.
\newblock Partial alignment of multislice spatially resolved transcriptomics data.
\newblock {\em Genome Research}, 33(7):1124--1132, 2023.

\bibitem{pan2024subgraph}
Wen-Xin Pan, Isabel Haasler, and Pascal Frossard.
\newblock Subgraph matching via partial optimal transport.
\newblock In {\em 2024 IEEE International Symposium on Information Theory (ISIT)}, pages 3456--3461. IEEE, 2024.

\bibitem{bonneel2011displacement}
Nicolas Bonneel, Michiel Van De~Panne, Sylvain Paris, and Wolfgang Heidrich.
\newblock Displacement interpolation using lagrangian mass transport.
\newblock In {\em Proceedings of the 2011 SIGGRAPH Asia conference}, pages 1--12, 2011.

\bibitem{cuturi2014fast}
Marco Cuturi and Arnaud Doucet.
\newblock Fast computation of wasserstein barycenters.
\newblock In {\em International conference on machine learning}, pages 685--693. PMLR, 2014.

\bibitem{feragen2013scalable}
Aasa Feragen, Niklas Kasenburg, Jens Petersen, Marleen de~Bruijne, and Karsten Borgwardt.
\newblock Scalable kernels for graphs with continuous attributes.
\newblock {\em Advances in neural information processing systems}, 26, 2013.

\bibitem{borgwardt2005shortest}
Karsten~M Borgwardt and Hans-Peter Kriegel.
\newblock Shortest-path kernels on graphs.
\newblock In {\em Fifth IEEE international conference on data mining (ICDM'05)}, pages 8--pp. IEEE, 2005.

\bibitem{kriege2016valid}
Nils~M Kriege, Pierre-Louis Giscard, and Richard Wilson.
\newblock On valid optimal assignment kernels and applications to graph classification.
\newblock {\em Advances in neural information processing systems}, 29, 2016.

\bibitem{sutherland2003spline}
Jeffrey~J Sutherland, Lee~A O'brien, and Donald~F Weaver.
\newblock Spline-fitting with a genetic algorithm: A method for developing classification structure- activity relationships.
\newblock {\em Journal of chemical information and computer sciences}, 43(6):1906--1915, 2003.

\bibitem{8443159}
Si~Zhang and Hanghang Tong.
\newblock Attributed network alignment: Problem definitions and fast solutions.
\newblock {\em IEEE Transactions on Knowledge and Data Engineering}, 31(9):1680--1692, 2019.

\bibitem{chowdhury2021generalized}
Samir Chowdhury and Tom Needham.
\newblock Generalized spectral clustering via gromov-wasserstein learning.
\newblock In {\em International Conference on Artificial Intelligence and Statistics}, pages 712--720. PMLR, 2021.

\bibitem{Solomon2016Entropic}
Justin Solomon, Gabriel Peyr{\'e}, Vladimir Kim, and Suvrit Sra.
\newblock Entropic metric alignment for correspondence problems.
\newblock {\em ACM Transactions on Graphics (Proc. SIGGRAPH 2016)}, 35(4):72:1--72:13, 2016.

\bibitem{Xu2019Gromov}
Hongteng Xu, Dixin Luo, Hongyuan Zha, and Lawrence Carin.
\newblock Gromov-wasserstein learning for graph matching and node embedding.
\newblock In {\em Proceedings of the International Conference on Machine Learning (ICML)}, volume~97, pages 6332--6341, 2019.

\bibitem{Li2023Convergent}
Jiajin Li, Jianheng Tang, Lemin Kong, Huikang Liu, Jia Li, Anthony Man-Cho So, and Jose~H. Blanchet.
\newblock A convergent single-loop algorithm for relaxation of gromov-wasserstein in graph data.
\newblock In {\em International Conference on Learning Representations (ICLR)}, 2023.
\newblock arXiv:2303.06595.

\bibitem{VincentCuaz2022Semi}
Cédric Vincent-Cuaz, Rémi Flamary, Marco Corneli, Titouan Vayer, and Nicolas Courty.
\newblock Semi‐relaxed gromov–wasserstein divergence with applications on graphs.
\newblock In {\em International Conference on Learning Representations (ICLR)}, 2022.
\newblock arXiv:2110.02753.

\bibitem{DBLP:journals/corr/RodolaCBTC15}
Emanuele Rodol{\`{a}}, Luca Cosmo, Michael~M. Bronstein, Andrea Torsello, and Daniel Cremers.
\newblock Partial functional correspondence.
\newblock {\em CoRR}, abs/1506.05274, 2015.

\bibitem{frank1956algorithm}
Marguerite Frank, Philip Wolfe, et~al.
\newblock An algorithm for quadratic programming.
\newblock {\em Naval research logistics quarterly}, 3(1-2):95--110, 1956.

\bibitem{vishwanathan2010graph}
S~Vichy~N Vishwanathan, Nicol~N Schraudolph, Risi Kondor, and Karsten~M Borgwardt.
\newblock Graph kernels.
\newblock {\em The Journal of Machine Learning Research}, 11:1201--1242, 2010.

\bibitem{villani2021topics}
C{\'e}dric Villani.
\newblock {\em Topics in optimal transportation}, volume~58.
\newblock American Mathematical Soc., 2021.

\bibitem{figalli2010new}
Alessio Figalli and Nicola Gigli.
\newblock A new transportation distance between non-negative measures, with applications to gradients flows with dirichlet boundary conditions.
\newblock {\em Journal de math{\'e}matiques pures et appliqu{\'e}es}, 94(2):107--130, 2010.

\bibitem{piccoli2014generalized}
Benedetto Piccoli and Francesco Rossi.
\newblock Generalized wasserstein distance and its application to transport equations with source.
\newblock {\em Archive for Rational Mechanics and Analysis}, 211:335--358, 2014.

\bibitem{liu2023ptlp}
Xinran Liu, Yikun Bai, Huy Tran, Zhanqi Zhu, Matthew Thorpe, and Soheil Kolouri.
\newblock Ptlp: Partial transport $ l^p $ distances.
\newblock In {\em NeurIPS 2023 Workshop Optimal Transport and Machine Learning}, 2023.

\bibitem{caffarelli2010free}
Luis~A Caffarelli and Robert~J McCann.
\newblock Free boundaries in optimal transport and monge-ampere obstacle problems.
\newblock {\em Annals of mathematics}, pages 673--730, 2010.

\bibitem{vay2019fgw}
Vayer Titouan, Nicolas Courty, Romain Tavenard, Chapel Laetitia, and R{\'e}mi Flamary.
\newblock Optimal transport for structured data with application on graphs.
\newblock In Kamalika Chaudhuri and Ruslan Salakhutdinov, editors, {\em Proceedings of the 36th International Conference on Machine Learning}, volume~97 of {\em Proceedings of Machine Learning Research}, pages 6275--6284, Long Beach, California, USA, 09--15 Jun 2019. PMLR.

\bibitem{santambrogio2015optimal}
Filippo Santambrogio.
\newblock Optimal transport for applied mathematicians.
\newblock {\em Birk{\"a}user, NY}, 55(58-63):94, 2015.

\bibitem{peyre2016gromov}
Gabriel Peyr{\'e}, Marco Cuturi, and Justin Solomon.
\newblock Gromov-wasserstein averaging of kernel and distance matrices.
\newblock In {\em International conference on machine learning}, pages 2664--2672. PMLR, 2016.

\bibitem{lacoste2016convergence}
Simon Lacoste-Julien.
\newblock Convergence rate of frank-wolfe for non-convex objectives.
\newblock {\em arXiv preprint arXiv:1607.00345}, 2016.

\bibitem{chizat2018scaling}
Lenaic Chizat, Gabriel Peyr{\'e}, Bernhard Schmitzer, and Fran{\c{c}}ois-Xavier Vialard.
\newblock Scaling algorithms for unbalanced optimal transport problems.
\newblock {\em Mathematics of Computation}, 87(314):2563--2609, 2018.

\bibitem{bai2024sinkhorn}
Yikun Bai.
\newblock Sinkhorn algorithms and linear programming solvers for optimal partial transport problems.
\newblock {\em arXiv preprint arXiv:2407.06481}, 2024.

\bibitem{benamou2015iterative}
Jean-David Benamou, Guillaume Carlier, Marco Cuturi, Luca Nenna, and Gabriel Peyr{\'e}.
\newblock Iterative bregman projections for regularized transportation problems.
\newblock {\em SIAM Journal on Scientific Computing}, 37(2):A1111--A1138, 2015.

\bibitem{debnath1991structure}
Asim~Kumar Debnath, Rosa~L Lopez~de Compadre, Gargi Debnath, Alan~J Shusterman, and Corwin Hansch.
\newblock Structure-activity relationship of mutagenic aromatic and heteroaromatic nitro compounds. correlation with molecular orbital energies and hydrophobicity.
\newblock {\em Journal of medicinal chemistry}, 34(2):786--797, 1991.

\bibitem{rossi2015network}
Ryan Rossi and Nesreen Ahmed.
\newblock The network data repository with interactive graph analytics and visualization.
\newblock In {\em Proceedings of the AAAI conference on artificial intelligence}, volume~29, 2015.

\bibitem{shervashidze2009efficient}
Nino Shervashidze, SVN Vishwanathan, Tobias Petri, Kurt Mehlhorn, and Karsten Borgwardt.
\newblock Efficient graphlet kernels for large graph comparison.
\newblock In {\em Artificial intelligence and statistics}, pages 488--495. PMLR, 2009.

\bibitem{kriege2018recognizing}
Nils~M Kriege, Matthias Fey, Denis Fisseler, Petra Mutzel, and Frank Weichert.
\newblock Recognizing cuneiform signs using graph based methods.
\newblock In {\em International Workshop on Cost-Sensitive Learning}, pages 31--44. PMLR, 2018.

\bibitem{da2012tree}
Giovanni Da~San~Martino, Nicolo Navarin, and Alessandro Sperduti.
\newblock A tree-based kernel for graphs.
\newblock In {\em Proceedings of the 2012 SIAM International Conference on Data Mining}, pages 975--986. SIAM, 2012.

\bibitem{sugiyama2015halting}
Mahito Sugiyama and Karsten Borgwardt.
\newblock Halting in random walk kernels.
\newblock {\em Advances in neural information processing systems}, 28, 2015.

\bibitem{lovasz1979shannon}
L{\'a}szl{\'o} Lov{\'a}sz.
\newblock On the shannon capacity of a graph.
\newblock {\em IEEE Transactions on Information theory}, 25(1):1--7, 1979.

\bibitem{jethava2013lovasz}
Vinay Jethava, Anders Martinsson, Chiranjib Bhattacharyya, and Devdatt Dubhashi.
\newblock Lov{\'a}sz $\vartheta$ function, svms and finding dense subgraphs.
\newblock {\em The Journal of Machine Learning Research}, 14(1):3495--3536, 2013.

\bibitem{neumann2016propagation}
Marion Neumann, Roman Garnett, Christian Bauckhage, and Kristian Kersting.
\newblock Propagation kernels: efficient graph kernels from propagated information.
\newblock {\em Machine learning}, 102:209--245, 2016.

\bibitem{beier2022linear}
Florian Beier, Robert Beinert, and Gabriele Steidl.
\newblock On a linear gromov--wasserstein distance.
\newblock {\em IEEE Transactions on Image Processing}, 31:7292--7305, 2022.

\bibitem{flamary2021pot}
R{\'e}mi Flamary, Nicolas Courty, Alexandre Gramfort, Mokhtar~Z. Alaya, Aur{\'e}lie Boisbunon, Stanislas Chambon, Laetitia Chapel, Adrien Corenflos, Kilian Fatras, Nemo Fournier, L{\'e}o Gautheron, Nathalie~T.H. Gayraud, Hicham Janati, Alain Rakotomamonjy, Ievgen Redko, Antoine Rolet, Antony Schutz, Vivien Seguy, Danica~J. Sutherland, Romain Tavenard, Alexander Tong, and Titouan Vayer.
\newblock Pot: Python optimal transport.
\newblock {\em Journal of Machine Learning Research}, 22(78):1--8, 2021.

\end{thebibliography}
\bibliographystyle{unsrt}

\newpage
\appendix
\onecolumn
\section{Notation and Abbreviations.}

\paragraph*{Abbreviations}
\begin{itemize}
\item OT: Optimal Transport problem; see~\eqref{eq: OT}.
\item POT: Partial Optimal Transport problem; see~\eqref{eq:pot}.
\item GW: Gromov--Wasserstein problem; see~\eqref{eq:gw}.
\item PGW: Partial Gromov--Wasserstein problem; see~\eqref{eq:PGW}.
\item MPGW: Mass-Constrained Partial Gromov--Wasserstein; see~\eqref{eq:MPGW}.
\item FW: Frank--Wolfe algorithm~\cite{frank1956algorithm}.

\item $EFUGW(\mathbb{X},\mathbb{Y})$: entropic fused unbalanced GW objective (cf.\ \eqref{eq:entropy-fugw}).
\item $EFPGW(\mathbb{X},\mathbb{Y})$: entropic fused \emph{partial} GW objective (cf.\ \eqref{eq:entropy-fpgw}).
\item $LB$-$FPGW_\lambda(\mathbb{X},\mathbb{Y})$: relaxed lower-bound formulation (cf.\ \eqref{eq:entropy-fpgw-2}).
\item $EFMPGW(\mathbb{X},\mathbb{Y})$: entropic fused mass-constrained PGW (cf.\ \eqref{eq:entropic-fmpgw}); $LB$-$EFMPGW$ its relaxation.
\item $c(x,y)$: feature cost; $d_X,d_Y$: intra-space distances; $|d_X-d_Y|^2$: structural discrepancy.

\end{itemize}

\paragraph*{Sets, Spaces, and Indices}
\begin{itemize}
\item $\mathbb{R}^d$: Euclidean space; $\mathbb{R}_+$: nonnegative reals.
\item $X,Y\subset \mathbb{R}^d$: non-empty, convex, compact sets (default setting).
\item $X^2=X^{\otimes 2}=X\times X$.
\item $[1\!:\!n]=\{1,\dots,n\}$.
\item $r,p,q\in[1,\infty)$: real exponents used in costs/metrics.
\end{itemize}

\paragraph*{Vector, Norms and Basic Operators}
\begin{itemize}
\item $\|\cdot\|$: Euclidean norm; $\|\cdot\|_{\rm TV}$: total variation norm.
\item $\langle A,B\rangle=\mathrm{tr}(A^\top B)$: Frobenius inner product.
\item $1_n$, $1_{n\times m}$, $1_{n\times m\times n\times m}$: all-ones vector, matrix, and tensor.
\item $\mathbbm{1}_E$: indicator of a measurable set $E$,
$\mathbbm{1}_E(z)=1$ if $z\in E$, else $0$.
\item $\nabla$: gradient.
\end{itemize}

\paragraph*{Measures and Pushforwards}
\begin{itemize}
\item $\mathcal{M}_+(X)$: finite nonnegative Radon measures on $X$; 
\item $\mathcal{P}_2(X)$: probability measures on $X$ with finite second moment.
\item $|\mu|=\mu(X)$: total mass (TV-norm) of $\mu$.
\item $\mu\le \sigma$: measure domination, i.e., $\mu(B)\le \sigma(B)$ for all Borel $B\subseteq X$.
\item $\mu^{\otimes 2}=\mu\otimes \mu$: product measure.
\item $\mu(\phi):=\langle \phi,\mu\rangle:= \int \phi(x)\,d\mu(x)$.
\item $T_\# \sigma$: pushforward of $\sigma$ by measurable $T\!:\!X\!\to\!Y$, i.e., $T_\#\sigma(A)=\sigma(T^{-1}(A))$.
\end{itemize}
\paragraph*{Metric-Measure (mm) Spaces and Distance Matrices}
\begin{itemize}
\item $\mathbb{X}=(X,d_X,\mu)$, $\mathbb{Y}=(Y,d_Y,\nu)$: mm-spaces.
\item For discrete $\mathbb{X}$ with $X=\{x_i\}_{i=1}^n$, define $C^X\in\mathbb{R}^{n\times n}$ by $C^{X}_{i,i'}=d_X^q(x_i,x_{i'})$; similarly $C^Y$.
\item $\mathbb{X}\sim \mathbb{Y}$: mm-space equivalence if they have equal total mass and $GW_q^p(\mathbb{X},\mathbb{Y})=0$.
\end{itemize}
\paragraph*{Couplings and Partial Couplings}
\begin{itemize}
\item $\Gamma(\mu,\nu):=\{\gamma\in\mathcal{P}_2(X\times Y):\ \gamma_1=\mu,\ \gamma_2=\nu\}$.
\item Discrete weights: $\mathrm{p}=[p_1^X,\dots,p_n^X]^\top$, $\mathrm{q}=[q_1^Y,\dots,q_m^Y]^\top$,
\quad $|\mathrm{p}|=\sum_i p_i$, \quad $\mathrm{p}\le \mathrm{p}'$ if $p_j\le p'_j\ \forall j$.
\item $\Gamma(\mathrm{p},\mathrm{q})=\{\gamma\in\mathbb{R}_+^{n\times m}:\ \gamma\,1_m=\mathrm{p},\ \gamma^\top 1_n=\mathrm{q}\}$.
\item $\Gamma_{\le}(\mu,\nu):=\{\gamma\in\mathcal{M}_+(X\times Y):\ \gamma_1\le \mu,\ \gamma_2\le \nu\}$.
\item $\Gamma_{\le}(\mathrm{p},\mathrm{q}):=\{\gamma\in\mathbb{R}_+^{n\times m}:\ \gamma\,1_m\le \mathrm{p},\ \gamma^\top 1_n\le \mathrm{q}\}$.
\item $\gamma,\gamma_1,\gamma_2$: joint and marginal measures; in discrete form $\gamma\in\mathbb{R}_+^{n\times m}$, $\gamma_1=\gamma 1_m$, $\gamma_2=\gamma^\top 1_n$.
\item $\pi_1:X\times Y\to X$, $\pi_2:X\times Y\to Y$.
\item $\pi_{1,2}:S\times X\times Y\to X\times Y$, $(s,x,y)\mapsto (x,y)$; similarly $\pi_{0,1},\pi_{0,2}$ for other coordinate pairs.
\end{itemize}

\paragraph*{Optimal Transport Problems}
\begin{itemize}
\item $c:X\times Y\to \mathbb{R}_+$: lower semicontinuous cost for (partial) OT.
\item $OT(\mu,\nu)$: classical OT; $W_p(\mu,\nu)$: $p$-Wasserstein distance; $POT_\lambda(\mu,\nu)$: partial OT with parameter $\lambda>0$; see~\eqref{eq: OT}, \eqref{eq: Wp}, \eqref{eq:pot}.
\item $L:\mathbb{R}\times\mathbb{R}\to\mathbb{R}$, $D:\mathbb{R}\times\mathbb{R}\to\mathbb{R}$: GW loss and scalar distance.
\item $GW^L(\cdot,\cdot)$: GW objective with loss $L$; $d_{GW,r}^q$: GW with $L(a,b)=|a-b|^q$; see~\eqref{eq:gw}.
\item $PGW_{r,L,\lambda}(\cdot,\cdot)$: partial GW objective; see~\eqref{eq:PGW}.
\item $C(\gamma;\lambda,\mu,\nu):=\gamma^{\otimes 2}\!\big(L(d_X^q,d_Y^q)\big)+\lambda\big(|\mu|^2+|\nu|^2-2|\gamma|^2\big)$: PGW transport cost for $\gamma\in\Gamma_{\le}(\mu,\nu)$.
\item $UGW_{r,L,\lambda}(\mathbb{X},\mathbb{Y})$: unbalanced GW; $FUGW_{r,L,\lambda}(\mathbb{X},\mathbb{Y})$: fused unbalanced GW.
\item $d_{FGW,r,q}(\mathbb{X},\mathbb{Y})$: fused GW distance; $d_{FPGW,r,q,\lambda}^p(\mathbb{X},\mathbb{Y})$: fused partial GW distance.
\end{itemize}

\paragraph*{Discrete Tensorized Forms}
\begin{itemize}
\item Discrete measure setting: $\mu=\sum_{i=1}^n \mathrm{p}_i\delta_{x_i},\nu=\sum_{j=1}^m \mathrm{q}_j\delta_{y_j}$. 
\item $M\in\mathbb{R}^{n\times m\times n\times m}$ with $M_{i,j,i',j'}=L(C^X_{i,i'},C^Y_{j,j'})$; $M^\top$ swaps index pairs: $M^\top_{i,j,i',j'}=M_{i',j',i,j}$.

In default, $L$ is L2 norm square, and we denote $M=\|C_X-C_Y\|^2$
\item $(M-2\lambda)_{i,j,i',j'}:=M_{i,j,i',j'}-2\lambda$.
\item $M\circ \gamma \in \mathbb{R}^{n\times m}$ with $[M\circ \gamma]_{i,j}=\sum_{i',j'} M_{i,j,i',j'}\,\gamma_{i',j'}$.
\item $\langle\cdot,\cdot\rangle_F$: Frobenius inner product on $\mathbb{R}^{n\times m}$.
\end{itemize}

\paragraph*{Frank-Wolfe Optimization and Algorithmic Symbols}
\begin{itemize}
\item $\mathcal{L}$: objective functional for $PGW_\lambda(\cdot,\cdot)$.
\item $\alpha\in[0,1]$: line-search step size.
\item $\gamma^{(1)}$: initialization; $\gamma^{(k)}$, $\gamma^{(k)'}$: transport plans before/after Step~1 in the $k$-th FW iteration.
\item $G_{C,M}=\omega_1 C+\omega_2 \tilde M + M^\top\!\circ\! \gamma$, \quad $G=2\hat M\!\circ\! \hat \gamma$: gradients in two FW variants.
\item $a,b,c\in\mathbb{R}$: coefficients defined in~\eqref{eq:line_search_2} (For FPGW, replace $M$ by $M-2\lambda$).
\item $MPGW_{r,L,\rho}$ and $\Gamma^\rho_{\le}(\mu,\nu)$: mass-constrained partial GW and its feasible set; see~\eqref{eq:MPGW}, \eqref{eq:Gamma_leq^rho}.
\end{itemize}

\paragraph*{Entropic / Sinkhorn Notation}
\begin{itemize}
\item $\epsilon>0$: entropic regularization parameter.
\item $D_\phi(\cdot\Vert\cdot)$: $\phi$-divergence; special cases below.
\item $D_{KL}(\mu\Vert\nu)=\bar D_{KL}(\mu\Vert\nu)+|\nu|-|\mu|$ with
$\bar D_{KL}(\mu\Vert\nu)=\int \log\!\left(\frac{d\mu}{d\nu}\right)\,d\mu$ when $\mu\ll\nu$, $+\infty$ otherwise.
\item $D_{PTV}(\mu\Vert\nu)=\begin{cases} \|\mu-\nu\|_{\rm TV}=|\nu-\mu| & \text{if } \mu\le \nu,\\ +\infty & \text{otherwise.}\end{cases}$
\item $\Gamma_{\le}(\mu,\nu)$, $\Gamma_{\leq}^\rho(\mu,\nu)$: partial feasible sets (mass-unconstrained and mass-constrained with $|\gamma|=\rho$).
\item $|\gamma|=\gamma(X\times Y)$, $|\mu|=\mu(X)$, $|\nu|=\nu(Y)$: total mass.
\item $\mathcal{L}(\gamma)$: fused GW functional (feature + structure + penalties) used in EFUGW/EFPGW.

\item $c_\pi(x,y)=\tfrac{1}{2}\omega_1 d(x,y)+\omega_2\,[L(d_X^r,d_Y^r)\!\circ\!\pi](x,y)+\epsilon\,\bar D_{KL}(\pi\Vert \mu\!\otimes\!\nu)$: $\pi$-conditioned cost (cf.\ Prop.~\ref{pro:sinkhorn-fpgw}).
\item $[L(d_X,d_Y)\!\circ\!\pi](x,y)=\int_{X\times Y}L(d_X^r(x,x'),d_Y^r(y,y'))\,d\pi(x',y')$.
\item $K=\exp(-c/\epsilon)$: Gibbs kernel for feature cost; elementwise exponential.
\item $u\in\mathbb{R}^n_+$, $v\in\mathbb{R}^m_+$: Sinkhorn scaling vectors.
\item $\odot$, $\oslash$: elementwise (Hadamard) product and division.
\item $\mathrm{diag}(a)$: diagonal matrix with vector $a$ on the diagonal.
\item $\text{Proj}^{KL}_{\mathcal{C}_i}$: Bregman (KL) projection onto constraint set $\mathcal{C}_i$ (cf.\ Alg.~\ref{alg:sinkhorn-mopt}).
\item $\mathcal{C}_1=\{\gamma\ge0:\ \gamma_2\le q\}$,\; $\mathcal{C}_2=\{\gamma\ge0:\ \gamma_1\le p\}$,\;
$\mathcal{C}_3=\{\gamma\ge0:\ |\gamma|=\rho\}$: partial-marginal and mass constraints.
\item $\gamma^{\otimes 2}$, $(\mu\otimes\nu)^{\otimes 2}$: product measures on $(X\times Y)^2$.
\item $\pi,\gamma\in\mathcal{M}_+(X\times Y)$: current and updated couplings in alternating minimization; $\rho\in[0,\min(|p|,|q|)]$ target mass.
\item Stopping criteria: norms/duality gap on $(u,v)$ or fixed-point tolerance on $\gamma$.


\end{itemize}
\paragraph*{Graph-Specific Notation}
\begin{itemize}
\item $G=(V,E)$: graph with nodes $V=\{v_1,\dots,v_N\}$ and edges $E\subset V^2$.
\item $d_V:V^2\to\mathbb{R}$: structural distance (default: shortest-path distance).
\item $f:V\to \mathcal{F}$: node feature map; $\mathcal{F}$ is the feature space.
\item $wl:\mathcal{F}\to S^H$: Weisfeiler--Lehman feature map~\cite{vishwanathan2010graph} for discrete $\mathcal{F}$; $S$ finite alphabet, $H\in\mathbb{N}$.
\item $d_{\mathcal{F}}$: feature-space metric. Default: Euclidean if $\mathcal{F}=\mathbb{R}^d$; if $\mathcal{F}$ is discrete, Hamming on $wl(\mathcal{F})$:
\[
d_{\mathcal{F}}\!\big(f(v_1),f(v_2)\big):=\sum_{h=1}^H \delta\!\big(wl(f(v_1))_h \neq wl(f(v_2))_h\big).
\]
\end{itemize}

\section{Background of optimal transport problems}

\textbf{The Optimal Transport (OT) problems} for $\mu,\nu\in\mathcal{P}(\Omega)$,  with transportation cost $c(x,y): \Omega\times \Omega\to \mathbb{R}_+$ being a
lower-semi continuous function, is defined as: 
\begin{align}
&OT(\mu,\nu):=\min_{\gamma\in \Gamma(\mu,\nu)}\gamma(c), \label{eq: OT} \\
&\text{ where } \quad 
\gamma(c):=\int_{\Omega^2}c(x,y) \, d\gamma(x,y) \nonumber
\end{align}
and where $\Gamma(\mu,\nu)$ denotes the set of all joint probability measures on $\Omega^2:=\Omega\times \Omega$ with marginals $\mu,\nu$, i.e., 
$\gamma_1:=\pi_{1\#}\gamma=\mu,
 \gamma_2:= \pi_{2\#}\gamma=\nu$,
where $\pi_1,\pi_2:\Omega^2\to \Omega$ are the canonical projections  $\pi_1(x,y):=x,\pi_2(x,y):=y$. 
A minimizer for \eqref{eq: OT} always exists \cite{Villani2009Optimal,villani2021topics} and when $c(x,y)=\|x-y\|^p$, for $p\ge 1$, 
it defines a metric on $\mathcal{P}(\Omega)$, which is referred to as the ``$p$-Wasserstein distance'':
\begin{align}
W_p^p(\mu,\nu):=\min_{\gamma\in\Gamma(\mu,\nu)}\int_{\Omega^2}\|x-y\|^pd\gamma(x,y). \label{eq: Wp}
\end{align}
\textbf{The Partial Optimal Transport (POT) problem} \cite{chizat2018unbalanced, figalli2010new,piccoli2014generalized} extends the OT problem to the set of Radon measures $\mathcal{M}_+(\Omega)$, i.e., non-negative and finite measures. For $\lambda>0$ and $\mu,\nu\in \mathcal{M}_+(\Omega)$, the POT problem is defined as:
\begin{equation}
POT(\mu,\nu;\lambda):=\inf_{\gamma\in\mathcal{M}_+(\Omega^2)}\gamma(c)+\lambda(|\mu-\gamma_1|+|\nu-\gamma_2|),
\label{eq:pot}
\end{equation}
where, in general, $|\sigma|$ denotes the total variation norm of a measure $\sigma$, i.e., $|\sigma|:=\sigma(\Omega)$.  
The constraint $\gamma\in \mathcal{M}_+(\Omega^2)$ in \eqref{eq:pot} can be further restricted to $\gamma \in \Gamma_\leq(\mu,\nu)$: 
$$\Gamma_\leq(\mu,\nu):=\{\gamma\in \mathcal{M}_+(\Omega^2): \, \gamma_1\leq \mu,\gamma_2\leq \nu\},$$ denoting $\gamma_1\leq \mu$ if for any Borel set $B\subseteq\Omega$, $\gamma_1(B)\leq \mu(B)$ (respectively, for $\gamma_2\leq \nu$) \cite{figalli2010optimal}. Roughly speaking, the linear penalization indicates that if the classical transportation cost exceeds 
$2\lambda$, it is better to create/destroy' mass (see \cite{bai2022sliced} for further details). In addition, the above formulation has an equivalent form, to distinguish them, we call it ``mass-constraint partial optimal transport problem'' (MPOT): 
\begin{align}
MOPT(\mu,\nu;\rho):=\inf_{\gamma\in\Gamma_\leq^\rho(\mu,\nu)}\gamma(c) \label{eq:MPOT}.
\end{align}

\section{(fused)-Unbalanced Gromov Wasserstein and (fused)  Partial Gromov Wasserstein.}
The unbalanced Gromov Wasserstein problem is firstly proposed by \cite{sejourne2021unbalanced,chapel2020partial,bai2024efficient}. Given two $mm-$spaces, $\mathbb{X}=(X,d_X,\mu),\mathbb{Y}=(Y,d_Y,\nu)$, the UGW problem is defined by 
\begin{align}
UGW_{r,L,\lambda}(\mathbb{X},\mathbb{Y}):=\inf_{\gamma\in\mathcal{M}_+(X\times Y)}\langle L(d_X^r,d_Y^r), \gamma^{\otimes2} \rangle+\lambda(D_{\phi_1}(\gamma_1^{\otimes2}\parallel \mu^{\otimes2})+D_{\phi_2}(\gamma_2^{\otimes2}\parallel \nu^{\otimes2})),\label{eq:ugw}
\end{align}
where $D_{\phi_1}, D_{\phi_2}$ are f-divergence and can be distinct in the general setting. 

Similar to this work, the fused-Unbalanced Gromov Wasserstein, proposed by \cite{thual2022aligning} is defined as: 

$$FUGW_{r,L,\lambda}(\mathbb{X},\mathbb{Y}):=\inf_{\gamma\in\mathcal{M}_+(X\times Y)}\omega_1\langle C,\gamma\rangle+ \omega_2\langle L(d_X^r,d_Y^r), \gamma^{\otimes2} \rangle+\lambda(D_{\phi_1}(\gamma_1^{\otimes2}\parallel \mu^{\otimes2})+D_{\phi_2}(\gamma_2^{\otimes2}\parallel \nu^{\otimes2})).$$

And when the $f$-divergence terms $D_{\phi_1},D_{\phi_2}$ are KL divergences, the authors in \cite{thual2022aligning} propose a Sinkhorn computational solver. 

At the end of this section, we introduce the following relation between FUGW and FPGW, which is equivalent to the statement (1) in Theorem \ref{thm:main}. 
\begin{proposition}
When $D_{\phi_1},D_{\phi_2}$ are total variation, then we have: 
\begin{align}
FUGW_{r,L,\lambda}(\mathbb{X},\mathbb{Y})&=FPGW_{r,L,\lambda}(\mathbb{X},\mathbb{Y})\nonumber\\
&=\inf_{\gamma\in\Gamma_\leq(\mu,\nu)}\omega_1 \langle C,\gamma\rangle+\omega_2 \langle L(d_X^r,d_Y^r),\gamma^{\otimes2}\rangle +\lambda(|\mu|^2+|\nu|^2-2|\gamma|^2). \nonumber 
\end{align}
Equivalently speaking, one can restrict the searching space from $\mathcal{M}_+(X\times Y)$ to $\Gamma_\leq(\mu,\nu)$. 
\end{proposition}
\begin{proof}
Choose $\gamma\in\mathcal{M}_+(X\times Y)$ and let $\gamma_{Y|X}(\cdot|x)$ denote the conditional measure of $\gamma$ given $x\in X$. Let $\gamma'=\gamma_{Y|X}(\cdot|X)\cdot (\gamma_1\wedge\mu)$, by proposition L.2, we have 
$\gamma'\leq \gamma,\gamma'_1\leq \mu$, and 
$$\omega_2\langle\|d_X-d_Y\|^p,\gamma{'}^{\otimes 2}\rangle+\lambda(\|\mu-\gamma_1'\|_{TV}+\|\nu-\gamma_2'\|_{TV})\leq \omega_2\langle\|d_X-d_Y\|^p,\gamma^{\otimes 2}\rangle+\lambda(\|\mu-\gamma_1\|_{TV}+\|\nu-\gamma_2\|_{TV}).$$

Since $C\ge 0, \gamma'\leq \gamma$, we have 
$\langle C, \gamma'\rangle \leq \langle C, \gamma\rangle$. Therefore, we have 
\begin{align}
&\omega_1 \langle C,\gamma'\rangle+
\omega_2\langle\|d_X-d_Y\|^p,\gamma{'}^{\otimes 2}\rangle+\lambda(\|\mu-\gamma_1'\|_{TV}+\|\nu-\gamma_2'\|_{TV})\nonumber\\
&\leq \omega_1 \langle C,\gamma\rangle+
\omega_2\langle\|d_X-d_Y\|^p,\gamma^{\otimes 2}\rangle+\lambda(\|\mu-\gamma_1\|_{TV}+\|\nu-\gamma_2\|_{TV})\nonumber.
\end{align}
That is, based on $\gamma$, we can construct $\gamma'$ such that $\gamma'_1\leq \mu\wedge \gamma_1,\gamma'\leq \gamma$ and the $\gamma'$ admits smaller cost.
Based on the same process, based on $\gamma'$, we can construct  $ \gamma''$ such that  $\gamma''\leq \gamma'$, $\gamma''\leq \nu$. That is $\gamma''\in\Gamma_\leq(\mu,\nu)$ and admits smaller cost than $\gamma'$. Therefore, we can restrict the searching space for \eqref{eq:fugw} in this case from $\mathcal{M}_+(X\times Y)$ to $\Gamma_\leq(X\times Y)$. 
\end{proof}

 \begin{figure}
    \centering
\includegraphics[width=0.7\linewidth]{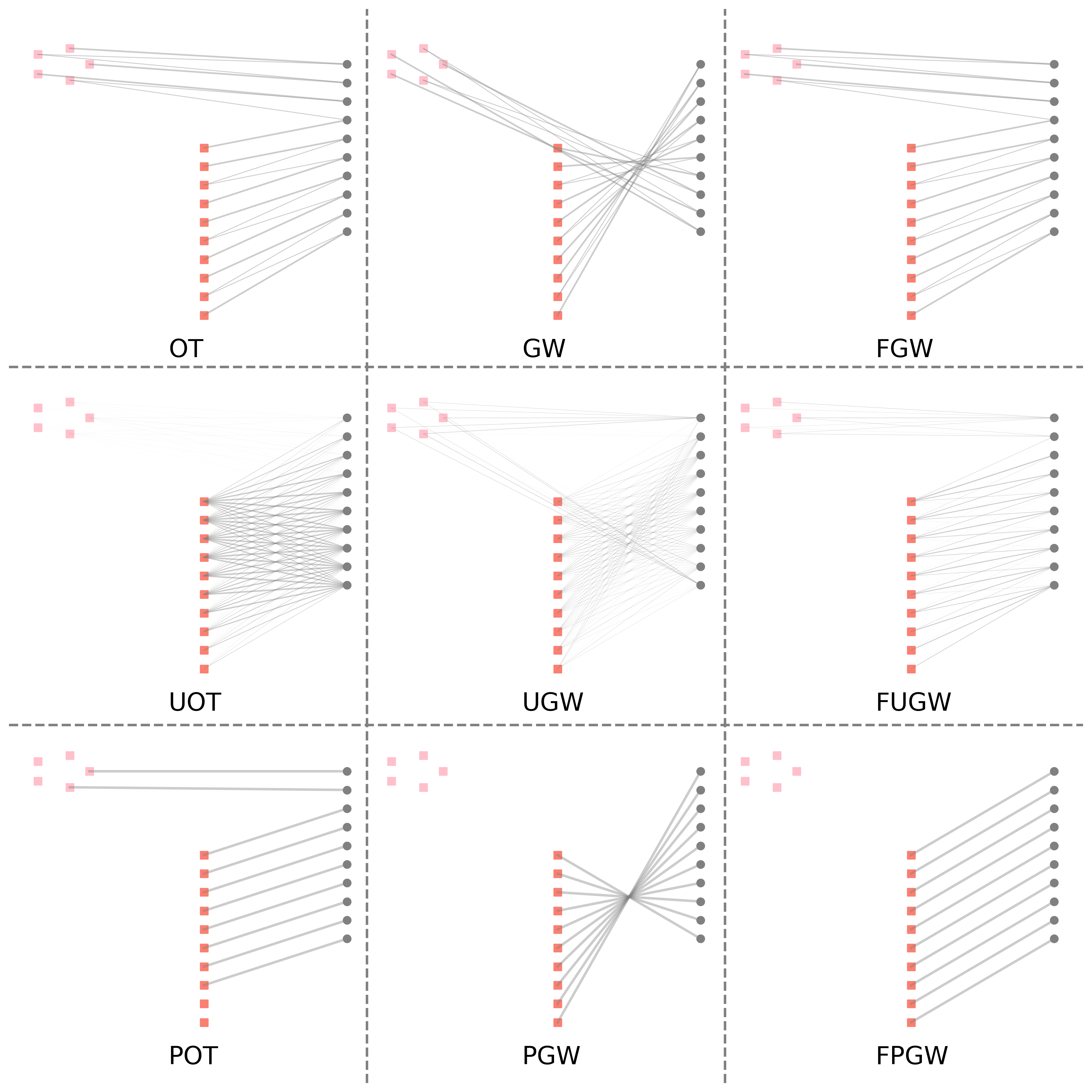}
    \vspace{-0.2in}
    \caption{In this toy example, we illustrate the relationships and differences between OT, Unbalanced OT, Partial OT, GW, Unbalanced GW, Partial GW, fused GW, fused unbalanced GW, and Fused Partial GW. The source shape consists of the union of the pink and red points, and the target shape is the grey shape. 
    The objective is to establish a reasonable matching between the two shapes. }
    \label{fig:toy_example}
\end{figure}

\begin{remark}
In Figure~\ref{fig:toy_example}, we illustrate the transportation plans of various methods, including optimal transport (OT), unbalanced optimal transport (UOT), partial optimal transport (POT), $\ldots$ and fused-Partial Gromov-Wasserstein (FPGW). 

Due to the balanced mass setting, OT, GW, and fused GW match all points from the source shape to the target shape. In contrast, the Sinkhorn entropic regularization in UOT, UGW, and FUGW allows for mass splitting, resulting in a small amount of mass between the grey shape and the pink points. PGW and fused PGW achieve the desired one-to-one matching, where the straight-line segment in the source shape (red points) corresponds to the straight-line segment in the target shape (grey points). However, since PGW does not account for the spatial location of points, there is a 50\% chance of obtaining ``anti-identity'' matching as demonstrated in Figure~\ref{fig:toy_example}.
\end{remark}

\subsection{Existence of minimizes of Fused PGW.}
In this section, we discuss the basic properties of the fused PGW. 

\begin{proposition}\label{pro:fpgw_minimizer}
Given mm-spaces $\mathbb{X}=(X,d_X,\mu),\mathbb{Y}=(Y,d_Y,\nu)$, where $X,Y$ are compact sets, the fused-PGW problems \eqref{eq:fpgw} and \eqref{eq:fmpgw} admit minimizers. That is, we can replace $\inf$ by $\min$ in the formulations.  
\end{proposition}
Note, this proposition is equivalent to the statement (2) in Theorem \ref{thm:main}.

We first introduce the following statements: 
\begin{proposition}\label{pro:compact}
The set $\Gamma_\leq(\mu,\nu)$ and $\Gamma_\leq^\rho(\mu,\nu)$ are weakly compact, convex sets, where $\rho\in[0,\min(|\mu|,|\nu|)]$.  
\end{proposition}
\begin{proof}
By \cite{liu2023ptlp} Lemma B.1, we have $\Gamma_\leq(\mu,\nu)$ is a weakly compact set. By e.g. \cite{caffarelli2010free,bai2022sliced}, we have $\Gamma_\leq(\mu,\nu)$ is convex. 

It remains to show the compactness and convexity of set $\Gamma^\rho_\leq(\mu,\nu)$. 

Pick  $\gamma^1,\gamma^2\in\Gamma^\rho_\leq(\mu,\nu)$ and $\omega\in[0,1]$ and let $\gamma=(1-\omega)\gamma^1+\omega\gamma^2$.
We have $\gamma\in\Gamma_\leq(\mu,\nu)$ by the convexity of $\Gamma_\leq(\mu,\nu)$. In addition, $$|(1-\omega)\gamma^1+\omega\gamma^2|=(1-\omega)|\gamma^1|+\omega|\gamma^2|\ge (1-\omega)\rho+\omega\rho=\rho,$$ 
thus we have: $\gamma\in\Gamma_\leq^\rho(\mu,\nu)$. 

Pick sequence $(\gamma^n)_{n=1}^\infty\subset \Gamma^\rho_\leq(\mu,\nu)\subset\Gamma_\leq(\mu,\nu)$. By the compactness of $\Gamma_\leq(\mu,\nu)$, there exists a subsequence $(\gamma^{n_k})_{k=1}^\infty$ that converges to some $\gamma^*\in\Gamma_\leq(\mu,\nu)$ weakly. It remains to show $\gamma^*\in\Gamma_\leq^\rho(\mu,\nu)$. 
Indeed, 
$$|\gamma^*|=\langle 1_{X\times Y}\times \gamma\rangle =\lim_{k\to\infty}\langle 1_{X\times Y},\gamma^{n_k} \rangle=\lim_{k\to\infty}|\gamma^{n_k}|.$$
Since for each $k\ge 1$, we have $|\gamma^{n_k}|\ge \rho$, thus $|\gamma^*|\ge \rho$. That is $\gamma^*\in\Gamma_\leq^\rho(\mu,\nu)$, and we complete the proof. 
\end{proof}

\begin{lemma}\label{lem:inf}
Suppose $X, Y$ are compact sets. Let $Z=(X\times Y)$ and $d_Z((x^1,y^1),(x^2,y^2))=d_X(x^1,x^2)+d_Y(y^1,y^2)$. Suppose $\phi: Z^2\to \mathbb{R}$ is Lipschitz continuous and $C: Z\to \mathbb{R}$ is bounded, then the following problem admits a minimizer: 
\begin{align}
\inf_{\gamma\in\mathcal{M}}\langle \phi,\gamma^{\otimes2} \rangle+\langle C,\gamma \rangle \label{eq:inf_prob}
\end{align}
where $\mathcal{M}$ is a non-empty weakly compact set. 
\end{lemma}
\begin{proof}
Pick weakly convergent $(\gamma^n)_{n=1}^\infty\subset \mathcal{M}$, we have
$\gamma^n\overset{w}{\rightharpoonup}\gamma^*\in\mathcal{M}$.

From lemma C.3 in \cite{bai2024efficient}, we have $$\langle \phi,(\gamma^{n})^{\otimes2}\rangle\to \langle \phi,\gamma^* \rangle. $$
By definition of weak convergence and the fact that the sets $X, Y$ are compact, we have 
\begin{align}
  \langle C,\gamma^n \rangle\to \langle C,\gamma \rangle.\label{pf:convergence}  
\end{align}
Thus, we have 
$$\langle \phi,(\gamma^n)^{\otimes2} \rangle+\langle C,\gamma^n\rangle\to \langle \phi,\gamma^{\otimes2} \rangle+\langle C,\gamma \rangle.$$

Choose a sequence $\gamma^n\in \mathcal{M}$ such that 
$\langle\phi,(\gamma^n)^{\otimes2} \rangle+\langle C,\gamma^n \rangle$ achieves the infimum of the problem \eqref{eq:inf_prob}. By compactness of $\mathcal{M}$, there exists a convergent subsequence $\gamma^{(n_k)}\overset{w}{\rightharpoonup}\gamma^*\in\mathcal{M}$.  By \eqref{pf:convergence}, we have $\gamma^*$ is a minimizer of \eqref{eq:inf_prob}. Thus, we complete the proof.  
\end{proof}

\begin{proof}[Proof of Proposition \ref{pro:fpgw_minimizer}.]
From Lemma C.1. in \cite{bai2024efficient}, we have $$(X\times Y)^2\ni ((x^1,y^1),(x^2,y^2))\to |d_X^r(x^1,x^2)-d_Y^r(y^1,y^2)|^q$$ is Lipschitz continuous. In addition, $C: X\times Y$ is a bounded mapping. From lemma \ref{pro:compact}, we have $\Gamma_\leq(\mu,\nu),\Gamma_\leq^\rho(\mu,\nu)$ are compact. From lemma \ref{lem:inf}, we complete the proof.  
\end{proof}

\section{Metric property of Fused PGW}
In this section, we discuss the proof of Theorem  \ref{thm:main} (3).
We will discuss the details in the following subsections. The related conclusion can be treated as an extension of Theorem 6.3 in \cite{vay2019fgw} and Theorem 3.1 in \cite{vayer2020fused}.
\subsection{Background: Isometry, fused-GW semi-metric.}

Given $\mathbb{X}=(X,d_X,\mu),\mathbb{Y}=(Y,d_Y,\nu)$, we say $X,Y$ are equivalent, noted $X\sim Y$ if and only if: 

There exists a mapping $\phi: X\to Y$ such that
\begin{itemize} 
    \item $\phi_\#\mu=\nu$
    \item $d_X(x,x')=d_Y(\phi(x),\phi(x')),\forall x,x'\in X$. 
\end{itemize}
Such a function $\phi$ is called \textbf{measure preserving isometry}.

In addition, in the formulation of fused-GW \eqref{eq:fpgw}, we set $d(x,y)=\|x-y\|^q$ and $L(\mathrm{r}_1,\mathrm{r}_2)=|\mathrm{r}_1-\mathrm{r}_2|^q$. The reduced formulation is called ``fused-Gromov Wasserstein distance'' \cite{vayer2020fused}:

\begin{align}
d_{FGW,r,q}(\mathbb{X},\mathbb{Y}):=
\inf_{\gamma\in\Gamma(\mu,\nu)}
\int_{(X\times Y)^2}\omega_1\frac{\|x-y\|^q}{|\gamma|}+\omega_2|d_X^r(x,x')-d_Y^r(y,y')|^q d\gamma(x,y)d\gamma(x',y').
\label{eq:fgw_metric}
\end{align}
and it defines a metric where the identity is induced by the above equivalence relation. 

Inspired by this formulation, we introduce the \textbf{Fused Partial GW metric}: 
{\footnotesize
\begin{align}
d_{FPGW,r,q,\lambda}^p(\mathbb{X},\mathbb{Y}):=
\inf_{\gamma\in\Gamma_\leq(\mu,\nu)}
\int_{(X\times Y)^2}\omega_1\frac{\|x-y\|^q}{|\gamma|}+\omega_2|d_X^r(x,x')-d_Y^r(y,y')|^q+\lambda\left(\frac{|\mu|^2}{|\gamma|^2}+\frac{|\nu|^2}{|\gamma|^2}-2\right) d\gamma(x,y)d\gamma(x',y').
\label{eq:fpgw_metric}
\end{align}
}

\begin{remark}
We adapt the convention $\frac{1}{0}\cdot0=1$, thus the above formulation is
well-defined. In particular, when $|\gamma|=0$, i.e., $\gamma$ is zero measure,
the above integration is defined by 
\begin{align}
&\int_{(X\times Y)^2}\omega_1\frac{\|x-y\|^q}{|\gamma|}+\omega_2|d_X^r(x,x')-d_Y^r(y,y')|^q+\lambda\left(\frac{|\mu|^2}{|\gamma|^2}+\frac{|\nu|^2}{|\gamma|^2}-2\right) d\gamma(x,y)d\gamma(x',y')\nonumber\\
&=\lambda(|\mu|^2+|\nu|^2)\label{eq:fpgw_gamma0}\\
&=\lim_{|\gamma|\searrow 0}\int_{(X\times Y)^2}\omega_1\frac{\|x-y\|^q}{|\gamma|}+\omega_2|d_X^r(x,x')-d_Y^r(y,y')|^q+\lambda\left(\frac{|\mu|^2}{|\gamma|^2}+\frac{|\nu|^2}{|\gamma|^2}-2\right) d\gamma(x,y)d\gamma(x',y').\nonumber
\end{align}
\end{remark}
\begin{remark}
By Proposition \ref{pro:fpgw_minimizer}, the above problem admits a minimizer. 
\end{remark}

Next, we introduce the formal statement of Theorem \ref{thm:main} (2). 

\begin{theorem}\label{thm:metric_formal}
Define space for mm-spaces, 
$$\mathcal{G}=\{\mathbb{X}=(X,d_X,\mu), X\subset\mathbb{R}^d, X \text{ is compact}; d_X \text{ is a metric}; \mu\in\mathcal{M}_+(X) \}.$$
 
Then Fused PGW \eqref{eq:fpgw_metric}  defines a semi-metric in quotient space $\mathcal{G}/\sim$. In particular: 
\begin{enumerate}
    \item $d_{FPGW,r,q,\lambda}(\cdot,\cdot)$ is non-negative and symmetric. 
   
    \item  Suppose $\omega_2,\lambda>0$, then 
    $d_{FPGW,r,q,\lambda}(\mathbb{X},\mathbb{Y})=0$ iff $\mathbb{X}\sim\mathbb{Y}$. 
    \item If $\omega_2>0$, for $q\ge 1$, we have  \begin{align}
     d_{FPGW,r,q,\lambda}(\mathbb{X},\mathbb{Y})\leq 2^{q-1}(d_{FPGW,r,q,\lambda}(\mathbb{X},\mathbb{Z})+d_{FPGW,r,q,\lambda}(\mathbb{Y},\mathbb{Z})).\label{eq:triangle_ineq}   
    \end{align}
    In particular, when $q=1$, fused-PGW satisfies the triangle inequality. 
\end{enumerate}
\end{theorem}

\subsection{Proof of the (1)(2) in Theorem \ref{thm:metric_formal}}
For statement (1), by definition, for each $\gamma\in\Gamma_\leq(\mu,\nu)$, we have 
$$\int_{X\times Y} \|x-y\|^q d\gamma,\int_{(X\times Y)^2} \|d^r_X(x,x')-d^r_Y(y,y')\|^q d\gamma^{\otimes2}\ge 0.$$
Thus, $d_{FPGW,r,q,\lambda}(\mathbb{X},\mathbb{Y})\ge0$. 

In addition, 
\begin{align}
&d_{FPGW,r,q,\lambda}(\mathbb{X},\mathbb{Y})\nonumber\\ &=\inf_{\gamma\in\Gamma_\leq(\mu,\nu)}\int_{(X\times Y)^2}\omega_1\|x-y\|^q +\omega_2|d_X^r(x,x')-d_Y^r(y,y')|^q)^p d\gamma(x,y)d\gamma(x',y')+\lambda(|\mu|^2+|\nu|^2-2|\gamma|^2)\nonumber\\
&=\inf_{\gamma\in\Gamma_\leq(\nu,\mu)}\int_{(Y\times X)^2}\omega_1\|y-x\|^q +\omega_2|d_Y^r(y,y')-d_X^r(x,x')|^q)^p d\gamma(y,x)d\gamma(y',x')+\lambda(|\mu|^2+|\nu|^2-2|\gamma|^2)\nonumber\\
&=d_{FPGW,r,q,\lambda}(\mathbb{Y},\mathbb{X})\nonumber. 
\end{align}
And we prove the symmetric.

For statement (2), suppose measure preserving isometry $\phi: X\to Y$ exists, then we have $|\mu|=|\nu|$. Let $\gamma=(\text{id}\times\phi)_\#\mu$, we have $\gamma\in\Gamma(\mu,\nu)\subset \Gamma_\leq(\mu,\nu)$. Thus $|\gamma|=|\mu|=|\nu|$. 

Furthermore, 
\begin{align}
&d_{FPGW,r,q,\lambda}^p(\mathbb{X},\mathbb{Y})\nonumber\\ &\leq \int_{(X\times Y)^2}\omega_1\|x-y\|^q +\omega_2|d_X^r(x,x')-d_Y^r(y,y')|^q d\gamma(x,y)d\gamma(x',y')+\lambda(|\mu|^2+|\nu|^2-2|\gamma|^2)\nonumber\\
&=\int_{(X\times Y)^2}\omega_1\underbrace{\|x-\phi(x)\|^q}_{0} +\omega_2\underbrace{|d_X^r(x,x')-d_Y^r(\phi(x),\phi(x'))|^q}_{0} d\mu(x)d\mu(x')+\lambda(\underbrace{|\mu|^2+|\nu|^2-2|\gamma|^2}_0)\nonumber\\
&=0.
\end{align}

For the other direction, suppose 
$d_{FPGW,r,q,\lambda}(\mathbb{X},\mathbb{Y})=0$. We have two cases:

Case 1: $|\mu|=|\nu|=0$. We've done. 

Case 2: $|\mu|>0$ or  $|\nu|>0$. 

By the Proposition \ref{pro:fpgw_minimizer}, a minimizer for problem \eqref{eq:fpgw_metric} exists. Choose one minimizer, denoted as $\gamma^*$. 

First, we claim $|\mu|=|\nu|=|\gamma^*|$. 

Indeed, assume the above equation is not true. For convenience, we suppose $|\mu|<|\nu|$. Then $|\gamma|\leq |\mu|<|\nu|$. 
We have 
\begin{align}
0=d_{FPGW,r,q,\lambda}(\mathbb{X},\mathbb{Y})\ge \lambda(|\mu|^2+|\nu|^2-2|\gamma^*|^2)> \lambda(|\mu|^2-|\gamma^*|^2)\ge0, 
\end{align}
since $\lambda>0$. Thus we have contradiction. 

Since $\gamma^*\in\Gamma_\leq(\mu,\nu)$, we have $\gamma^*\in\Gamma(\mu,\nu)$. Thus, we have
\begin{align}
0&\leq d_{FPGW,r,q}(\mathbb{X},\mathbb{Y}) \nonumber\\
&\leq \int_{(X \times Y)^2} \omega_1 \frac{\|x-y\|^q}{|\gamma|} + \omega_2 |d_X^r(x,x') - d_Y^r(y,y')|^q \, d\gamma^*(x,y) d\gamma^*(x',y') \nonumber\\
&= \int_{(X \times Y)^2} \omega_1 \frac{\|x-y\|^q}{|\gamma|} + \omega_2 |d_X^r(x,x') - d_Y^r(y,y')|^q 
+ \lambda \left( \frac{|\mu|^2}{|\gamma|^2} + \frac{|\nu|^2}{|\gamma|^2} - 2 \right) 
\, d\gamma^*(x,y) d\gamma^*(x',y') \nonumber\\
&=d_{FPGW,r,q}(\mathbb{X},\mathbb{Y})\nonumber\\
&=0.\nonumber
\end{align}
That is, $d_{FPGW,r,q}(\mathbb{X},\mathbb{Y})=0$. Since $\omega_2>0$, by Proposition 5.2 in \cite{vayer2020fused} or Theorem 6.4 in \cite{vay2019fgw}, there exists measure preserving isometry $\phi: X\to Y$ and thus we have $\mathbb{X}\sim\mathbb{Y}$ and we complete the proof.

\subsection{Proof of statement (3) in Theorem \ref{thm:metric_formal}.}
Choose mm-spaces $\mathbb{S}=(S,d_S,\sigma),\mathbb{X}=(X,d_X,\mu),\mathbb{Y}=(Y,d_Y,\nu)$. In this section, we will prove the triangle inequality
$$d_{FPGW,r,q,\lambda}(\mathbb{X},\mathbb{Y})\leq d_{FPGW,r,q,\lambda}(\mathbb{S},\mathbb{X})+d_{FPGW,r,q,\lambda}(\mathbb{S},\mathbb{Y}).$$

First, introduce auxiliary points $\hat\infty_0,\hat\infty_1,\hat\infty_2$ and set 
\begin{align}
\begin{cases}
\hat{S}&=S\cup 
\{\hat\infty_0,\hat\infty_1,\hat\infty_2\},\nonumber\\
\hat{X}&=X\cup \{\hat\infty_0,\hat\infty_1,\hat\infty_2\},\nonumber \\
\hat{Y}&=Y\cup \{\hat\infty_0,\hat\infty_1,\hat\infty_2\}.\nonumber
\end{cases}
\end{align}
Define $\hat\sigma, \hat\mu,\hat\nu$ as follows: 
\begin{equation}
\begin{cases} 
\hat\sigma&=\sigma+|\mu|\delta_{\hat\infty_1}+|\nu|\delta_{\hat\infty_2},\\
\hat\mu&=\mu+|\sigma|\delta_{\hat\infty_0}+|\nu|\delta_{\hat\infty_2},\\
\hat\nu&=\nu+|\sigma|\delta_{\hat\infty_0}+|\mu|\delta_{\hat\infty_1}.
\end{cases}\label{pf:sigma_hat,mu_hat,nu_hat}    
\end{equation}

Next, we define $d_{\hat{S}}:\hat{S}^2\to \mathbb{R}\cup\{\infty\}$ as follows: 
\begin{align}
d_{\hat{S}}(s,s')=\begin{cases}
d_S(s,s') &\text{if }(s,s')\in S^2,\\
\infty &\text{elsewhere. }
\end{cases} \label{pf:d_S_hat}
\end{align}
Note, $d_{\hat{S}}(\cdot,\cdot)$ is not a rigorous metric in $\hat{S}$ since we allow $d_{\hat{S}}=\infty$, $d_{\hat{X}},d_{\hat{Y}}$ are defined similarly. 

Then, we define the following mapping $L_\lambda: (\mathbb{R}\cup\{\infty\})\times (\mathbb{R}\cup\{\infty\})\to \mathbb{R}_+$: 
\begin{align}\label{eq:D_lambda}
    L^q_{\lambda}(\mathrm{r}_1,\mathrm{r}_2)=
    \begin{cases}
    |\mathrm{r}_1-\mathrm{r}_2|^q &\text{if }\mathrm{r}_1,\mathrm{r}_2<\infty, \\ 
    \lambda/\omega_2  &\text{if }\mathrm{r}_1=\infty, \mathrm{r}_2<\infty \text{ or vice versa},\\
    0 &\text{if }\mathrm{r}_1=\mathrm{r}_2=\infty;
    \end{cases}
\end{align}

and mapping $\hat{D}
:\mathbb{R}^d\cup\{\hat{\infty}_i: i\in[0:2]\}\to \mathbb{R}$: 
\begin{align}
\hat{D}^q(x,y):=\begin{cases}
\|x-y\|^q &\text{if }x,y\in\mathbb{R}^d\\
0 &\text{elsewhere}. 
\end{cases}
\end{align}

Finally, we define the following mappings: 

\begin{align}
&\Gamma_\leq(\sigma,\mu)\ni \gamma^{01}\mapsto\hat\gamma^{01} \in \Gamma(\hat\sigma,\hat\mu), \nonumber\\
&\hat\gamma^{01}:=\gamma^{01}+(\sigma-\gamma_1^{01})\otimes\delta_{\hat\infty_0}+\delta_{\hat\infty_1}\otimes(\mu-\gamma_2^{01})+|\gamma|\delta_{\hat\infty_1,\hat\infty_0}+|\nu|\delta_{\hat\infty_2,\hat\infty_2};\nonumber \\
&\Gamma_\leq(\sigma,\nu)\ni \gamma^{02}\mapsto\hat\gamma^{02} \in \Gamma(\hat\sigma,\hat\nu), \nonumber\\
&\hat\gamma^{02}:=\gamma^{02}+(\sigma-\gamma_1^{02})\otimes\delta_{\hat\infty_0}+\delta_{\hat\infty_2}\otimes(\nu-\gamma_2^{02})+|\gamma|\delta_{\hat\infty_2,\hat\infty_0}+|\mu|\delta_{\hat\infty_1,\hat\infty_1}; \nonumber \\
&\Gamma_\leq(\mu,\nu)\ni \gamma^{12}\mapsto\hat\gamma^{12}\in\Gamma(\hat\mu,\hat\nu), \nonumber\\
&\hat\gamma^{12}:=\gamma^{12}+(\mu-\gamma_1^{12})\otimes\delta_{\hat\infty_1}+\delta_{\hat\infty_2}\otimes(\nu-\gamma_2^{12})+|\gamma|\delta_{\hat\infty_2,\hat\infty_1}+|\mu|\delta_{\hat\infty_0,\hat\infty_0} \label{eq:gamma_hat_12}.
\end{align}

\begin{remark}
It is straightforward to verify the above mappings are well-defined. 
In addition, we can observe that, for each $\gamma^{01}\in\Gamma_\leq(\sigma,\mu),\gamma^{02}\in\Gamma_\leq(\sigma,\nu),\gamma^{12}\in\Gamma_\leq(\mu,\nu)$,
\begin{align}
&\hat{\gamma}^{01}(\{\hat\infty_2\}\times X)=\hat{\gamma}^{01}(S\times \{\hat\infty_2\})=0,  \label{pf:gamma01_cond} \\
&\hat{\gamma}^{02}(\{\hat\infty_1\}\times Y)=\hat{\gamma}^{02}(S\times \{\hat\infty_1\})=0,  \label{pf:gamma02_cond} \\
&\hat{\gamma}^{12}(\{\hat\infty_0\}\times Y)=\hat{\gamma}^{12}(X\times \{\hat\infty_0\})=0.  \nonumber 
\end{align}
\end{remark}

Based on these concepts, we define the following fused-GW variant problem: 
\begin{align}
\hat{d}_{FGW,r,q,\lambda}(\hat{\mathbb{X}}, \hat{\mathbb{Y}}) 
:= \inf_{\hat{\gamma} \in \Gamma(\hat{\mu}, \hat{\nu})} 
\int_{(\hat{X} \times \hat{Y})^2} \omega_1 \frac{1}{c}\hat D^q(x,y) 
+ \omega_2 L^q_\lambda( d_{\hat{X}}^r(x, x'), d_{\hat{Y}}^r(y, y')) 
 d\hat{\gamma}(x, y)d\hat{\gamma}(x', y').
\label{eq:fgw_hat_metric}
\end{align}
where constant $c=|\sigma|+|\mu|+|\nu|$. 

\begin{proposition}\label{pro:fpgw-fgw}
For each  $\gamma^{12}\in\Gamma(\mu,\nu)$, construct $\hat{\gamma}^{12}\in\Gamma(\hat\mu,\hat\nu)$, we have: 
\begin{align}
&\int_{(X\times Y)^2} \left( \omega_1 \frac{\|x-y\|^q}{|\gamma^{12}|} + \omega_2 \left| d^r_{X}(x,x') - d^r_{Y}(y,y') \right|^q + \lambda \left( \frac{|\mu|}{|\gamma^{12}|} + \frac{|\nu|}{|\gamma^{12}|} - 2 \right) \right) 
\, d(\gamma^{12})^{\otimes 2} \nonumber\\
&=\int_{(\hat X\times \hat Y)^2} \left( \omega_1 \frac{D^q(x,y)}{c} + \omega_2 L^q_\lambda(d_{\hat{X}}(x,x'), d_{\hat{Y}}(y,y')) \right) 
\, d(\hat\gamma^{12})^{\otimes 2} \label{eq:gamma^12-hat_gamma^12}
\end{align}
Furthermore, we have: 
\begin{align}
d_{FPGW,r,q,\lambda}(\mathbb{X},\mathbb{Y})=\hat{d}_{FGW,r,q,\lambda}(\hat{\mathbb{X}},\hat{\mathbb{Y}}) \label{eq:fpgw-fgw}.
\end{align}

\end{proposition}
\begin{proof}
We have: 
\begin{align}
&\int_{(X\times Y)^2}  \omega_1 \frac{\|x-y\|^q}{|\gamma^{12}|}
, d(\gamma^{12})^{\otimes 2} \nonumber\\
&=\int_{(X\times Y)}\omega_1 \|x-y\|^q d\gamma^{12}\nonumber\\
&=\int_{(\hat X\times \hat Y)}\omega_1 D^q(x,y) d\hat\gamma^{1,2} \nonumber\\
&=\int_{(\hat X\times \hat Y)^2}  \omega_1 \frac{D^p(x,y)}{|\hat\gamma^{12}|} 
\, d(\hat\gamma^{12})^{\otimes 2} \nonumber\\
&=\int_{(\hat X\times \hat Y)^2}  \omega_1 \frac{1}{c}D^p(x,y), d(\hat\gamma^{12})^{\otimes 2}.
\end{align}
In addition, by \cite{bai2024efficient} Proposition D.3. We have 
\begin{align}
&\int_{(X\times Y)^2}\omega_2 |d_X^r(x,x')-d_Y^r(y,y')|^q d(\gamma^{12})^{\otimes2}+\lambda(|\mu|^2+|\nu|^2-2|\gamma^{12}|^2)\nonumber\\
&=\int_{(\hat X\times \hat Y)^2}\omega_2 D_\lambda^q(d_{\hat X}^r(x,x'),d_{\hat Y}^r(y,y'))d(\hat \gamma^{12})^{\otimes2}.
\end{align}
Combining the above two equalities, we prove \eqref{eq:gamma^12-hat_gamma^12}. 

Now, we prove the second equality. Note, if we merge the three auxiliary points, i.e., $\hat\infty_1=\hat\infty_2=\hat\infty_3$. The optimal value for the above problem \eqref{eq:fgw_hat_metric} is unchanged. 

As we merged the points $\hat\infty_1,\hat\infty_2,\hat\infty_3$, by Bai et al. \cite{bai2022sliced}, the mapping $\gamma^{12}\mapsto \hat\gamma^{12}$
defined in \eqref{eq:gamma_hat_12} is a bijection. Thus, by \eqref{eq:gamma^12-hat_gamma^12}, we have $\gamma^{12}\in\Gamma_\leq(\mu,\nu)$ is optimal in \eqref{eq:fpgw_metric} iff  $\hat{\gamma}^{12}$ is optimal in \eqref{eq:fgw_hat_metric} and we complete the proof. 
\end{proof}

\begin{lemma}\label{lem:gamma_hat}
Choose $\gamma^{01}\in\Gamma_\leq(\sigma,\mu),\gamma^{02}\in\Gamma_\leq(\sigma,\nu),\gamma^{12}\in\Gamma_\leq(\mu,\nu)$ and construct $\hat{\gamma}^{01},\hat{\gamma}^{02},\hat{\gamma}^{12}$. Then there exists $\hat\gamma\in\mathcal{M}_+(\hat{S}\times\hat{X}\times\hat{Y})$ such that:
\begin{align}
&(\pi_{0,1})_\#\hat\gamma= \hat \gamma^{01}, \label{eq:gamma_cond1}\\
&(\pi_{0,2})_\#\hat\gamma= \hat \gamma^{02}, \label{eq:gamma_cond2}\\
&\gamma(A_i)=0,\forall i=0,1,2 \text{ where }A_i=\{\hat\infty_i\}\times X\times Y.
\end{align}
\end{lemma}
\begin{proof}
By \textit{gluing lemma} (see Lemma 5.5 \cite{santambrogio2015optimal}), there exists $\hat\gamma\in\mathcal{M}_+(\hat{S}\times\hat{X}\times\hat{Y})$, such that \eqref{eq:gamma_cond1},\eqref{eq:gamma_cond2} are satisfied. 
For the third property, we have: 

\begin{align}
&\hat\gamma(A_0)\leq \hat\gamma(\{\infty_0\}\times \hat X\times \hat{Y})=\hat\sigma(\{\infty_0\})=0\nonumber \qquad\text{by definition \eqref{pf:sigma_hat,mu_hat,nu_hat} of $\hat\sigma$ , }\nonumber\\
&\hat\gamma(A_1)\leq \hat\gamma(\{\infty_1\}\times  \hat X\times Y)=\hat\gamma^{02}(\{\hat\infty_1\times Y\})=0 \qquad\text{by  }\eqref{pf:gamma02_cond}, \nonumber \\
&\hat\gamma(A_2)\leq \hat\gamma(\{\infty_2\}\times   X\times \hat Y)=\hat\gamma^{01}(\{\hat\infty_2\times X\})=0 \qquad\text{by }\eqref{pf:gamma01_cond} \nonumber, 
\end{align}
and we complete the proof. 
\end{proof}

Now, we demonstrate the proof of triangle inequality for fused-PGW distance \eqref{eq:fpgw_metric}. 

\begin{proof}[Proof of Theorem \ref{thm:metric_formal} (3).]

Note, by the Proposition \ref{pro:fpgw-fgw}, the triangle inequality for fused-PGW distance \eqref{eq:fpgw_metric} is equivalent to show 
\begin{align}
\hat{d}_{FGW,r,q,\lambda}(\hat{\mathbb{X}},\hat{\mathbb{Y}})\leq 2^{q-1}(\hat{d}_{FGW,r,q,\lambda}(\hat{\mathbb{S}},\hat{\mathbb{X}})+\hat{d}_{FGW,r,q,\lambda}(\hat{\mathbb{S}},\hat{\mathbb{Y}}))\label{pf:triangle_inq}.
\end{align}

Choose optimal transportation plans  $\gamma^{01},\gamma^{02},\gamma^{12}$ for fused PGW problems 
$d_{FPGW,r,\lambda}(\mathbb{S},\mathbb{X}),d_{FPGW,r,\lambda}(\mathbb{S},\mathbb{Y})$ and $d_{FPGW,r,\lambda}(\mathbb{X},\mathbb{Y})
$ respectively. We construct the corresponding  $\hat{\gamma}^{01},\hat{\gamma}^{02},\hat{\gamma}^{12}$. By the Proposition \ref{pro:fpgw-fgw}, $\hat\gamma^{01},\hat\gamma^{02},\hat\gamma^{12}$ are optimal.

Choose $\hat\gamma$ in lemma \ref{lem:gamma_hat}. We have: 

\begin{align}
&\hat{d}_{FGW,r,q,\lambda}(\mathbb{X},\mathbb{Y})\nonumber\\
&=\int_{(\hat X\times \hat Y)^2}\omega_1 \frac{D^q(x,y)}{c}+\omega_2 L^q_\lambda(d_{\hat{X}}(x,x'),d_{\hat{Y}}(y,y'))d\hat\gamma^{12}(x,y)d\hat\gamma^{12}(x',y')\nonumber\\
&\leq \int_{(\hat{S}\times \hat X\times \hat Y)^2}\omega_1 \frac{D^q(x,y)}{c}+\omega_2 L^q_\lambda(x,y)d\hat\gamma(s,x,y)d\hat\gamma(s',x',y')\nonumber\\
&=\underbrace{\int_{(\hat{S}\times \hat X\times \hat Y)}\omega_1 D^q(x,y)d\hat\gamma(s,x,y)}_A+\underbrace{\int_{(\hat{S}\times \hat X\times \hat Y)^2}\omega_2 L^q_\lambda(x,y)d\hat\gamma(s,x,y)d\hat\gamma(s',x',y')}_B.\nonumber 
\end{align}

To bound $A$, we consider the case $D(x,y)>D(s,x)+D(s,y)$ for some $(s,x,y)\in \hat{S}\times \hat{X}\times \hat{Y}$. By definition of $D$, we have $s\in \{\hat{\infty}_i:i\in[0:2]\},x\in X,y\in Y$. That is, $(s,x,y)\in \bigcup_{i=0}^2A_i$. By lemma \ref{lem:gamma_hat}, we have $\hat\gamma(\bigcup_{i=0}^2A_i)=0$. That is, this case has a measure of 0. 

Thus, we have 
\begin{align}
&A\leq \int_{\hat{S}\times \hat{X}\times \hat{Y}}\omega_1 (D(s,x)+D(s,y))^q d\hat\gamma(s,x,y)\nonumber\\
&\leq \int_{\hat{S}\times \hat{X}\times \hat{Y}}\omega_1 (2^{q-1}D(s,x)+2^{q-2}D(s,y)) d\hat\gamma(s,x,y)\nonumber\\
&=\int_{\hat{S}\times \hat{X}\times \hat{Y}}\omega_1 2^{q-1}D(s,x) d\hat\gamma(s,x,y)+\int_{\hat{S}\times \hat{X}\times \hat{Y}}\omega_1 2^{q-2}D(s,y) d\hat\gamma(s,x,y)\nonumber\\
&= 2^{q-1}\int_{\hat{S}\times \hat{X}}\omega_1\frac{D(s,x)}{c} d\hat\gamma^{01}(s,x)+2^{q-1}\int_{\hat{S}\times \hat{Y}}\omega_1\frac{D(s,y)}{c} d\hat\gamma^{02}(s,y)\label{pf:triangle_term1}.
\end{align}
where the second inequality follows from the fact
\begin{align}
(a+b)^q\leq 2^{q-1}a^q+2^{q-2}b^q,\forall a,b\ge 0 \label{eq:ab_inq}.
\end{align}

Now we bounde the term $B$.  Proposition D.4 in \cite{bai2024efficient}, we have 
\begin{align}
B&\leq \int_{(\hat{S}\times \hat X\times \hat Y)^2}\omega_2 (L_\lambda(s,x)+L_\lambda(s,y))^qd\hat\gamma(s,x,y)d\hat\gamma(s',x',y')\nonumber\\
&\leq \int_{(\hat{S}\times \hat X\times \hat Y)^2}\omega_2 \left(2^{q-1}L_\lambda^q(s,x)+2^{q-1}L_\lambda^q(s,y)\right)d\hat\gamma(s,x,y)d\hat\gamma(s',x',y')\nonumber\\
&=2^{q-1}\int_{(\hat{S}\times \hat X)^2}\omega_2 L_\lambda^q(s,x)d\hat\gamma^{01}(s,x)d\hat\gamma^{01}(s',x')+2^{q-1}\int_{(\hat{S}\times \hat Y)^2}\omega_2 L_\lambda^q(s,y)d\hat\gamma^{02}(s,y)d\hat\gamma^{02}(s',y')\label{pf:triangle_term2}.
\end{align}
where the second inequality holds from \eqref{eq:ab_inq}. 
Combining \eqref{pf:triangle_term1} and \eqref{pf:triangle_term2}, we prove the inequality \eqref{pf:triangle_inq} and we complete the proof. 
\end{proof}

\subsection{Future's direction.}
Note, the general version for (the p-th power of) fused-Gromov Wasserstein distance is defined by 
\begin{align}
d_{FGW,r,q,p}^p(\mathbb{X},\mathbb{Y}):=\inf_{\gamma\in\Gamma(\mu,\nu)}\int \left(\omega_1 \|x-y\|^q+|d_X(x,x')-d_Y(y,y')|^q\right)^pd\gamma^{\otimes2}\label{eq:fgw_metric_general}. 
\end{align}
Inspired by the above formulation, we propose the following generalized fused-partial Gromov Wasserstein distance: 

\begin{align}
d_{FGW,r,q,p}^p(\mathbb{X},\mathbb{Y}):=\inf_{\gamma\in\Gamma(\mu,\nu)}\int \left(\omega_1 \frac{\|x-y\|^q}{|\gamma|}+|d_X(x,x')-d_Y(y,y')|^q+\lambda(\frac{|\mu|^2}{|\gamma|^2}+\frac{|\nu|^2}{|\gamma|^2}-2)\right)^pd\gamma^{\otimes2}\label{eq:fgw_metric_general_0}. 
\end{align}

The fused-PGW distance defined in \eqref{eq:fpgw_metric} can be treated as a special case of this general formulation by setting $p=1$. In our conjecture, a similar (semi-)metric property proposed in Theorem \ref{thm:metric_formal} holds for the above general form. We leave the theoretical study of the metric property for our future work.

\section{Frank Wolf Algorithm for the Fused Mass-constraint Partial Gromov Wasserstein problem}

In discrete setting, the  FMPGW problem \eqref{eq:fmpgw} becomes the following: 
\begin{align}
FMPGW_\rho(\mathbb{X},\mathbb{Y})=\min_{\gamma\in\gamma_\leq^\rho(p,q)}\underbrace{\omega_1\langle C,\gamma \rangle+\omega_2\langle M\circ \gamma,\gamma}_{\mathcal{L}_{C,M}} \rangle \label{eq:fmpgw_empirical} 
\end{align} 
where $\Gamma_\leq^\rho(\mathrm{p},\mathrm{q}):=\{\gamma\in\mathbb{R}_+^{n\times m}:\gamma 1_m\leq \mathrm{p}, \gamma^\top 1_n\leq \mathrm{q}, |\gamma|\ge \rho \}$, 
$C=[C(x_i,y_j)]_{i\in[1:n],j\in[1:m]}$, $M_{i,j,i',j'}:=L(C^X,C^Y):=[L(C^X_{i,i'},C^Y_{j,j'})]_{i,i\in[1:n],j,j'\in[1:n]}$, $M\circ \gamma=[\langle M[i,j,\cdot,\cdot],\gamma\rangle]_{i,j}$.

Similar to the Fused Gromov-Wasserstein problem, we propose the following Frank-Wolfe algorithm as a solver: The above problem will be solved iteratively. In every iteration, say $k$, we will adapt the following steps: 

\textbf{Step 1. Gradient computation.}

Suppose $\gamma^{(k-1)}$ is the transportation plan in the previous iteration; it is straightforward to verify: 
$$\nabla\mathcal{L}_{C,M}(\gamma)=\omega_1 C+\omega_2((M+M^\top)\circ\gamma),$$
where $M^\top=[M_{i',j',i,j}]_{i,i'\in[1:n],j,j'\in[1:m]}\in\mathbb{R}^{n\times m\times n\times m}$.
Next, we aim to find $\gamma\in\Gamma_\leq^\rho(\mu,\nu)$ for the following problem: 
\begin{align}
\gamma^{(k)}{'}:=\arg\min_{\gamma\in\Gamma_\leq^\rho(\mu,\nu)} \langle \nabla\mathcal{L}_{C,M}(\gamma^{(k-1)}),\gamma\rangle\label{eq:fmpgw_step1}. 
\end{align}
which is a mass-constraint partial OT problem. 

\textbf{Step 2. linear search algorithm.}
In this step, we aim to find the optimal step size $\alpha^{(k)}\in[0,1]$. In particular,  
\begin{align}
\alpha^{(k)}:=\arg\min_{\alpha\in[0,1]}\mathcal{L}_{C,M}((1-\alpha)\gamma^{(k-1)}+\alpha\gamma^{(k)}{'})\nonumber. 
\end{align}

Let  $\delta\gamma=\gamma^{(k)}{'}-\gamma^{(k-1)}$, the above problem is essentially quadratic problem: 
\begin{align}
&\mathcal{L}((1-\alpha)\gamma^{(k-1)}+\alpha\gamma^{(k)}{'})
=\alpha^2\underbrace{\langle \omega_2M\circ \delta\gamma,\delta\gamma\rangle}_{a}\nonumber\\
&+\alpha \underbrace{\langle \omega_2(M+M^\top)\circ \gamma^{(k-1)}+\omega_1C,\delta\gamma\rangle }_{b}\nonumber\\
&+\underbrace{\langle \omega_2M\circ\gamma^{(k-1)}+\omega_1C,\gamma\rangle}_c \label{eq:abc_2}
\end{align}
and $\alpha^*$ is given by 
\begin{align}
\alpha^*=\begin{cases}
    1 &\text{if } a\leq 0, a+b\leq 0, \\
    0 &\text{if }a\leq 0, a+b>0\\
    \text{clip}(\frac{-b}{2a},[0,1]), &\text{if }a>0 
\end{cases}\label{eq:line_search_sol_2}. 
\end{align}
Then $\gamma^{(k)}=(1-\alpha^*)\gamma^{(k-1)}+\alpha^*\gamma^{(k)}{'}$. 

\begin{algorithm}[bt]
\caption{Frank-Wolfe Algorithm for FMPGW}
   \label{alg:fmpgw}
\begin{algorithmic}
   \STATE {\bfseries Input:} $C\in \mathbb{R}^{n\times m}, C^X\in \mathbb{R}^{n\times n},C^Y\in \mathbb{R}^{m\times m}, \mathrm{p}\in \mathbb{R}^n_+, \mathrm{q}\in\mathbb{R}^m_+$, $\omega_1\in[0,1],\rho\in [0, \min(|\mathrm{p}|,|\mathrm{q}|)]$.
   \STATE {\bfseries Output:}
$\gamma^{(final)}$

\FOR{$k=1,2,\ldots$}
   \STATE 
$G^{(k)}\gets \omega_1C+\omega_2 (M+M^\top)\circ \gamma^{(k)}$ // Compute gradient 
   \STATE $\gamma^{(k)}{'}\gets \arg\min_{\gamma\in \Gamma^\rho_\leq(\mathrm{p},\mathrm{q})}\langle G^{(k)}, \gamma\rangle_F$ // Solve the POT problem.  \\ 
   \STATE Compute $\alpha^{(k)}\in[0,1]$ via  \eqref{eq:line_search_sol_1} //  Line Search
   \STATE $\gamma^{(k+1)}\gets (1-\alpha^{(k)})\gamma^{(k)}+\alpha^{(k)} \gamma^{(k)'}$//  Update $\gamma$  
   \STATE if convergence, break
   \ENDFOR
\STATE $\gamma^{(final)}\gets \gamma^{(k)}$
\end{algorithmic}
\end{algorithm}
In the next section, we provide a detailed introduction to how the FW algorithms are derived. 

\section{Gradient computation of FW algorithms}
In this section, we discuss in detail the gradient computation in fused-MPGW and fused-PGW.

\subsection{Basics in Tensor product}
In this section, we introduce some fundamental results for tensor computation. 
Suppose $M\in \mathbb{R}^{n\times m\times n\times m}$. We define the \textbf{Transportation of tensor}, denoted as $M^\top$, as:
\begin{align}
M^\top_{i,j,i',j'}=M_{i',j',i,j}, \label{eq:M_trans}
\end{align}
and we say $M$ is symmetric if $M=M^\top$. 

It is straightforward to verify the following: 
\begin{proposition}
Given tensor $M\in\mathbb{R}^{n\times m\times n\times m}$, then we have: 
\begin{itemize}
    \item $(M^\top )^\top=M.$
    \item $M^\top\circ\gamma=[\sum_{i',j'}M_{i',j',i,j}\gamma_{i,j}]_{i,j\in[1:n]\times[1:m]}$. 
\end{itemize}
\end{proposition}
\begin{proof}
The first item follows from the definition of $M^\top$. For the second statement, pick $(i,j)\in[1:n]\times [1:m]$, we have
\begin{align}
&\sum_{i',j'}M_{i',j',i,j}\gamma_{i',j'}\nonumber\\
&=\sum_{i',j'}M^\top_{i,j,i',j'}\gamma_{i',j'}\nonumber\\
&=\langle M^\top\circ \gamma\rangle \nonumber. 
\end{align}
\end{proof}

Therefore, we have: 
\begin{proposition}
The gradient for $\mathcal{L}_{C,M}(\gamma)$ in Fused-PGW \eqref{eq:fmpgw} is given by:      
\begin{align}
\nabla \mathcal{L}_{C,M}(\gamma)=\omega_1 C+\omega_2( M\circ \gamma +M^\top \circ\gamma )\label{eq:gradient}.
\end{align}.

Similarly, the gradient for $\mathcal{L}_{C,M-2\lambda}$ in fused-PGW \eqref{eq:fpgw} is given by: 

\begin{align}
\nabla \mathcal{L}_{C,M-2\lambda}(\gamma)=\omega_1 C+\omega_2( (M-2\lambda)\circ \gamma +(M-2\lambda)^\top \circ\gamma )\label{eq:gradient_1}.
\end{align}.

\end{proposition}
\begin{proof}
Pick $(i,j)\in[1:n]\times[1:m]$, we have: 
\begin{align}
&\frac{\partial }{\partial \gamma_{ij}}L_{C,M}(\gamma)\nonumber\\
&=\frac{\partial }{\partial \gamma_{ij}}\sum_{i,j,i',j'}\omega_1C_{i,j}\gamma_{i,j}+\omega_2\sum_{i,j,i',j'}M_{i,j,i',j'}\gamma_{i,j}\gamma_{i',j'}\nonumber\\
&=\omega_1C_{i,j}+\omega_2(\sum_{i',j'}M_{i,j,i',j'}\gamma_{i,j}\gamma_{i',j'}+\sum_{i',j'}M_{i',j',i,j}\gamma_{i,j}\gamma_{i',j'})\nonumber
\end{align}
Therefore, $\nabla L_{C,M}(\gamma)=\omega_1 C+\omega_2(M\circ\gamma+M^\top\circ\gamma)$ and we complete the proof. The gradient for Fused-PGW \eqref{eq:fpgw} can be derived similarly. 
\end{proof}

At the end of this subsection, we discuss the computation of $M\circ\gamma$ and $M^\top\gamma$. 

In general, the computation cost for $M\circ\gamma$ is $n^2m^2$. However, if the cost function $L$ satisfies: 
\begin{align}
  L(\mathrm{r}_1,\mathrm{r}_2)=f_1(\mathrm{r}_1)+f_2(\mathrm{r}_2)-h_1(\mathrm{r}_1)h_2(\mathrm{r}_2),\label{eq:cond_L}  
\end{align}

\begin{align}
M\circ\gamma=f_1(C^X)\gamma_11_m^\top+1_n\gamma_2^\top f_2(C^Y)-h_1(C^X)\gamma h_2(C^Y),\label{eq:M_circ_gamma}
\end{align}

and the corresponding complexity is $\mathcal{O}(n^2+m^2)$ (see e.g. \cite{peyre2016gromov,chapel2020partial,bai2024efficient} for details.)

Therefore, we have the following: 
\begin{lemma}
Suppose $M=[L(d_X(x_i,x_{i'}),d_Y(y_j,y_{j'}))]$, then we have: 
$$(M^\top)_{i,j,i',j'}=f_1((C^X)^\top_{i,i'})+f_2((C^Y)^\top_{j,j'})-h_1((C^X)_{i,i'})h_2((C^Y)_{j,j'}).$$
It directly implies: 
$$M^\top\circ \gamma=
f_1((C^X)^\top)\gamma_11_m^\top+1_n\gamma_2^\top f_2((C^Y)^\top)-h_1((C^X)^\top)\gamma h_2((C^Y)^\top).$$
\end{lemma}
\begin{proof}
We have: 
\begin{align}
M_{i,j,i',j'}^\top&=M_{i',j',i,j}\nonumber\\
&=f_1(C^X_{i',i})+f_2(C^Y_{j',j})-h_1(C^X_{i',i})h_2(C^Y_{j',j})
)\nonumber\\
&=f_1((C^X)^\top_{i',i})+f_2((C^Y)^\top_{j',j})-h_1((C^X)^\top_{i',i})h_2((C^Y)^\top_{j',j}).
)\nonumber 
\end{align}
And we complete the proof. 
\end{proof}

\subsection{Gradient in fused-MPGW.}
As we discussed in previous section, after iteration $k-1$, the gradient is given by $\nabla \mathcal{L}(\gamma)=\omega_1C+\omega_2 (M\circ \gamma+M^\top\circ\gamma)$, where 
$M_{i,j,i',j'}=(L(C^X_{i,i'},C^Y_{j,j'}))\in \mathbb{R}_+^{n\times m\times n\times m}$. Furthermore, the computational complexity for $M\circ \gamma:=[\langle M[i,j,:,:],\gamma \rangle ]$ can be improved to be $\mathcal{O}(n^2+m^2)$.

Based on the gradient, we aim to solve the following to update the transportation plan in iteration $k$: 
\begin{align}
\gamma^{(k)}{'}=\arg\min_{\gamma\in\Gamma_\leq^\rho(\mu,\nu)}\langle \underbrace{\omega_1C+\omega_2 (M\circ \gamma^{(k-1)}+M^\top \gamma^{(k-1)})}_{\mathcal{C}},\gamma\rangle.  
\end{align}
The above problem is essentially the partial OT problem, where the cost function is defined by $\mathcal{C}$. 

\subsection{Gradient of Fused-PGW.}
Similarly, after iteration $k-1$, the gradient of cost $\mathcal{L}_{C,M-2\lambda}$ in FPGW with respect to $\gamma$ is given by 
$$\nabla \mathcal{L}(\gamma)=\omega_1 C+\omega_2((M-2\lambda )\circ \gamma^{(k-1)}+(M^\top-2\lambda )\circ \gamma^{(k-1)}).$$

We aim to solve the following problem: 
\begin{align}
&\min_{\gamma\in\Gamma_\leq(\mu,\nu)}\langle \omega_1 C+\omega_2((M-2\lambda )\circ \gamma^{(k-1)}+(M^\top-2\lambda)\circ\gamma^{(k-1)}),\gamma\rangle \nonumber\\
&=\min_{\gamma\in\Gamma_\leq(\mu,\nu)}\langle \underbrace{\omega_1 C+\omega_2 (M\circ\gamma^{(k-1)}+M^\top\circ\gamma^{(k-1)})}_{\mathcal{C}},\gamma \rangle +\lambda |\gamma^{(k-1)}|(|\mu|+|\nu|-2|\gamma|)-\underbrace{\lambda |\gamma^{(k-1)}|(|\mu|+|\nu|)}_{\text{constant}} \nonumber. 
\end{align}

Note, if we ignore the constant term, the above problem is the partial OT problem and can be solved by \cite{bonneel2011displacement} or \cite{bai2022sliced}.

\section{Line search algorithm}
The line search for Fused-MPGW is defined as 
\begin{align}
\min_{\alpha\in[0,1]}\mathcal{L}((1-\alpha)\gamma^{(k-1)}+\alpha\gamma^{(k)}{'}) \label{eq:line_search_prob_fpgw},
\end{align}
where $\mathcal{L}(\gamma)=\omega_1\langle C,\gamma \rangle+\omega_2\langle M,\gamma^{\otimes2} \rangle$. 
Let $\delta\gamma=\gamma^{(k)}{'}-\gamma^{(k-1)}$. The above problem is essentially a quadratic problem with respect to $\alpha$: 
\begin{align}
&\mathcal{L}((1-\alpha)\gamma^{(k-1)}+\alpha\gamma^{(k)}{'})\nonumber\\
&=\mathcal{L}(\gamma^{(k-1)}+\alpha\delta\gamma)\nonumber\\
&=\omega_1 \langle C,(\gamma^{(k-1)}+\alpha\delta\gamma) \rangle+\omega_2\langle M\circ (\gamma^{(k-1)}+\alpha\delta\gamma),(\gamma^{(k-1)}+\alpha\delta\gamma) \rangle \nonumber\\
&=\alpha^2\underbrace{\omega_2\langle M\circ \delta\gamma,\delta\gamma\rangle}_{a}+\alpha \underbrace{\langle \omega_2M\circ \gamma^{(k-1)}+\omega_1C,\delta\gamma\rangle+\langle \omega_2M\circ \delta\gamma,\gamma^{(k-1)}\rangle}_{b}+\underbrace{\langle \omega_2 M\circ\gamma^{(k-1)}+\omega_1C,\gamma^{(k-1)}\rangle}_c, \label{eq:line_search_2}
\end{align}
and $\alpha^*$ is given by 
\begin{align}
\alpha^*=\begin{cases}
    1 &\text{if } a\leq 0, a+b\leq 0, \\
    0 &\text{if }a\leq 0, a+b>0\\
    \text{clip}(\frac{-b}{2a},[0,1]), &\text{if }a>0 
\end{cases}\label{eq:line_search_sol_1_2}. 
\end{align}
Next, we simplify the term $b$ in the above formula. We first introduce the following lemma: 
\begin{lemma}
Choose $\gamma^1,\gamma^2\in\mathbb{R}^{n\times m},M\in\mathbb{R}^{n\times m\times n\times m}$, we have: 
\begin{align}
\langle M\circ\gamma^1,\gamma^2 \rangle =\langle M^\top\circ\gamma^2,\gamma^1\rangle.
\end{align}
\end{lemma}
\begin{proof}
We have 
\begin{align}
\langle M\circ\gamma^1,\gamma^2 \rangle&=\sum_{i,j}\sum_{i',j'}M_{i,j,i',j'}\gamma^1_{i',j'}\gamma^2_{i,j}\nonumber\\
&=\sum_{i',j'}\sum_{i,j}M_{i',j',i,j}^\top\gamma^2_{i,j}\gamma^{1}_{i,j}\nonumber\\
&=\langle M^\top\circ\gamma^2,\gamma^1 \rangle. 
\end{align} 
\end{proof}
Therefore, the term $b$ can be further simplified as 
\begin{align}
b&=\omega_1 \langle C,\delta\gamma \rangle+\omega_2 (\langle M\circ\gamma^{(k-1)},\delta\gamma\rangle+\langle M^\top\circ\gamma^{(k-1)},\delta\gamma\rangle) \nonumber\\
&=\langle\underbrace{\omega_1 C+M\circ\gamma^{(k-1)}+M^\top \circ\gamma^{(k-1)}}_{\nabla L_{C,M}(\gamma)},\delta\gamma \rangle. 
\end{align}
That is, we can directly adapt the gradient obtained before to compute $b$ and improve the computation efficiency.

\begin{remark}
When we select quadratic cost, i.e.,  $L(\mathrm{r}_1,\mathrm{r}_2):=|\mathrm{r}_1-\mathrm{r}_2|^2$, we can set $f_1(\mathrm{r}_1)=\mathrm{r}_1^2,f_2(\mathrm{r}_2)=\mathrm{r}_2^2, h_1(\mathrm{r}_1)=2\mathrm{r}_1,h_2(\mathrm{r}_2)=\mathrm{r}_2$, then 
$L(\mathrm{r}_1,\mathrm{r}_2)$ satisfies \eqref{eq:cond_L}. 
Furthermore, suppose the problem is in the balanced fused-GW setting, i.e., $\rho=|\mu|=|\nu|=1$.  We have: 
\begin{align}
a&=\omega_2\langle M\circ \delta\gamma,\delta\gamma \rangle   \nonumber\\
&=\omega_2 \langle f_1(C^X)\underbrace{\delta\gamma_1}_{0_n} 1_m^\top+\underbrace{1_n\delta\gamma_2^\top}_{0_m} f_2(C^Y)-h_1(C^X)\gamma h_2(C^Y),\delta\gamma\rangle\\
&=-2\omega_2\langle  C^X\delta \gamma C^Y,\delta\gamma \rangle. \nonumber 
\end{align}
Similarly, 
\begin{align}
b&=\langle \omega_2 (M+M^\top)\circ\gamma^{(k-1)}+\omega_1 C,\delta \gamma \rangle\nonumber\\
&=\omega_1 \langle C,\delta\gamma \rangle+\omega_2\langle (M+M^\top)\circ\gamma^{(k-1)},\delta\gamma\rangle \nonumber\\
&=\omega_1 \langle C,\delta\gamma \rangle+\omega_2 \langle (M+M^\top)\circ \delta\gamma,\gamma^{(k-1)} \rangle\nonumber\\
&=\omega_1 \langle C,\delta\gamma \rangle -2\omega_2 \langle C^X\delta \gamma C^Y,\gamma^{(k-1)}\rangle-2\omega_2 \langle (C^X)^\top\delta \gamma (C^Y)^\top,\gamma^{(k-1)}\rangle+\omega_2\langle c_{C^X,C^Y},\gamma^{(k-1)}\rangle\nonumber\\
&+\omega_2\langle c_{(C^X)^\top,(C^Y)^\top},\gamma^{(k-1)}\rangle.   
\end{align}
where $c_{C^X,C^Y}=f_1(C^X)\mathrm{p}1_{m}^\top+1_{m}\mathrm{q}^\top f_2(C^Y)$. Thus, the above formulations recover the line search algorithm (see algorithm 2) in \cite{peyre2016gromov}. 
\end{remark}

Next, we discuss the line search step in fused-PGW \eqref{eq:fgw}. 

Replacing $M$ by $M-2\lambda$ in \eqref{eq:line_search_2}, the solution is obtained by \eqref{eq:line_search_sol_1}.

\section{Convergence Analysis}\label{sec: convergence}

As in \cite{chapel2020partial}, we will use the results from \cite{lacoste2016convergence} on the convergence of the Frank-Wolfe algorithm for non-convex objective functions.

\subsection{Fused-MPGW}

Consider the minimization problems
\begin{equation}\label{eq: min prob appendix}
\min_{\gamma\in\Gamma^\rho_{\leq}(\mathrm{p},\mathrm{q})}\mathcal{L}_{C,M}(\gamma):=\omega_1\langle C,\gamma \rangle+ \omega_2\langle M\circ\gamma,\gamma \rangle  \nonumber 
\end{equation}
that corresponds to the discrete fused-PGW problem \eqref{eq:fpgw}. 

The objective functions 
$$\gamma\mapsto \mathcal{L}_{\hat M}(\gamma)=\omega_1 \langle C,\gamma \rangle+\omega_2 \langle M\circ\gamma,\gamma \rangle, $$ is non-convex in general. 
However, the constraint set $\Gamma_\leq^\rho(p,q)$ are convex and compact on $\mathbb{R}^{n\times m}$ (see Proposition  \ref{pro:compact}.) 

Consider the \textbf{Frank-Wolfe gap} of $\mathcal{L}_{C,M}$ at the approximation $\gamma^{(k)}$ of the optimal plan $\gamma$:
\begin{equation}\label{eq:g_k}
    g_k=\max_{\gamma\in\Gamma_{\leq}({\mathrm{p}}, {\mathrm{q}})} \langle \nabla\mathcal{L}_{\tilde M}(\gamma^{(k)}),\gamma^{(k)}-\gamma\rangle_F.
\end{equation}
It provided a good criterion to measure the distance to a stationary point at iteration $k$. Indeed, a plan $ \gamma^{(k)}$ is a stationary transportation plan for the corresponding constrained optimization problem in \eqref{eq: min prob appendix} if and only if $g_k=0$. Moreover, $g_k$ is always non-negative ($g_k\geq 0$). 

From Theorem 1 in \cite{lacoste2016convergence}, after $K$ iterations, we have the following upper bound for the minimal Frank-Wolf gap:
\begin{equation}
g_K:=\min_{1\leq k\leq K}g_k\leq \frac{\max\{2L_1,\mathrm{Lip} \cdot  (\text{diam}(\Gamma_{\leq}^\rho({\mathrm{p}}, {\mathrm{q}})))^2\}}{\sqrt{K}}, \label{eq:g_k_bound}
\end{equation}
where $$L_1:=\mathcal{L}_{\tilde M}(\gamma^{(1)})-\min_{\gamma\in\Gamma_{\leq}({\mathrm{p}}, {\mathrm{q}})}\mathcal{L}_{\tilde M}( \gamma)$$ is the initial global sub-optimal bound for the initialization $ \gamma^{(1)}$ of the algorithm; $\text{Lip}$ is the Lipschitz constant of function 
$\gamma\mapsto\nabla\mathcal{L}_{\tilde M}$; and 
$$\text{diam}(\Gamma_{\leq}({\mathrm{p}}, {\mathrm{q}}))=\sup_{\gamma,\gamma'\in\Gamma_\leq(\mu,\nu)}\|\gamma-\gamma'\|_F$$ is the $\|\cdot\|_F$ diameter of $\Gamma_{\leq}({\mathrm{p}}, {\mathrm{q}})$ in $\mathbb{R}^{n\times m}$. 

The important thing to notice is that the constant $\max\{2L_1, D_L\}$ does not depend on the iteration step $k$. Thus, according to Theorem 1 in \cite{lacoste2016convergence}, the rate in $\tilde g_K$ is $\mathcal{O}(1/\sqrt{K})$. That is, the algorithm takes at most $\mathcal{O}(1/\varepsilon^2)$ iterations to find an approximate stationary point with a gap smaller than $\varepsilon$.

Next, we will continue to simplify the upper bound \eqref{eq:g_k_bound}. We first introduce the following fundamental results:

\begin{lemma}\label{lem:diam_Gamma_bound}
In the discrete setting, we have 
\begin{align}
{\rm diam}(\Gamma_\leq(p,q))\leq 2\rho. \label{eq:diam_Gamma_bound}
\end{align}
\end{lemma}
\begin{proof}
Choose $\gamma,\gamma'\in \Gamma_\leq^\rho(p,q)$, 
we apply the property
\begin{align}
(a-b)^2\leq 2a^2+2b^2,\quad \forall a,b\in \mathbb{R}. \label{eq:ineq_ab}
\end{align} and obtain 
\begin{align}
\|\gamma-\gamma'\|_{F}^2&=\sum_{i,j}^{n,m}|\gamma_{i,j}-\gamma'_{i,j}|^2\nonumber\\
&\leq\sum_{i,j}^{n,m}2|\gamma_{i,j}|^2+2|\gamma'_{i,j}|^2\nonumber\\
&\leq 2\left[\left(\sum_{i,j}^{n,m}\gamma_{i,j}\right)^2+\left(\sum_{i,j}^{n,m}\gamma'_{i,j}\right)^2\right]\nonumber\\
&= 2(|\gamma|^2+|\gamma'|^2)\\
&=2 (\rho^2+\rho^2)=4\rho^2,\nonumber 
\end{align}
and thus, we complete the proof. 
\end{proof}

\begin{lemma}\label{lem:lip_bound}
The Lipschitz constant term in \eqref{eq:g_k_bound} can be bounded as follows:
\begin{align}
\text{Lip}\leq \omega_2nm \max(|M|)^2.\label{eq:lip_bound}
\end{align}
In particular, when $L(\mathrm{r}_1,\mathrm{r}_2)=|\mathrm{r}_1-\mathrm{r}_2|^p$ where $p\ge 1$, we have 
$\text{Lip}\leq nm (2^{p-1}C^X+2^{p-1}C^Y)^2$. 
\end{lemma}
\begin{proof}
Pick $\gamma,\gamma'\in\Gamma_{\leq}({\mathrm{p}}, {\mathrm{q}})$ we have,
\begin{align*}
&\|\nabla\mathcal{L}_{ M}(\gamma)-\nabla\mathcal{L}_{C,M}(\gamma')\|_F^2\nonumber\\
&=\|(\omega_1 C+\omega_2 (M\circ \gamma+M^\top\circ\gamma)) - (\omega_1 C+\omega_2  (M\circ\gamma'+M^\top\circ\gamma'))\|_F^2\\
&=2\omega_2^2\|M\circ(\gamma-\gamma')\|_F^2+2\omega_2^2\|M^\top\circ(\gamma-\gamma')\|_F^2\nonumber\\
&=2\omega_2^2\sum_{i,j}\left(\left[M\circ(\gamma-\gamma')\right]_{i,j}\right)^2+2\omega_2^2\sum_{i,j}\left(\left[M^\top\circ(\gamma-\gamma')\right]_{i,j}\right)^2\nonumber\\
&=2\omega_2^2\sum_{i,j}\left(\sum_{i',j'}M_{i,j,i',j'}(\gamma_{i',j'}-\gamma'_{i',j'})\right)^2+2\omega_2^2\sum_{i,j}\left(\sum_{i',j'}M^\top_{i,j,i',j'}(\gamma_{i',j'}-\gamma'_{i',j'})\right)^2\\
&\leq 2\omega_2^2\sum_{i,j}\left(\sum_{i',j'}|M_{i,j,i',j'}||\gamma_{i',j'}-\gamma'_{i',j'}|\right)^2+
2\omega_2^2\sum_{i,j}\left(\sum_{i',j'}|M^\top_{i,j,i',j'}||\gamma_{i',j'}-\gamma'_{i',j'}|\right)^2\nonumber \\
&\leq4\omega_2^2\underbrace{\max(M)^2}_{A}\cdot
\underbrace{\sum_{i,j}^{n,m}\left(\sum_{i',j'}^{n,m}|\gamma_{i',j'}-\gamma'_{i',j'}|\right)^2}_B
\end{align*}

For the second term, we have: 
\begin{align}
&\sum_{i,j}^{n,m}\left(\sum_{i',j'}^{n,m}|\gamma_{i',j'}-\gamma'_{i',j'}|\right)^2 \nonumber \\
&\leq \sum_{i,j}^{n,m}\left(nm\sum_{i',j'}^{n,m}\left|\gamma_{i',j'}-\gamma'_{i',j'}\right|^2\right)\nonumber\nonumber\\
&\leq n^2m^2 \|\gamma-\gamma'\|^2_F \label{pf:convergence_bound_B}
\end{align}

For the first term, when $L(\mathrm{r}_1,\mathrm{r}_2)=|\mathrm{r}_1-\mathrm{r}_2|^p$, from the inequality \eqref{eq:ab_inq}, we have: 
\begin{align*}
A&=\max M\leq \max\{2^{p-1}(C^X)^p+2^{p-1}(C^Y)^p\}^2.
\end{align*}
where 
$$\max\{2^{p-1}(C^X)^p+2^{p-1}(C^Y)^{p}\}:=\max\{\max_{i,i',j',j'}2^{p-1}(C^X_{i,i'})^{p}+2^{p-1}(C^Y_{j,j'})^{p}\}.$$

Thus we obtain
\begin{align}
\text{Lip}\leq \frac{\max(2^{p-1}(C^X)^p+2^{p-1}(C^Y)^p)nm\|\gamma-\gamma'\|_F}{\|\gamma-\gamma'\|_F}=\omega_2nm\max(2^{p-1}(C^X)^p+2^{p-1}(C^Y)^p). \nonumber 
\end{align}
and we complete the proof. 
\end{proof}

Combined the above two lemmas, we derive the convergence rate of the Frank-Wolfe gap \eqref{eq:g_k}: 
\begin{proposition}\label{pro:convergence_fmpgw}
When $L(\mathrm{r}_1,\mathrm{r}_2)=|\mathrm{r}_1-\mathrm{r}_2|^2$ in the PGW problem, the Frank-Wolfe gap of algorithm \ref{alg:fmpgw}, defined in \eqref{eq:g_k} at iteration $k$ satisfies the following: 
\begin{align}
g_k \leq \frac{\max\Bigg\{2L_1,4\omega_2\rho^2\cdot nm(\max\{2(C^X)^2+2(C^Y)^2\})\Bigg\}}{\sqrt{k}}.\label{eq:g_k_bound_pgw}
\end{align}
\end{proposition}
\begin{proof}
The proof directly follows from the upper bounds \eqref{eq:diam_Gamma_bound},\eqref{eq:lip_bound} and the inequality \eqref{eq:g_k_bound}.
\end{proof}
\begin{remark}
Note, if the cost function in PGW is defined by $|\mathrm{r}_1-\mathrm{r}_2|^p$ for some $p\neq 2$, it is straightforward to verify that the upper bound of $g_k$ is obtained by replacing the term $\max ((C^X)^2+(C^Y)^2)$ should be replaced by 
$$\max_{i,j,i',j'}\left[2^{p-1}((C^X)^p+(C^Y)^p)\right].$$
\end{remark}
\begin{remark}
From the proposition \eqref{pro:convergence_fmpgw}, to achieve an $\epsilon-$accurate solution, the required number of iterations is 
\begin{align}
\frac{\max\Bigg\{2L_1,4\rho\omega_2\cdot n^2m^2\max(\{2(C^X)^2+2(C^Y)^2\})\Bigg\}^2}{\epsilon^2}.\nonumber
\end{align}
\end{remark}
\begin{remark}
Note, when $\omega_2=1$, this convergence rate is constant to the convergence rate of the FW algorithm for MPGW in \cite{chapel2020partial}. In addition, the convergence rate is independent of $C$. The main reason is that the part $\omega_1\langle C,\gamma \rangle$ is linear with respect to $\gamma$. Thus, this part does not contribute to the Lipschitz of the mapping $\gamma\mapsto \nabla_M(\mathcal{L})$. 
\end{remark}

\subsection{Convergence of Fused-PGW}
Similar to the above, in this subsection, we discuss the convergence rate for algorithm 2. 
Suppose $\gamma$ is a solution for the fused-PGW problem \eqref{eq:fpgw}. 
In iteration $k$, we define the gap between $\gamma$ and the approximation $\gamma^{(k)}$ and the Frank-Wolfe gap: 
\begin{align}
&g_k:=\max_{\gamma\in\Gamma_\leq(p,q)}\langle \nabla\mathcal{L}_{C,M-2\lambda}(\gamma^{(k)}),\gamma^{(k)}-\gamma \rangle_F \label{eq:g_k_fpgw}. \\
&g_K:=\min_{1\leq k\leq K}g_k \label{eq:g_K_fpgw}. 
\end{align}

Thus, the convergence rate for the FW algorithm for the fused-PGW problem \eqref{eq:fpgw} can be bounded by the following proposition: 
\begin{proposition}\label{pro:convergence_fpgw}
When $L(\mathrm{r}_1,\mathrm{r}_2)=|\mathrm{r}_1-\mathrm{r}_2|^2$ in the Fused-PGW problem, the Frank-Wolfe gap of algorithm \ref{alg:fpgw},
\eqref{eq:g_k} at iteration $k$ satisfies the following: 
\begin{align}
g_K \leq \frac{\max\Bigg\{2L_1,4\omega_2\min^2(|p|,|q|)\cdot nm(\max\{2(C^X)^2+2(C^Y)^2,2\lambda\})\Bigg\}}{\sqrt{k}},\label{eq:g_k_bound_pgw_1}
\end{align}
\end{proposition}
where $\max(2(C^X)^2+2(C^Y)^2,2\lambda)=\max\{\max_{i,i',j,j'} (2(C^X_{i,i'})^2+2(C^Y_{j,j'})^2,2\lambda\}$. 

From Theorem 1 in \cite{lacoste2016convergence}, we have: 
\begin{align}
g_K\leq \frac{\max\{2L_1, \text{Lip}\cdot(\text{diam}(\Gamma_\leq(p,q)))\}}{\sqrt{K}},\label{eq:converge_bound_fpgw}
\end{align}
where $\text{Lip}$ is the Lipschitz constant of function $\gamma\to \nabla_M \mathcal{L}$ (with respect to Fubini norm). 

By \cite{bai2024efficient},
\begin{align}
 \text{diam}(\Gamma_\leq(p,q))\leq 2\min(|p|,|q|).\label{eq:diam_2}   
\end{align}
Next, we bound the Lipschitz constant.
\begin{lemma}\label{lem:lip_bound_fpgw}
The Lipschitz constant in \eqref{eq:converge_bound_fpgw} can be bounded as follows: 
\begin{align}
\text{\rm Lip}\leq \omega_2nm\max |M-2\lambda|^2\label{eq:lip_fpgw}.
\end{align}
When $L(\mathrm{r}_1,\mathrm{r}_2)=|\mathrm{r}_1-\mathrm{r}_2|^p$, we have: 
\begin{align}
 \text{\rm Lip}\leq \omega_2 nm (\max(2^{p-1}(C^X)^2+2^{p-1}(C^Y)^p,2\lambda))^2.\label{eq:lip_fpgw_2}   
\end{align}
\end{lemma}
\begin{proof}

We have: 

\begin{align*}
&\|\nabla\mathcal{L}_{C,M-2\lambda}(\gamma)-\nabla\mathcal{L}_{C,M-2\lambda}(\gamma')\|_F^2\nonumber\\
&=\|(\omega_1 C+\omega_2 ((M-2\lambda)\circ \gamma+(M^\top-2\lambda)\circ\gamma)) - (\omega_1 C+\omega_2  ((M-2\lambda)\circ\gamma'+(M^\top-2\lambda)\circ\gamma'))\|_F^2\\
&=2\omega_2^2\|(M-2\lambda)\circ(\gamma-\gamma')\|_F^2+2\omega_2^2\|(M^\top-2\lambda)\circ(\gamma-\gamma')\|_F^2\nonumber\\
&=2\omega_2^2\sum_{i,j}\left(\left[(M-2\lambda)\circ(\gamma-\gamma')\right]_{i,j}\right)^2+2\omega_2^2\sum_{i,j}\left(\left[(M^\top-2\lambda)\circ(\gamma-\gamma')\right]_{i,j}\right)^2\nonumber\\
&=2\omega_2^2\sum_{i,j}\left(\sum_{i',j'}(M_{i,j,i',j'}-2\lambda)(\gamma_{i',j'}-\gamma'_{i',j'})\right)^2+2\omega_2^2\sum_{i,j}\left(\sum_{i',j'}(M^\top_{i,j,i',j'}-2\lambda)(\gamma_{i',j'}-\gamma'_{i',j'})\right)^2\\
&\leq 2\omega_2^2\sum_{i,j}\left(\sum_{i',j'}|M_{i,j,i',j'}-2\lambda||\gamma_{i',j'}-\gamma'_{i',j'}|\right)^2+
2\omega_2^2\sum_{i,j}\left(\sum_{i',j'}(|M^\top_{i,j,i',j'}-2\lambda||\gamma_{i',j'}-\gamma'_{i',j'}|\right)^2\nonumber \\
&\leq4\omega_2^2\underbrace{\max(|M-2\lambda|)^2}_{A}\cdot
\underbrace{\sum_{i,j}^{n,m}\left(\sum_{i',j'}^{n,m}|\gamma_{i',j'}-\gamma'_{i',j'}|\right)^2}_B.
\end{align*}

Term $B$ can be bounded by \eqref{pf:convergence_bound_B}. Thus, we obtain the upper bound \eqref{eq:lip_fpgw}. Furthermore, when $L(\mathrm{r}_1,\mathrm{r}_2)=|\mathrm{r}_1-\mathrm{r}_2|^p$, from inequality \eqref{eq:ab_inq}, we have: 
\begin{align}
\max (|M-2\lambda|)\leq \max(\max_{i,i'}2^{p-1}(C^X)^p_{i,i'}+\max_{j,j'}2^{p-1}(C^Y)_{j,j'}^p,2\lambda):=\max(2^{p-1}(C^X)^{p-1},2^{p-1}(C^Y)^p,2\lambda)\nonumber 
\end{align}
and we obtain the bound \eqref{eq:lip_fpgw_2}. 
\end{proof}

\begin{proof}[Proof of Proposition \ref{pro:convergence_fpgw}.]
Combining Lemma \ref{lem:lip_bound}, \eqref{eq:diam_2} and \eqref{eq:converge_bound_fpgw}, we obtain the upper bound \eqref{eq:g_k_bound_pgw} and complete the proof. 
\end{proof}

\section{Sinkhon Algorithm for Fused Partial Gromov Wasserstein problem}
For convenience, in this section we default to setting 
$$L(d_X^r,d_Y^R)=\|d_X-d_Y\|^2,$$ 

while all propositions, algorithms and proofs extend without loss of generality to a generic loss function $L(d_X^r,d_Y^r)$. 

\subsection{Sinkhorn Algorithm for FPGW metric}
Given mm-spaces $\mathbb{X}=(X,d_X,\mu),\mathbb{Y}=(Y,d_Y,\nu)$, the general unbalanced Fused-Gromov Wasserstein setting, the entropic problem is defined as the following: 
\begin{align}
\min_{\gamma\in\mathcal{M}_+(\mathcal{X}\times\mathcal{Y})}& \mathcal{L}(\gamma)+\epsilon D_{KL}(\gamma^{\otimes 2}\parallel (\mu\otimes\nu)^{\otimes2})\label{eq:entropy-fugw}\\
\mathcal{L}(\gamma)&:=\omega_1\int_{X\times Y}d(x,y)d\gamma+\omega_2\int_{X^2\times Y^2}|d_X-d_Y|^2 d\gamma^{\otimes 2}\nonumber\\
&\qquad+\lambda(D_\phi(\gamma_1^{\otimes 2}\parallel \mu^{\otimes 2})+D_\phi(\gamma_2^{\otimes 2} \parallel \nu^{\otimes2}))\nonumber
\end{align}
where $D_\phi$ is $\phi-$divergence,

$$D_{KL}(\mu\parallel\nu)=\begin{cases}
\underbrace{\int \ln (\frac{d\mu}{d\nu})d\mu}_{\bar{D}_{KL}(\mu\parallel \nu)}+|\nu|-|\mu|&\text{if }\mu\ll \nu \\
+\infty &\text{elsewhere}
\end{cases}.$$

In the Fused Partial GW setting, $D_\phi$ becomes 
\begin{align}
&D_{PTV} (\mu\parallel\nu):=\begin{cases}
|\mu-\nu|_{TV}=|\nu-\mu| &\text{if }\mu\leq \nu \\
+\infty &\text{elsewhere}
\end{cases}\label{eq:ptv}.
\end{align}
Note, from the definition of \eqref{eq:ptv}, we can restrict the searching space for $\gamma$ to $\Gamma_\leq(\mu,\nu)$: 
\begin{align}
EFPGW (\mathbb{X},\mathbb{Y})&:=\min_{\gamma\in \Gamma_\leq(\mu,\nu)}\mathcal{L}(\gamma)+\epsilon D_{KL}(\gamma^{\otimes 2}\parallel (\mu\otimes \nu)^{\otimes2})\label{eq:entropy-fpgw}\\
&\mathcal{L}(\gamma)=\omega_1\langle c,\gamma \rangle+\omega_2 \langle |d_X-d_Y|^2\otimes \gamma^{\otimes 2} \rangle+\lambda(|\mu|^2+|\nu|^2-2|\gamma|^2) \nonumber  
\end{align}

Problem \eqref{eq:entropy-fpgw} can be further relaxed as:

\begin{align}
&\min_{\gamma,\pi\in\Gamma_\leq(\mu,\nu)}\mathcal{F}(\gamma,\pi)+\epsilon D_{KL}(\pi\otimes \gamma \parallel (\mu\otimes \nu)^{\otimes2})\label{eq:entropy-fpgw-2}\\
&\mathcal{F}(\mu,\nu):=\omega_1 \langle d(x,y),\frac{\gamma+\pi}{2}\rangle+\omega_2\langle |d_X-d_Y|^2,\gamma\otimes \pi\rangle+\lambda(|\mu|^2+|\nu|^2-2|\gamma||\pi|)  \nonumber 
\end{align}
It is clear $\mathcal{F}(\gamma,\gamma)=\mathcal{F}(\gamma)$. Thus, $\eqref{eq:entropy-fpgw-2}\leq \eqref{eq:entropy-fpgw}$, and we denote \eqref{eq:entropy-fpgw-2} as $LB-FPGW_\lambda(\mathbb{X},\mathbb{Y})$ (lower bound of Fused Partial Gromov Wasserstein). Essentially, the Sinkhorn algorithm aims to solve $LB-FPGW$. 

We first introduce the following fundamental proposition. Note, a similar version can be found in \cite[Proposition 4]{sejourne2021unbalanced}:  
\begin{proposition}\label{pro:sinkhorn-fpgw}
Given a fixed $\pi\in \Gamma_\leq(\mu,\nu)$, considering the problem: 
$$\min_{\gamma\in\Gamma_\leq(\mu,\nu)}\mathcal{F}(\pi,\gamma)+\epsilon D_{KL}(\pi\otimes \gamma\parallel (\mu\otimes \nu)^{\otimes 2}),$$
it is equivalent to solve the following entropic optimal partial transport problem: 
\begin{align}
\min_{\gamma\in\Gamma_\leq(\mu,\nu)}\int_{X\times Y} c_{\pi}(x,y) d\gamma+\lambda |\pi| (|\mu|+|\nu|-2|\gamma|)+\epsilon |\pi|D_{KL}(\gamma\parallel \mu\otimes \nu) \label{eq:entropic_pot2}  
\end{align}
where
\begin{align}
c_\pi(x,y)&=\frac{1}{2}\omega_1d(x,y)+\omega_2 [|d_X-d_Y|^2\circ \pi] (x,y)+\epsilon \overline{D}_{KL}(\pi\parallel \mu\otimes 
\nu)\nonumber\\
[|d_X-d_Y|^2\circ \pi](x,y) &=\int_{X\times Y}|d_X(x,x')-d_Y(y,y')|d\pi(x',y')\nonumber 
\end{align} 
\end{proposition}
\begin{proof}
Fix $\pi\in \Gamma_\leq(\mu,\nu)$ in \eqref{eq:entropy-fpgw-2}. 
Since $\pi,\gamma\leq \mu\otimes \nu$, we have $\pi,\gamma\ll \mu\otimes \nu$, thus we have: 
\begin{align}
&D_{KL}(\pi\otimes \gamma\parallel (\mu\otimes \nu)^{\otimes 2})\nonumber\\
&=\int_{(X\times Y)^2}\ln (\frac{d\pi d\gamma}{d\mu\otimes \nu \cdot d\mu\otimes\nu})d\mu d\gamma+(|\mu||\nu|)^2-|\pi||\gamma|\nonumber\\
&=\int_{X\times Y} \left[\int_{X\times Y} \ln(\frac{d\pi}{d\mu\otimes \nu})d\pi\right] d\gamma +\int_{X\times Y} \left[\int_{X\times Y} \ln(\frac{d\gamma}{d\mu\otimes \nu})d\gamma\right] d\pi+(|\mu||\nu|)^2-|\pi||\gamma|\nonumber\\
&=\int_{X\times Y}d\pi \overline{D}_{KL}(\gamma\parallel \mu\otimes \nu)+\int_{X\times Y}d\gamma \overline{D}_{KL}(\pi\parallel \mu\otimes \nu)+(|\mu||\nu|)^2-|\pi||\gamma|\nonumber\\
&=|\pi| D_{KL}(\gamma\parallel \mu\otimes \nu)+|\gamma|\overline{D}_{KL}(\pi\parallel \mu\otimes \nu)+((|\mu||\nu|)^2-|\pi||\mu||\nu|)\label{eq:kl_decomp}
\end{align}
we obtain: 
\begin{align}
&\mathcal{F}(\pi,\gamma)+\epsilon D_{KL}(\gamma\times \pi\parallel (\mu\otimes \nu)^{\otimes 2})\nonumber\\
&=\omega_1 \langle d, \frac{\gamma+\pi}{2}\rangle +\omega_2 \langle |d_X-d_Y|^2,\gamma\otimes \pi\rangle+\lambda(D_{PTV}(\gamma_1\otimes \pi_1\parallel \mu^{\otimes2})+D_{PTV} (\gamma_2\otimes \mu_2\parallel \nu^{\otimes2}))\nonumber\\
&\quad+\epsilon D_{KL}(\gamma\times \pi\parallel (\mu\otimes \nu)^{\otimes 2})\nonumber \\
&=\omega_1 \langle d, \frac{\gamma+\pi}{2}\rangle +\omega_2 \langle |d_X-d_Y|^2,\gamma\otimes \pi\rangle+\lambda(|\mu|^2+|\nu|^2-2|\gamma||\pi|)+\epsilon D_{KL}(\gamma\times \pi\parallel (\mu\otimes \nu)^{\otimes 2})\nonumber\\
&=\underbrace{\frac{1}{2}\omega_1 \langle d,\pi\rangle+\lambda(|\mu|^2+|\nu|^2)+\epsilon [(|\mu||\nu|)^2-|\pi||\mu||\nu|]}_{constant}\nonumber\\
&\quad+\langle \frac{1}{2}\omega_1 d+\omega_2 |d_X-d_Y|^2\circ \pi+\epsilon \overline{D}_{KL}(\pi\parallel\mu\otimes \nu),\gamma  \rangle-2\lambda|\pi||\gamma|)+\epsilon |\pi| D_{KL}(\gamma \parallel \mu\otimes \nu)\nonumber
\end{align}
where we use the fact \eqref{eq:kl_decomp} and 
\begin{align}
&|\gamma|\epsilon \overline{D}_{KL}(\pi\parallel \mu\otimes \nu)=\langle \epsilon \overline{D}_{KL}(\pi\parallel \mu\otimes \nu),\gamma \rangle\nonumber.  
\end{align}
If we ignore the constant part, the remaining part is exactly the entropic partial OT (up to a constant $\lambda|\pi|(|\mu|+|\nu|)$), and we complete the proof. 
\end{proof}

Note, the partial OT problem \eqref{eq:entropic_pot2} can be solved by the classical Sinkhorn algorithm. See e.g. \cite{chizat2018scaling,bai2024sinkhorn}.

\begin{algorithm}\caption{Sinkhorn Partial OT} 
\label{alg:sinkhorn-pot}
\KwInput{$c, \epsilon,\lambda,\mathrm{p},\mathrm{q}$}
\KwOutput{$\gamma$}
Initialize $u=1_n,v = 1_m$, $K = e^{-c/\epsilon}p q^\top $\\
\For{$l=1,2,\ldots$}{
$u = \min (\frac{\mathrm{p}}{Kv},e^{\lambda/\epsilon})$\\
$v = \min (\frac{\mathrm{q}}{K^\top u},e^{\lambda/\epsilon})$ \\
\text{If $(u,v)$ converge}, {break}
}
$\gamma \gets (u_i K_{ij} v_j)_{ij}$
\end{algorithm}

\subsection{Sinkhorn for the fused Mass-constraint Gromov Wasserstein problem.}
Similar to the preious section, in the entropic regularziation setting, the fused mass-constraint Gromov Wasserstein problem is defined as: 
\begin{align}
EFMPGW(\mathbb{X},\mathbb{Y}):=\inf_{\gamma\in\Gamma_\leq^\rho(\mu,\nu)}\omega_1 \langle c,\gamma \rangle+\omega_2 \langle |d_X-d_Y|,\gamma^{\otimes2} \rangle +\epsilon D_{KL}(\gamma^{\otimes2}\parallel (\mu\otimes \nu)^{\otimes2})\label{eq:entropic-fmpgw}
\end{align}
and it can be relaxed as 
\begin{align}
LB-EFMPGW(\mathbb{X},\mathbb{Y})&:=\inf_{\gamma,\pi\in\Gamma_\leq^\rho(\mu,\nu)}\underbrace{\omega_1 \langle c,\frac{1}{2}(\gamma+\pi)\rangle+\omega_2 \langle |d_X-d_Y|,\gamma^{\otimes2} \rangle}_{\mathcal{F}(\gamma,\pi)}\nonumber\\
&\qquad+\epsilon D_{KL}(\gamma\otimes \pi\parallel (\mu\otimes \nu)^{\otimes2})\label{eq:entropic-fmpgw-2} 
\end{align}

Note, since $\frac{1}{2}(\gamma+\pi)\in \Gamma_\leq^\rho(\mu,\nu)$, and $\mathcal{F}(\gamma,\gamma)=\mathcal{F}(\gamma)$, thus, we have 
$$LB-EFMPGW(\mathbb{X},\mathbb{Y})\leq EFMPGW(\mathbb{X},\mathbb{Y}).$$

Similar to the previous section, we have the following property:

\begin{proposition}
Given a fixed $\pi\in \Gamma_\leq^\rho(\mu,\nu)$, considering the problem: 
$$\min_{\gamma\in\mathcal{M}_+(X\times Y)}\mathcal{F}(\pi,\gamma)+\epsilon D_{KL}(\pi\otimes \gamma\parallel (\mu\otimes \nu)^{\otimes 2}),$$
it is equivalent to solve the following entropic optimal partial transport problem: 
\begin{align}
\min_{\gamma\in\Gamma_\leq^\rho(\mu,\nu)}\int_{X\times Y} c_{\pi}(x,y) d\gamma+\epsilon D_{KL}(\gamma\parallel \mu\otimes \nu) \label{eq:entropic_mpot}  \end{align}
where $c_\pi$ is defined in Proposition \ref{pro:sinkhorn-fpgw}. 
\end{proposition}
\begin{proof}
Similar to the proof in Proposition \ref{pro:sinkhorn-fpgw}, given $\pi\in\Gamma^\rho_\leq(\mu,\nu)$, we have 

\begin{align}
&\mathcal{F}(\pi,\gamma)+\epsilon D_{KL}(\gamma\times \pi\parallel (\mu\otimes \nu)^{\otimes 2})\nonumber\\
&=\omega_1 \langle d, \frac{\gamma+\pi}{2}\rangle +\omega_2 \langle |d_X-d_Y|^2,\gamma\otimes \pi\rangle+\epsilon D_{KL}(\gamma\times \pi\parallel (\mu\otimes \nu)^{\otimes 2})\nonumber \\
&=\underbrace{\frac{1}{2}\omega_1 \langle d,\pi\rangle+\epsilon [(|\mu||\nu|)^2-|\pi||\mu||\nu|]}_{constant}\nonumber\\
&\quad+\langle \underbrace{\frac{1}{2}\omega_1 d+\omega_2 |d_X-d_Y|^2\circ \pi+\epsilon \overline{D}_{KL}(\pi\parallel\mu\otimes \nu)}_{c_\pi},\gamma +\epsilon \underbrace{|\pi|}_{=\rho} D_{KL}(\gamma \parallel \mu\otimes \nu)\nonumber
\end{align}
and we complete the proof. 
\end{proof}

The Entropic partial OT problem \eqref{eq:entropic_mpot} can be solved by the corresponded Sinkhorn algorithm \cite{benamou2015iterative,bai2024sinkhorn}, thus we can derive the Sinkhorn algorithm problem for the relaxed fused-mass constraint Gromov Wasserstein problem:

\begin{algorithm}[bt]
   \caption{Sinkhorn Algorithm for FMPGW}
   \label{alg:sink-fmpgw}
\begin{algorithmic}
\STATE 
  {\bfseries Input:} $C\in \mathbb{R}^{n\times m}, C^X\in \mathbb{R}^{n\times n},C^Y\in \mathbb{R}^{m\times m}, p\in \mathbb{R}^n_+, q\in\mathbb{R}^m_+$, $\omega_2\in[0,1],\rho\in[0,\min(|p|,|q|)]$.
   \STATE {\bfseries Output:}
$\gamma$

\FOR{$k=1,2,\ldots$}
\STATE $\pi \gets \gamma$
\STATE Solve the Sinkhorn partial OT problem \eqref{eq:entropic_mpot} via algorithm \eqref{alg:sinkhorn-mopt}:
$$\gamma\gets \min_{\gamma\in\Gamma^\rho_\leq(\mu,\nu)}\int_{X\times Y} c_{\pi}(x,y) d\gamma+\epsilon \rho D_{KL}(\gamma\parallel \mu\otimes \nu) $$

Fix $\gamma$ and solve the similar Sinkhorn partial OT problem \eqref{eq:entropic_mpot} via algorithm \eqref{alg:sinkhorn-mopt}:
 
$$\pi\gets \min_{\gamma\in\Gamma^\rho_\leq(\mu,\nu)}\int_{X\times Y} c_{\gamma}(x,y) d\pi+\epsilon \rho D_{KL}(\pi\parallel \mu\otimes \nu) $$

\STATE Break if $\pi\approx \gamma$
\ENDFOR
\end{algorithmic}
\end{algorithm}

\begin{algorithm}\caption{Sinkhorn Mass-constraint Partial OT}
\label{alg:sinkhorn-mopt}
\KwInput{$p, q, \rho, c$}
\KwOutput{$\gamma$}
\mycommfont{Initialization}\\
$K = e^{-c/\epsilon}pq^\top$\\
\For{i=1,2,3}{
    $\xi^{i} \gets 1_{n \times m}$ \\
}
$\gamma^{(0)} \gets K \frac{\rho}{\|K\|}$ \\
\mycommfont{Main loop}\\
$k=0$
\For{$l=0,1,2,\ldots$}{
    \For{i=1,2,3}{
        $k \gets k+1$ \\ 
        $\gamma^{(k)} \gets \text{Proj}_{\mathcal{C}_i}^{KL}(\gamma^{(k-1)} \odot \xi^{i})$ \\
        $\xi^{i} \gets \xi^{i} \odot \frac{\gamma^{(k-1)}}{\gamma^{(k)}}$
    }
    Break if $\gamma^{(k)}$ converges
}
\end{algorithm}
where 
\begin{align}
\mathcal{C}_1&=\{\gamma\in \mathbb{R}_+^{n\times m}:\gamma_2\leq q\}\nonumber\\
\mathcal{C}_2&=\{\gamma\in \mathbb{R}_+^{n\times m}:\gamma_1\leq p\}\nonumber\\
\mathcal{C}_3&=\{\gamma\in \mathbb{R}_+^{n\times m}:|\gamma|=\rho \}\nonumber\\
\text{Proj}_{\mathcal{C}_i}^{KL}(\gamma)&:=\min_{\gamma^i\in \mathcal{C}_i}KL(\gamma^i\parallel \gamma)=\begin{cases}
    \text{diag}(\min (\frac{p}{\gamma_2},1_n))\gamma \\
    \gamma\text{diag}(\min(\frac{q}{\gamma^\top 1_n},1_m))\\
    \gamma\frac{\eta}{\|\gamma\|}
\end{cases}\nonumber 
\end{align}

\section{Fused Partial Gromov Wasserstein Barycenter}\label{sec:barycenter}
In discrete setting, suppose we have $K$ mm-spaces $\mathbb{X}^k=(X^k\subset\mathbb{R}^d,d_{X^k},\mu^k:=\sum_{i=1}^{n^k}p^k_i\delta_{x_i^k})$ for $k=1,\ldots, K$ and a fixed pmf function $p\in \mathbb{R}_+^n$ where $n\in\mathbb{N}$. Note for each $k\in [1:K]$, let $C^k=[d^r_{X^k}(x^k_i,x^k_{i'})]_{i,i'\in[1:n^k]}\in\mathbb{R}^{n^k\times n^k}$ and let $X^k=[x^k_1,\ldots x^k_{n^k}]^\top\in\mathbb{R}^{n_k\times d}$, we have $(X^k, \mathrm{p}^k, C^k)$ can represent the space $\mathbb{X}^k$. Thus, for convenience, we use the convention $\mathbb{X}^k=(X^k,p^k,C^k)$ to denote the corresponsded mm-space. 

Then, given pmf function $\mathrm{p}\in\mathbb{R}_+^n$ where $n\in\mathbb{N}$, and $\beta_1,\ldots \beta_K\ge 0$ with $\sum_{k=1}^K\beta_k=1$,  the fused-Gromov Wasserstein barycenter problem \cite{titouan2019optimal} is defined as: 

\begin{align}
\min_{C\in\mathbb{R}^{n\times n},X\in\mathbb{R}^{n\times d}}\beta_i FGW_{L}(\mathbb{X},\mathbb{X}^k), &\text{ where }\mathbb{X}=(X,\mathrm{p},C), \label{eq:fused-PGW barycenter}
\end{align}
where we adapt notation $FGW_L(\cdot,\cdot)$ to simplify the notation  $FGW_{r,L}(\cdot,\cdot)$ since $r$ has been incorporated into matrix $C^k$.

Inspired by this work, we present the following fused Partial GW barycenter problems.

\subsection{Fused-MPGW barycenter.}

Choose values   $\rho_1,\rho_2,\ldots,\rho_K$ where $\rho_k\in [0,\min(|p|,|p^k|)]$ for each $k$. In addition, choose $(\omega_1^k,\omega_2^k)\in[0,1]$ for $k=1,2,\ldots K$ such that $\omega_1^k+\omega_2^k=1,\forall k$. 
The fused mass-constrained Partial GW barycenter problem is defined as: 
\begin{align}
&\min_{C\in\mathbb{R}^{n\times n},X\in\mathbb{R}^{n\times d}}\sum_{k=1}^K\beta_k FMPGW_{L,\rho_k}(\mathbb{X},\mathbb{X}^k),  \nonumber\\
&=\min_{\substack{C\in\mathbb{R}^{n\times n}\\
X\in\mathbb{R}^{n\times d}}}\min_{\substack{\gamma^k\in\Gamma^{\rho_k}_\leq(p,p^k)\\
k\in[1:K]}}\sum_{k=1}^K\beta_k \left(\omega_1^k \langle D(X,X^k),\gamma^k\rangle+\omega_2^k\langle L(C,C^k) (\gamma^k)^{\otimes2}\rangle\right),\label{eq:fused-mpgw barycenter} 
\end{align}
where $\mathbb{X}=(X,p,C)$, $\mathbb{X}^k=(X^k,p^k,C^k)$, and  $D(X,X^K)=[\|x_i-x^k_j\|^2]_{i\in[1:n],j\in[1:n^k]}\in\mathbb{R}^{n\times n_k}$. 

Note that the above problem is convex with respect to $(C, X)$ when $\gamma^k$ is fixed for each $k$. However, it is not convex with respect to $\gamma^k$ for each $k$. Similar to classical fused-GW, it can be solved iteratively by updating $(C, X)$ and $(\gamma^k)_{k=1}^K$ alternatively in each iteration. 

\textbf{Step 1.} Given $(C,X)$, we update $(\gamma^k)_{k=1}^K$. 
Note, when $(C,X)$ is fixed,  each optimal $(\gamma^k)^*$ is given by 
\begin{align}
(\gamma^k)^*=\arg\min_{\gamma^k\in\Gamma_\leq^{\rho_k}(p,p^k)}\omega_1^k \langle D(X,X^k),\gamma^k\rangle +\omega_2^k\langle L(C,C^k),(\gamma^k)^{\otimes2}\rangle, \nonumber  
\end{align}
which is a solution for the fused partial GW problem $F-MPGW_{\rho^k}(\mathbb{X},\mathbb{X}^k)$. 

\textbf{Step 2.} Given $\gamma^k$, update $(C,X)$. 

Suppose $\gamma^k$ is given for each $k$, the objective function in \eqref{eq:fused-PGW barycenter} becomes: 
\begin{align}
&\min_{C\in\mathbb{R}^{n\times n},X\in\mathbb{R}^{n\times d}}\sum_{k=1}^K\beta_k \left(\omega_1 \langle D(X,X^k),\gamma^k\rangle+\omega_2\langle L(C,C^k), (\gamma^k)^{\otimes2}\rangle\right)\nonumber\\
&=\omega_1\underbrace{\min_{X\in \mathbb{R}^{n\times d}}\sum_{k=1}^K\beta_k\langle D(X,X^k),\gamma^k \rangle}_{A}+\omega_2 \underbrace{\min_{C\in\mathbb{R}^{n\times n}}\sum_{k=1}^K\beta_k\langle L(C,C^k),(\gamma^k)^{\otimes2} \rangle}_{B} \nonumber. 
\end{align}

Problem $B$ admits the solution (we refer \cite{bai2024efficient} Section M for details). In particular, if $L$ satisfies \eqref{eq:cond_L}, $f_1,h_1$ are differentiable, then
\begin{align}
  C=\left(\frac{f_1'}{h_1'}\right)^{-1}\left(\frac{\sum_k\xi_k\gamma^kh_2(C^k)(\gamma^k)^\top}{\sum_k\xi_k\gamma_1^k(\gamma_1^k)^\top}\right).\label{eq:optimal_C}  
\end{align}

In particular, when $L(\mathrm{r}_1,\mathrm{r}_2)=|\mathrm{r}_1-\mathrm{r}_2|^2$, the above formula becomes:
\begin{align}
C=\frac{\sum_k\xi_k\gamma^kC^k(\gamma^k)^\top}{\sum_k\xi_k\gamma_1^k(\gamma_1^k)^\top} \nonumber. 
\end{align}

\subsection{Solving the subproblem A.}
We first introduce the barycentric projection: 

For each $(X^k,p^k,\gamma^k)$, the barycentric projection \cite{bai2023linear} is defined by 
\begin{align}
\hat x^k_i=\begin{cases}
\frac{1}{\gamma^k_1[i]}\sum_{j=1}^{n_k}\gamma^k_{i,j}x^k_{j} &\text{if }\gamma_1^k[i]=\sum_{j=1}^{n_k}\gamma_{i,j}^k>0, \\
x_i &\text{elsewhere}.
\end{cases}\label{eq:barycentric_proj} 
\end{align}

Note, if $|\gamma^k|=|p|,\forall k$, by \cite{cuturi2014fast} ( Eq 8), the  optimal $X$ is given by 
$$X=[x_1,\ldots x_n]^\top, x_i=\sum_{k=1}^K\beta_k \hat{x}_i^k,\forall i\in[1:n].$$
In this subsection, we will extend the above result to the general case. 

\begin{proposition}\label{pro:optimal_X}
Any matric $X$ satisfies the following is a solution for problem $A$. 
\begin{align}
X=[x_1,\ldots, x_n]^\top, x_i=\frac{\sum_{k=1}^K\beta_k\hat{x}_i^k}{\sum_{k=1}^K\beta_i\gamma^k_1[i]}\qquad\text{if }\sum_{k=1}^K\beta_i\gamma_1^k[i]>0\label{eq:optimal_X}.
\end{align}
\end{proposition}
Note, for each $i$, we use convention $\frac{0}{0}=0$ if $\sum_{k=1}^K\beta_i\gamma_1^k[i]=0$. 
\begin{proof}
Problem A can be written in terms of barycentric projection \eqref{eq:barycentric_proj}. In particular, let $\mathcal{D}_k:=\{i:\sum_{i,j}\gamma^k_{i,j}>0\}$ and $\mathcal{D}=\bigcup_{k=1}^K\mathcal{D}_k$. 
We have: 
\begin{align}
A&=\min_{X\in\mathbb{R}^{n\times d}}\sum_{k=1}^K\beta_k\sum_{i=1}^n\gamma_1^k[i](\|x_i-\hat x_i^k\|^2)\nonumber\\
&=\min_{X\in\mathbb{R}^{n\times d}}\sum_{i=1}^n\sum_{k=1}^K\beta_k\gamma_1^k[i](\|x_i-\hat x_i^k\|^2)\nonumber\\
&=\sum_{i=1}^n\min_{x_i\in\mathbb{R}^d}\sum_{k=1}^K\beta_k\gamma_1^k[i](\|x_i-\hat x_i^k\|^2).\nonumber
\end{align}
For each $i$, we have two cases:

Case 1:  $\sum_{k=1}^K\beta_k\gamma_1^k[i]>0$. 
Then optimal $x_i$ is given by the weighted average vector:
$$\frac{\sum_{k\in D_i}\beta_k\gamma_1^k[i]\hat{x}^k_i}{\sum_{i\in D_k}\beta_k\gamma^k_1[i]}.$$

Case 2:  $\sum_{k=1}^K\beta_k\gamma_1^k[i]=0$. The problem becomes 
$$\min_{x_i\in \mathbb{R}^d}0,$$
and there is no requirement for $x_i$. 
And we complete the proof. 

\end{proof}
\begin{remark}
In practice, the above formulation can be described by matrices. In particular, let 
$\hat{X}^k=[\hat{x}^k_1,\ldots \hat{x}^k_n]^\top$, we have  
$$\hat{X}^k=\frac{\gamma^k X^k}{\gamma^k_1},\gamma^k_1=\gamma^k 1_m.$$
Then we have \eqref{eq:optimal_X} becomes
$$X=\frac{\sum_{k=1}^K\beta_k\gamma_1^k1_d^\top\odot\hat{X}}{\sum_{k=1}^K\beta_k\gamma_1^k},$$
where $\odot$ denotes the element-wise multiplication; the notation $\frac{A}{B}$ where $A\in \mathbb{R}^{n\times d}, B\in \mathbb{R}^{n}$ denotes the following
$$\frac{A}{B}=[A[:,1]/B_1,\ldots,A[:,n]/B_n]^\top,$$
and we use $\frac{0}{0}=0$ if any element in the denominator is 0.

\end{remark}

\subsection{Fused-PGW barycenter}
Similar to the previous subsection, we define and derive the solution for the fused-Partial GW problem. 

Given $\lambda_1,\ldots \lambda_K\ge 0, p\in \mathbb{R}_+^n$, the fused partial GW barycenter problem is defined as: 

\begin{align}
&\min_{C\in\mathbb{R}^{n\times n},X\in\mathbb{R}^{n\times d}}\sum_{k=1}^K\beta_k FPGW_{L,\lambda_i}(\mathbb{X},\mathbb{X}^k), \text{where }\mathbb{X}=(X,p,C)\nonumber\\
&=\min_{C\in\mathbb{R}^{n\times n},X\in \mathbb{R}^{n\times d}} \min_{\gamma^k\in\Gamma_\leq(p,p^k):k\in[1:K]}\sum_{k=1}^K\omega_1^k\langle D(X,X^k),\gamma^k\rangle \nonumber\\
&\qquad+\omega_2^k\langle L(C,C^k),(\gamma^k)^{\otimes2}\rangle+\lambda_k(|p^k|^2+|p|^2-2|\gamma^k|^2).
\label{eq:fpgw_barycenter}
\end{align}

Similar to the fused MPGW barycenter problem, the above problem can be solved iteratively. In each iteration, we have two steps: 

\text{Step 1.} Given $X,D$, update $\gamma^k$. 
Note finding each $\gamma^k$ is essentially solving the fused-PGW problem: 
\begin{align}
\gamma^k=\arg\min_{\gamma\in\Gamma_\leq(p,p^k)}\omega_1^k\langle D(X,X^k),\gamma\rangle+\omega_2^k \langle L(C,C^k),(\gamma^{\otimes2}) \rangle+\lambda_k(|p^k|^2+|p|^2-2|\gamma^k|^2) \nonumber   
\end{align}
And it can be solved by algorithm \ref{alg:fpgw}. 

\text{Step 2.} Given $\gamma^k:k\in[1:K]$, update $C,X$.

In this case, \eqref{eq:fpgw_barycenter} becomes 
\begin{align}
\omega_2^k\underbrace{\min_{C\times \mathbb{R}^{n\times n}} \langle L(C,C^k)-2\lambda_k,(\gamma^k)^{\otimes2} \rangle}_{A}+ \omega_1^k\underbrace{\min_{X\in\mathbb{R}^{n\times d}}\langle D(X,X^k),\gamma^k\rangle}_{B}.
\end{align}

Optimal $C$ for subproblem A is \eqref{eq:optimal_C} by  [Proposition M.2 \cite{bai2024efficient}]. Optimal $X$ for subproblem B is given by \eqref{eq:optimal_X} by the Proposition \ref{pro:optimal_X}. 

\section{Relation between FGW, FPGW and FMPGW.}
In this section, we briefly discuss the relation between FGW, Fused-PGW problem \eqref{eq:fpgw}, and the fused-MPGW \eqref{eq:fmpgw}.

First, we introduce the following ``equivalent relation'' between FPGW and FMPGW. It can be treated as the generalization of Proposition L.1. in \cite{bai2024efficient} in the fused-PGW formulation. 

\begin{proposition}\label{pro:fpgw_fmpgw}
Given mm-spaces $\mathbb{X}=(X,d_X,\mu),\mathbb{Y}=(Y,d_Y,\nu)$, $r\ge 1,\lambda>0$. Suppose $\gamma^*$ is a minimizer for FPGW problem $FPGW_{r,L,\lambda}(\mathbb{X},\mathbb{Y})$, then $\gamma^*$ is also a minimizer for Fused-MPGW problem $FPGW_{r,L,\rho}(\mathbb{X},\mathbb{Y})$ where $\rho=|\gamma^*|$. 
\end{proposition}
\begin{proof}
Pick $\gamma\in\Gamma_\leq^\rho(\mu,\nu)\subset \Gamma_\leq(\mu,\nu)  $. Since $\gamma^*$ is optimal for $FPGW_{r,L,\lambda}(\mathbb{X},\mathbb{Y})$, we have: 
\begin{align}
&\omega_1\langle C,\gamma^*\rangle +\omega_2\langle L,(\gamma^*)^{\otimes2}\rangle+\lambda(|\mu|^2+|\nu|^2-2|\gamma^*|^2)  \nonumber\\
&\leq \omega_1\langle C,\gamma\rangle +\omega_2\langle L,\gamma^{\otimes2}\rangle+\lambda(|\mu|^2+|\nu|^2-2|\gamma|^2).  \nonumber
\end{align}
Combine it with the fact $|\gamma|=|\gamma^*|=\rho$, we complete the proof. 
\end{proof}

Next, we discuss the relation between FGW and FMPGW problems. 
\begin{proposition}\label{pro:fmpgw_fgw}
Under the setting of Proposition \ref{pro:fpgw_fmpgw}, 
suppose $|\mu|=|\nu|$, then 
$$FMPGW_{r,L,\rho}(\mathbb{X},\mathbb{Y})=FGW_{r,L,\rho}(\mathbb{X},\mathbb{Y}).$$
\end{proposition}
\begin{proof}
In this case, $\Gamma_\leq^\rho(\mu,\nu)=\Gamma(\mu,\nu)$ and we complete the proof. 
\end{proof}

Similarly, in the extreme case, the FPGW problem can also recover the FGW problem. We first introduce the following lemma: 
\begin{lemma}\label{lem:big_lambda}
Suppose $\mu,\nu$ are supported in compact set and $|\mu|,|\nu|>0$. 
Let $X=\text{supp}(\mu),Y=\text{supp}(\nu)$, and
$\max L(d_X^r,d_Y^r):=\max_{x,x'\in X,y,y'\in Y}L(d_X^r(x,x'),d_Y^r(y,y')),\max C:=\max_{x\in X,y\in Y}C(x,y)$. 

Suppose 

\begin{align}
  &2\lambda\ge \omega_1\frac{1}{\min(|\mu|,|\nu|)}\max (C)+\omega_2 \max L(d_X^r,d_Y^r)\label{eq:big_lambda_cond}
\end{align}
there exists a solution, denoted as $\gamma^*$, for $FPGW_{r,L,\lambda}(\mathbb{X},\mathbb{Y})$ such that
\begin{align}
|\gamma^*|=\min(|\mu|,|\nu|)\label{eq:big_lambda_res}.
\end{align} 
If we replace ``$\ge$'' by ``$>$'' in the inequality
$\eqref{eq:big_lambda_cond}$, then every solution of $FPGW_{r,L,\lambda}(\mathbb{X},\mathbb{T})$ satisfies $\eqref{eq:big_lambda_res}$. 
\end{lemma}

\begin{proof}
For convenience, we suppose $|\mu|\leq |\nu|$.
Suppose \eqref{eq:big_lambda_cond} holds. 
Choose an optimal $\gamma\in\Gamma_\leq(\mu,\nu)$. 
By lemma E.1. in \cite{bai2024efficient}, there exists $\gamma'\in\Gamma_\leq(\mu,\nu)$ such that $\gamma\leq \gamma'$ with $|\gamma'|=|\mu|$, i.e. $\gamma'_1=\mu$. 

We have: 
\begin{align}
& \int_{X\times Y} \omega_1C(x,y) d\gamma' +\int_{(X\times Y)^2} \omega_2L(d_X^r,d_Y^r) d\gamma'^{\otimes2}+\lambda(|\mu|^2+|\nu|^2-2|\gamma'|^2)\nonumber\\
&- \int_{X\times Y}\omega_1 C(x,y) d\gamma +\int_{(X\times Y)^2}\omega_2 L(d_X^r,d_Y^r) d\gamma^{\otimes2}+\lambda(|\mu|^2+|\nu|^2-2|\gamma|^2)\nonumber\\
&= \int_{X\times Y}\omega_1C(x,y) d(\gamma'-\gamma)+ \int_{(X\times Y)^2}\omega_2L(d_X^r,d_Y^r)-2\lambda  d(\gamma'^{\otimes 2}-\gamma^{\otimes2})\nonumber\\
&\leq \omega_1 \max (C)|\gamma'-\gamma|+\omega_2 (\max (L(d_X^r,d_Y^r))-2\lambda)|\gamma'^{\otimes2}-\gamma^{\otimes2}|\nonumber\\
&=\underbrace{\left(\omega_1 \frac{\max(C)}{|\gamma+\gamma'|}+\omega_2 \max (L(d_X^r,d_Y^r))-2\lambda\right)}_{A}\underbrace{|\gamma'^{\otimes2}-\gamma^{\otimes2}|}_B\nonumber\\
&\leq 0 \label{pf:gamma'_optimal},
\end{align}
where \eqref{pf:gamma'_optimal} follows by the fact $A\leq 0,B\ge 0$. 
Thus we have $\gamma'$ is optimal.

Now we suppose
$$2\lambda>\omega_1\frac{1}{\min(|\mu|,|\nu|)}\max (C)+\omega_2 \max L(d_X^r,d_Y^r).$$
We assume there exists an optimal $\gamma$ such that  $|\gamma|<|\mu|$.

If we can find optimal $\gamma$ with $|\gamma|<|\mu|$, then $A<0$ and $B>0$ and we have
$\eqref{pf:gamma'_optimal}<0$. 
That is, $\gamma'$ admits a smaller cost. It is a contradiction since $\gamma$ is optimal, and we complete the proof.  
\end{proof}
\begin{remark}
If $\min(|\mu|,|\nu|)=0$, $\Gamma_\leq(\mu,\nu)=\{\mathbf{0}\}$ where $\{\mathbf{0}\}$ is the zero measure. In this case, \eqref{eq:big_lambda_res} is automatically satisfied, and there is no requirement for $\lambda$. 
\end{remark}
Based on this lemma, the FPGW can recover FGW in the extreme case: 
\begin{proposition}\label{pro:fpgw_fgw}
Under the setting of Proposition \ref{pro:fmpgw_fgw}, suppose $\lambda$ satisfies \eqref{eq:big_lambda_cond}, then 
$$FPGW_{r,L,\lambda}(\mathbb{X},\mathbb{Y})=FGW_{r,L}(\mathbb{X},\mathbb{Y}).$$
\end{proposition}
\begin{proof}
By Lemma \ref{lem:big_lambda}, there exists optimal $\gamma$ for $FPGW_{r,L,\lambda}(\mathbb{X},\mathbb{Y})$ with $|\gamma|=|\mu|=|\nu|$. By Proposition \ref{pro:fpgw_fmpgw}, we have $\gamma$ is optimal for $FMPGW_{r,L,\rho}(\mathbb{X},\mathbb{Y})$ where $\rho=|\mu|=|\nu|$. 
By Proposition \ref{pro:fmpgw_fgw}, we have $\gamma$ is optimal for $FGW_{r,L}(\mathbb{X},\mathbb{Y})$. Thus, $\gamma$ and we complete the proof. 
\end{proof}

\section{Measure graph similarity via Fused GW/PGW.}\label{sec:graph_mm}
In the following, we discuss how to model graphs as mm-spaces. Based on this modeling, fused-GW and fused-PGW can be adapted to measure graph similarity.

Given graphs $G_1 = (V_1, E_1)$ and $G_2 = (V_2, E_2)$, where $V_1, V_2$ are sets of vertices (nodes) and $E_1, E_2$ are sets of edges.

First, consider $G_1$. Suppose $V_1 = \{v_1^1, \ldots, v^1_{N_1}\}$. We construct the mm-space $\mathbb{X}_1 = (V_1, d_{V_1}, \mu^1 = \sum_{i=1}^{N_1} \mathrm{p}_i^1 \delta_{v_i})$, where $\mathrm{p}_i^1 > 0$ for all $i$. Let $\mathrm{p}^1 = [p_1^1, \ldots, p_{N_1}^1]$. Note that when $\sum(\mathrm{p}^1) = 1$, $\mathrm{p}^1$ represents the probability mass function (pmf) of the measure $\mu^1$. For general cases, we refer to $\mathrm{p}^1$ as the mass function (mf) of $\mu^1$.

In this formulation, the function $d_{V_1}: V_1^2 \to \mathbb{R}$ can be defined as follows:
\begin{itemize}
    \item \textbf{Adjacency indicator function:} $d_{V_1}(v_1, v_2) = 1$ if $(v_1, v_2) \in E_1$, and $d_{V_1}(v_1, v_2) = 0$ otherwise.
    \item \textbf{Shortest distance function:} $d_{V_1}(v_1, v_2)$ is the length of the shortest path connecting $v_1$ and $v_2$ if such a path exists, or $\infty$ if $v_1$ and $v_2$ are not connected.
\end{itemize}

Let $\mathcal{F}$ denote the space of feature assignments for all nodes. We define a feature function $f: V \to \mathcal{F}$ such that $x_i = f(v_i)$ represents the feature of the vertex $v_i$. For convenience, when the graph has discrete features, $\mathcal{F}$ is a discrete set; when the graph has continuous features, $\mathcal{F} = \mathbb{R}^d$, where $d$ is the dimension of each feature. In both cases, we define a metric $d_\mathcal{F}$ in the feature space.

The feature similarity function $d_\mathcal{F}$ is set as follows:
\begin{itemize}
    \item \textbf{Continuous feature graph:} Since $\mathcal{F} = \mathbb{R}^d$, we define $d_\mathcal{F}$ as the (squared) Euclidean distance in $\mathbb{R}^d$.
    \item \textbf{Discrete feature graph:} We first apply the Weisfeiler-Lehman kernel \cite{vishwanathan2010graph}, which encode elements in feature space as $\text{wl}: \mathcal{F} \to S^H$, where $S$ is finite discrete set, 
    $H \in \mathbb{N}$, typically set to $2$ or $4$. We then use the Hamming distance in $S^\mathrm{h}$ as $d_\mathcal{F}$. For details, refer to Section 4.2 of \cite{feydy2017optimal}.
\end{itemize}

\section{Rapameter Setting in Graph Matching and Geometry Matching Experiments}
We present the detailed parameter settings for the graph and geometry matching experiments in this section.  
First, consider the graph node distribution settings  
$$
\mu = \sum_{i=1}^n p_i \delta_{v_i^q}, \quad 
\nu = \sum_{j=1}^m q_j \delta_{v_j^o},
$$  
where $\mu$ denotes the node distribution of the query graph and $\nu$ denotes the node distribution of the original graph.  

For unbalanced GW methods,  \textbf{SpecGW} (Spectral Gromov--Wasserstein~\cite{chowdhury2021generalized}), 
\textbf{eBPG} (Entropic Bregman Projected Gradient~\cite{Solomon2016Entropic}), 
\textbf{BPG} (Bregman Projected Gradient~\cite{Xu2019Gromov}), 
\textbf{BAPG} (Bregman Alternating Projected Gradient~\cite{Li2023Convergent}), 
and \textbf{srGW} (Semi-relaxed Gromov--Wasserstein~\cite{VincentCuaz2022Semi}), 
we follow their default settings
\begin{align}
p_i = \tfrac{1}{n}, \quad q_j = \tfrac{1}{m}.\label{eq:balanced_p}    
\end{align}

For \textbf{PGW} (Partial Gromov--Wasserstein~\cite{chapel2020partial,bai2024efficient}), 
\textbf{UGW} (Unbalanced Gromov--Wasserstein~\cite{sejourne2021unbalanced}), 
\textbf{RGW} (Outlier Robust Gromov--Wasserstein~\cite{kong2024outlier}), 
\textbf{FUGW} (Fused Unbalanced Gromov--Wasserstein~\cite{thual2022aligning}), 
and \textbf{Sink-FUPGW} (Sinkhorn Fused Partial Gromov--Wasserstein, ours), 
in addition to the above balanced setting, we alternatively consider the unbalanced setting: 
\begin{align}
    p_i = q_j = \tfrac{1}{\min(n,m)} = \tfrac{1}{n}.\label{eq:unbalanced_p}
\end{align}
And select the best option for each method.  

Other parameter settings are summarized in Tables~\ref{tab:parameter_gm} and \ref{tab:parameter_geometry}.

In particular, in Spectral GW, $time$ is the time limit of the heat kernel. In methods: eBPG, BPG, BAPG, srGW, UGW, FUGW, RGW, and sink-FPGW, $\epsilon$ is the weight of the entropic regularization term. In BAPG, $\rho$ is the weight of the Bregman divergence penalty. In PGW, ``$\text{mass}$'' is the mass constraint.  In UGW, FUGW, sink-FPGW $\rho$ is the weight of marginal regularization terms. In FPGW, $\lambda$ is the weight of marginal regularization terms. In the fused methods (FGW, FUGW, sink-FPGW), 
$\alpha$ denotes the balance among graph structure and node features. In RGW, $\rho$ is the hard marginal constraint, $\eta$ is the weight of the marginal constraint (soft constraint), $t$ are the stepsizes in Bregman proximal alternating linearized minimization (BPALM).

\textbf{Other settings in geometry matching.}
In geometric matching, we only consider unbalanced GW methods (RGW, FUGW, FPGW), since UGW and PGW can be regarded as special cases of FUGW and FPGW and are therefore not included separately. We adopt the unbalanced mass function setting in \eqref{eq:unbalanced_p}. 

In this experiment, we further assume that a fixed pair of points from the ground-truth correspondence is known, denoted by $(x_0^q, x_0^o)$, where the first point belongs to the query shape and the second to the original shape. For each point $x$ in the query shape, its feature is defined as $d(x, x_0^q)$, where $d$ is the Euclidean distance on the shape. Features of points in the original shape are defined analogously.

\begin{table*}[h!]
\captionsetup{skip=2pt}
\caption{The parameter settings in all the methods. Datasets with node attributes include Synthetic, Enzymes, Cuneiform, COX2, BZR, Protein and AIDS.}
\label{tab:parameter_gm}
\vskip -0.10in
\begin{center}
\begin{small}

\begin{tabular}{lccc} 
\toprule
 &parameter&{Dataset with node attributes}&{Douban}\\
\midrule 

SpecGW & $time$ & 10 & 10\\
\hline
eBPG &$\epsilon$ & 0.1 & 0.01\\
\hline
BPG &$\epsilon$ & 0.2 & 0.01  \\
\hline
\multirow{2}{*}{BAPG} &$\rho$ & 0.1 & 0.01  \\
&$\epsilon$ & 1e-6 & 1e-6\\
\hline
{srGW} &$\epsilon$ & 2.0 & 10  \\
\hline
PGW & $\text{mass}$ & $\min(|p|
,|q|)$ & $\min(|p|
,|q|)$\\
\hline
\multirow{2}{*}{UGW} &$\epsilon$ & 0.01 & 1e-3 \\
&$\rho$ & 0.05 & 1.0\\
\hline
\multirow{2}{*}{RGW} &$\rho$ & 0.05/0.1 & 0.1  \\
&$\epsilon$ & 0.05 & 1e-3\\
&$t$ & 0.1 & 0.1\\
&$\tau$ & 0.1 & 0.1\\
\hline
FGW &$\alpha$ & 0.5 & 0.5  \\
\hline
\multirow{3}{*}{FUGW} &$\alpha$ & 0.5  &  0.5 \\
&$\epsilon$ & 0.01 & 2e-4 \\
&$\rho$ & 1.0 & 1.0 \\
\hline
\multirow{3}{*}{sink-FPGW(ours)} &$\alpha$ & 0.33 & 0.5  \\
&$\epsilon$ & 0.02 & 1e-4\\
&$\lambda$ & 1.0 & 1.0 \\

\bottomrule
\end{tabular}

\end{small}
\end{center}
\end{table*}

\begin{table*}[h!]
\captionsetup{skip=2pt}
\caption{Parameter setting of all the methods in geometry matching.}
\label{tab:parameter_geometry}
\vskip -0.10in
\begin{center}
\begin{small}

\begin{tabular}{lccc} 
\toprule
 &parameter&Upper Body&Lower Body\\
\midrule 

\multirow{2}{*}{RGW} &$\rho$ &0.1 & 0.1  \\
&$\epsilon$ & 1e-3 & 1e-3\\
&$t$ & 0.1 & 0.1\\
&$\tau$ & 0.1 & 0.1\\
\hline
\multirow{3}{*}{FUGW} &$\alpha$ & 0.5  &  0.5 \\
&$\epsilon$ & 0.005 & 0.1 \\
&$\rho$ & 1.0 & 1.0 \\
\hline
\multirow{3}{*}{sink-FPGW(ours)} &$\alpha$ & 0.5 & 0.5  \\
&$\epsilon$ & 1e-3 & 1e-4\\
&$\lambda$ & 1.0 & 1.0 \\

\bottomrule
\end{tabular}

\end{small}
\end{center}
\end{table*}

\section{Graph Classification}

\textbf{Dataset Setup.} We consider four widely used benchmark datasets, divided into two groups. The first group includes \textit{Mutag} \cite{debnath1991structure} and \textit{MSRC-9} \cite{rossi2015network}, which consist of graphs with discrete attributes. The second group consists of vector-attributed graphs, including \textit{Synthetic} \cite{feragen2013scalable}, \textit{Cuneiform} \cite{kriege2016valid}, and \textit{Proteins} \cite{borgwardt2005shortest}.

For each dataset, we randomly select $50\%$ of the graphs and add outlier nodes. Specifically, for each selected graph, suppose $N$ is the number of vertices. We manually add $\eta N$ extra nodes, where $\eta \in \{0, 10\%, 20\%, 30\%\}$ represents the ``level of outliers/noise''. For clarity, nodes that are not outliers are referred to as ``regular nodes''. The outlier nodes are randomly connected to both the regular nodes and each other. The features of the outlier nodes are defined as follows:
\begin{itemize}
    \item For graphs with discrete features, each outlier node is assigned a label that does not occur among the regular nodes in the graph.
    \item For graphs with continuous features, suppose all features lie within a compact set $[x_1, y_1] \times \ldots \times [x_d, y_d] \subset \mathbb{R}^d$, where $d \in \mathbb{N}$ is the dimension of the feature space. We assign features to the outlier nodes by sampling vectors from $[y_1, y_1 + 2\text{sd}_1] \times \ldots \times [y_d, y_d + 2\text{sd}_d]$, where $\text{sd}_i$ is the standard deviation of all node features in dimension $i$.
\end{itemize}

Furthermore, for each graph $G$, let $N_G$ denote the number of regular nodes (i.e., nodes that are not outliers). We assume $N_G$ is known.

\begin{figure}[t]
    \centering
   \vspace{0.2in}
\includegraphics[width=\columnwidth]{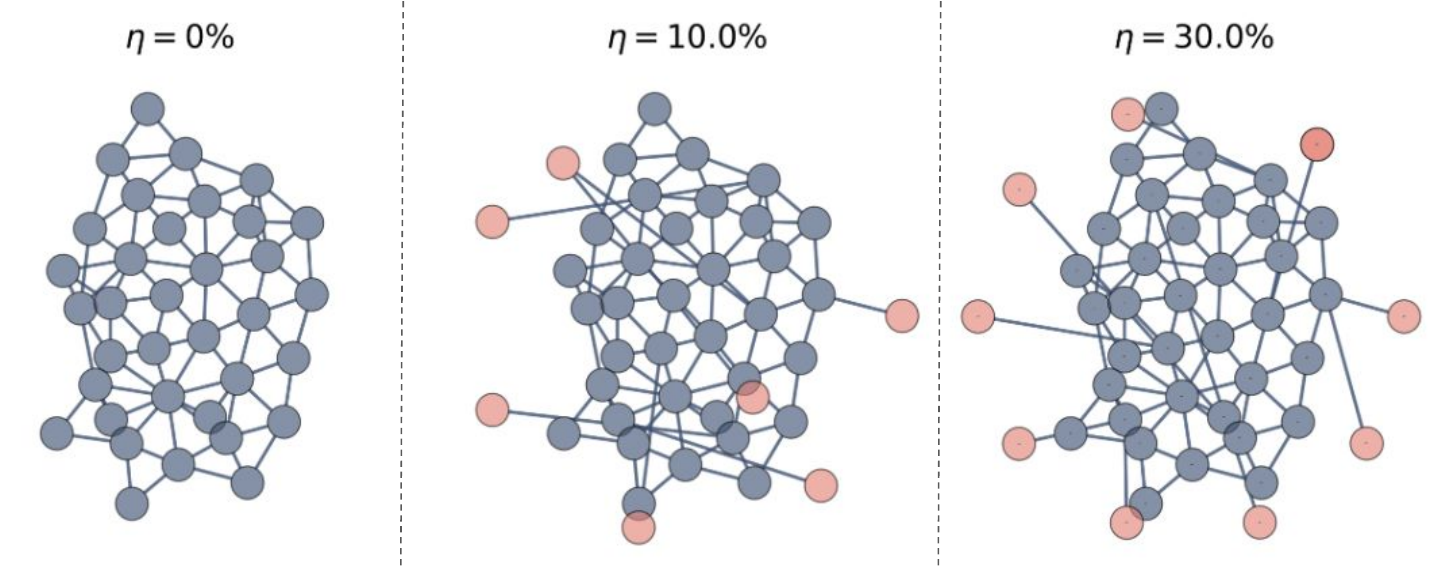}
    \vspace{-0.3in}
    \caption{We visualize the graph classification data in this plot. The parameter $\eta\in\{0,10\%,20\%,30\%\}$ represents the proportion of outlier nodes. The blue nodes are the regular nodes in the graph, and the pink nodes are the outliers.   
    }
    \label{fig:fig3}
\end{figure}

\textbf{Fused Gromov/Unbalanced Gromov/Partial Gromov Setting.}

We adapt the process in section \ref{sec:graph_mm} to model each graph $G$ as a mm-space. 
In particular, in the FGW setting, we define the (probability) mass function $\mathrm{p}$ as $\mathrm{p}_i = \frac{1}{N}$, where $N$ is the number of nodes in graph $G$. In the FUGW/FPGW setting, we define the mass function $\mathrm{p}$ as $\mathrm{p}_i = \frac{1}{N_G}$, where $N_G$ is the number of regular nodes in $G$. 

For the distance functions $d_{V_1}$ and $d_{V_2}$ in graphs $G_1$ and $G_2$, we use the ``shortest path'' metric in this experiment. We select FMPGW \eqref{eq:fmpgw} and set the parameter $\rho = 1$.

\textbf{Baseline methods.} 
We consider the fused Gromov Wasserstein \cite{feydy2017optimal} and the Unbalanced fused Gromov Wasserstein \cite{thual2022aligning} as baselines for both discrete and continuous feature graph datasets. In addition, 
for discrete feature graph datasets, we add 
the Weisfeler Lehman kernel (WLK) \cite{vishwanathan2010graph},Graphlet count kernel (GK) \cite{shervashidze2009efficient}, random walk kernel (RWK) \cite{kriege2018recognizing}, ODD-STh Kernel  (ODD) \cite{da2012tree}, vertex histogram (VH) \cite{sugiyama2015halting}, Lovasz Theta (LT) \cite{lovasz1979shannon}, SVM theta (ST) \cite{jethava2013lovasz} as baselines methods. For the continuous feature graphs dataset, we also
consider the HOPPER kernel \cite{feragen2013scalable} and the propagation kernel Neumann \cite{neumann2016propagation}. 

\textbf{Classifier setup.} We adapt the SVM classification model in this experiment. In particular, given a graph dataset, we first compute the pairwise distance via FGW/(other baselines)/FPGW. By using the approach given by \cite{beier2022linear, vay2019fgw}, we combine each distance with a support vector machine (SVM), applying stratified 10-fold cross-validation. In each iteration of cross-validation, we train an SVM using $\exp(-\sigma D)$ as the kernel, where $D$ is the matrix of pairwise distances (w.r.t. one of the considered distances) restricted to 9 folds, and compute the accuracy of the model on the remaining fold. We report the accuracy averaged over all 10 folds for each model.
\begin{table*}[t]
\centering
\caption{Kernel accuracy (\%) with standard deviation under different noise levels (0\%, 10\%, 20\%, 30\%). Top: results for discrete node features. Bottom: results for continuous node features.}
\setlength{\tabcolsep}{3.5pt}
\begin{tabular}{lcccccccc}
\toprule
& \multicolumn{4}{c}{MSRC-9} & \multicolumn{4}{c}{MUTAG} \\
\cmidrule(lr){2-5} \cmidrule(lr){6-9}
& $0\%$ & $10\%$ & $20\%$ & $30\%$ & $0\%$ & $10\%$ & $20\%$ & $30\%$ \\
\midrule
WLK & 85.5$^{\pm5.9}$ & 86.5$^{\pm5.2}$ & 86.9$^{\pm4.7}$ & 86.0$^{\pm3.7}$ & 75.0$^{\pm6.3}$ & 73.4$^{\pm6.7}$ & 72.4$^{\pm7.0}$ & 73.5$^{\pm7.8}$ \\
GK & 15.8$^{\pm4.5}$ & 15.4$^{\pm3.7}$ & 14.9$^{\pm6.2}$ & 17.7$^{\pm8.0}$ & 77.1$^{\pm6.3}$ & 70.2$^{\pm2.6}$ & 70.2$^{\pm4.4}$ & 67.0$^{\pm6.6}$ \\
RWK & 74.7$^{\pm5.8}$ & 73.8$^{\pm6.6}$ & 74.7$^{\pm7.4}$ & 74.2$^{\pm4.1}$ & 66.5$^{\pm2.3}$ & 66.5$^{\pm2.3}$ & 66.5$^{\pm2.3}$ & 66.5$^{\pm2.3}$ \\
ODD & 67.9$^{\pm6.5}$ & 69.2$^{\pm7.6}$ & 64.7$^{\pm8.2}$ & 63.8$^{\pm9.0}$ & 64.9$^{\pm4.0}$ & 64.9$^{\pm4.0}$ & 64.9$^{\pm4.0}$ & 64.9$^{\pm4.0}$ \\
VH & 86.4$^{\pm4.0}$ & \textbf{86.9}$^{\pm4.2}$ & \textbf{87.8}$^{\pm4.5}$ & \textbf{87.8}$^{\pm4.0}$ & 66.5$^{\pm2.3}$ & 66.5$^{\pm2.3}$ & 66.5$^{\pm2.3}$ & 66.5$^{\pm2.3}$ \\
LT & 15.8$^{\pm5.0}$ & 13.1$^{\pm4.3}$ & 13.1$^{\pm5.1}$ & 14.5$^{\pm4.5}$ & 69.7$^{\pm3.9}$ & 68.1$^{\pm2.5}$ & 66.5$^{\pm2.3}$ & 68.6$^{\pm8.5}$ \\
ST & 13.6$^{\pm0.2}$ & 13.6$^{\pm0.2}$ & 13.6$^{\pm0.2}$ & 13.6$^{\pm0.2}$ & 75.0$^{\pm2.7}$ & 72.9$^{\pm3.7}$ & 72.3$^{\pm3.9}$ & 72.3$^{\pm3.9}$ \\
\midrule
RGW & {79.7}$^{\pm 0.1}$ & {76.1}$^{\pm 0.1}$ & 59.8$^{\pm 0.1}$ & 60.2$^{\pm 0.1}$ & {81.8}$^{\pm 0.1}$ & {80.3}$^{\pm 0.1}$ & 77.7$^{\pm 0.1}$ & 76.1$^{\pm 0.0}$ \\

FGW & \textbf{87.4}$^{\pm4.4}$ & \textbf{86.9}$^{\pm4.7}$ & 63.8$^{\pm5.8}$ & 62.0$^{\pm5.0}$ & \textbf{85.6}$^{\pm5.6}$ & \textbf{83.5}$^{\pm5.7}$ & 79.8$^{\pm6.5}$ & 76.6$^{\pm6.8}$ \\
FUGW & 73.8$^{\pm6.1}$ & 68.3$^{\pm4.9}$ & 5.0$^{\pm3.7}$ & 5.4$^{\pm4.4}$ &
82.4$^{\pm5.5}$ & 81.9$^{\pm5.6}$ & \textbf{81.9}$^{\pm6.1}$ & \textbf{80.3}$^{\pm6.0}$ \\
FPGW (ours) & \textbf{87.0}$^{\pm4.7}$ & \textbf{88.3}$^{\pm4.1}$ & \textbf{87.3}$^{\pm4.4}$ & \textbf{86.9}$^{\pm4.7}$ & \textbf{85.6}$^{\pm5.6}$ & \textbf{85.1}$^{\pm5.9}$ & \textbf{84.6}$^{\pm5.6}$ & \textbf{82.5}$^{\pm6.3}$ \\
\midrule
\midrule
& \multicolumn{4}{c}{PROTEINS} & \multicolumn{4}{c}{SYNTHETIC} \\
\cmidrule(lr){2-5} \cmidrule(lr){6-9}
& $0\%$ & $10\%$ & $20\%$ & $30\%$ & $0\%$ & $10\%$ & $20\%$ & $30\%$ \\
\cmidrule(lr){2-5} \cmidrule(lr){6-9}
Propagation & 59.6$^{\pm0.2}$ & 59.6$^{\pm0.2}$ & 59.6$^{\pm0.2}$ & 59.6$^{\pm0.2}$ & 48.0$^{\pm7.0}$ & 51.3$^{\pm6.5}$ & 47.3$^{\pm5.5}$ & 53.0$^{\pm6.1}$ \\
GraphHopper & 69.7$^{\pm3.8}$ & 68.2$^{\pm4.8}$ & 67.3$^{\pm3.5}$ & 67.2$^{\pm4.7}$ & 86.0$^{\pm5.3}$ & 78.7$^{\pm9.5}$ & \textbf{70.3}$^{\pm7.1}$ & \textbf{70.7}$^{\pm10.7}$ \\
\midrule
RGW & {59.6}$^{\pm 0.0}$ & {47.3}$^{\pm 0.1}$ & 44.7$^{\pm 0.1}$ & 51.4$^{\pm 0.1}$ & 46.0$^{\pm 0.1}$ & 50.0$^{\pm 0.0}$ & 50.0$^{\pm 0.0}$ & 49.7$^{\pm 0.0}$ \\
FGW & \textbf{72.0}$^{\pm3.7}$ & 68.5$^{\pm3.1}$ & 65.3$^{\pm4.4}$ & 65.8$^{\pm4.4}$ & \textbf{97.7}$^{\pm4.0}$ & \textbf{95.7}$^{\pm2.9}$ & 48.7$^{\pm1.6}$ & 49.3$^{\pm0.7}$ \\
FUGW & 70.6$^{\pm3.7}$ & \textbf{69.7}$^{\pm2.7}$ & \textbf{69.0}$^{\pm3.3}$ & \textbf{68.7}$^{\pm2.4}$& 60.0$^{\pm6.1}$ & 48.3$^{\pm6.7}$ & 45.3$^{\pm4.5}$ & 45.3$^{\pm4.5}$ \\
FPGW (ours) & \textbf{72.0}$^{\pm3.7}$ & \textbf{71.6}$^{\pm4.0}$ & \textbf{71.2}$^{\pm4.9}$ & \textbf{69.2}$^{\pm5.5}$ & \textbf{97.7}$^{\pm4.0}$ & \textbf{96.3}$^{\pm3.1}$ & \textbf{97.7}$^{\pm4.0}$ & \textbf{94.3}$^{\pm4.5}$ \\
\bottomrule
\end{tabular}

\label{tab:kernel-accuracy}
\end{table*}

\textbf{Performance Analysis.}  
The accuracy comparison is summarized in Table~\ref{tab:kernel-accuracy}. For discrete feature graph datasets, VH, FGW, and FPGW achieve the highest overall accuracy. However, as the noise level $\eta$ increases, FGW's performance noticeably deteriorates, while the other three methods remain robust against outlier corruption.

In contrast, for continuous datasets such as \textit{Proteins} and \textit{Synthetic}, the baseline methods, including ``Propagation,'' ``GraphHopper,'' and ``FGW,'' experience a significant drop in performance. Notably, FPGW maintains strong performance across these datasets.

Regarding FUGW, its performance on the Mutag/Proteins datasets is comparable to that of FPGW. However, on the SCRC/Synthetic datasets, FUGW's accuracy is significantly lower than that of FGW and FPGW.
     

We refer to Section  \ref{sec:graph_classification_2} for the parameter settings and wall-clock time comparison. In summary, focusing on the comparison between FGW, FUGW, and FPGW, FGW is slightly faster than FPGW, while both FGW and FPGW are significantly faster than FUGW.

\begin{table}[htbp]
\caption{Kernel computation times (minutes) across different datasets and noise levels}
\label{tab:classification_run_time}
\renewcommand{\tabcolsep}{1.2em}
\begin{tabular*}{\textwidth}{@{\extracolsep{\fill}}lrrrr@{}}
\toprule
\textbf{Method} & \textbf{0\%} & \textbf{10\%} & \textbf{20\%} & \textbf{30\%} \\
\midrule
\multicolumn{5}{l}{\textbf{SYNTHETIC}} \\
Propagation & 0.01 & 0.02 & 0.01 & 0.02 \\
GraphHopper & 5.69 & 5.11 & 5.32 & 8.31 \\
\cmidrule(lr){1-5}
RGW & 4.3 & 4.5 & 4.7 & 5.1 \\
FGW & 7.20 & 13.23 & 16.80 & 22.81 \\
FUGW & 41.00 & 40.53 & 35.32 & 37.88 \\

FPGW & 13.23 & 14.66 & 17.03 & 19.44 \\
\midrule
\multicolumn{5}{l}{\textbf{PROTEINS}} \\
Propagation & 0.02 & 0.02 & 0.02 & 0.03 \\
GraphHopper & 7.06 & 7.43 & 7.79 & 8.16 \\
\cmidrule(lr){1-5}
RGW & 844.7 & 804.1 & 766.8 & 713.2 \\
FGW & 26.05 & 28.71 & 29.83 & 30.17 \\
FUGW & 96.78 & 186.93 & 204.20 & 220.96 \\
FPGW & 62.03 & 65.99 & 71.54 & 73.19 \\
\midrule
\multicolumn{5}{l}{\textbf{MSRC}} \\
WLK (auto) & 0.00 & 0.00 & 0.00 & 0.00 \\
GK (k=3) & 0.04 & 0.04 & 0.04 & 0.04 \\
RWK & 70.79 & 71.01 & 100.29 & 90.90 \\
Odd Sth (k=3) & 0.05 & 0.05 & 0.05 & 0.05 \\
Vertex Histogram & 0.00 & 0.00 & 0.00 & 0.00 \\
Lovasz Theta & 151.14 & 313.54 & 298.87 & 228.57 \\
SVM Theta & 0.00 & 0.00 & 0.00 & 0.00 \\
\cmidrule(lr){1-5}
RGW & 3.0 & 3.2 & 3.4 & 3.6 \\
FGW & 4.03 & 4.31 & 4.63 & 5.05 \\
FUGW & 44.09 & 44.16 & 41.10 & 42.30  \\
FPGW & 4.54 & 4.95 & 5.17 & 5.77 \\
\midrule
\multicolumn{5}{l}{\textbf{MUTAG}} \\
WLK (auto) & 0.00 & 0.00 & 0.00 & 0.00 \\
GK (k=3) & 0.03 & 0.03 & 0.03 & 0.03 \\
RWK & 211.48 & 61.07 & 66.10 & 305.90 \\
Odd Sth (k=3) & 0.00 & 0.00 & 0.00 & 0.00 \\
Vertex Histogram & 0.00 & 0.00 & 0.00 & 0.00 \\
Lovasz Theta & 22.48 & 19.14 & 18.67 & 13.82 \\
SVM Theta & 0.00 & 0.00 & 0.00 & 0.00 \\
\cmidrule(lr){1-5}
RGW & 32.6 & 31.2 & 30.2 & 27.5 \\
FGW & 0.79 & 0.83 & 0.87 & 0.91 \\
FUGW & 19.70 & 21.76 & 20.98 & 17.69 \\

FPGW & 1.09 & 1.54 & 3.62 & 9.21 \\
\bottomrule
\end{tabular*}
\end{table}

\subsection{Numerical details in graph classification}\label{sec:graph_classification_2}

\textbf{Parameter and Numerical Settings.}  

The parameter settings are provided in Table \eqref{tab:parameter}. For methods not explicitly listed, we use the default values from the \href{https://ysig.github.io/GraKeL/0.1a8/classes.html}{GraKeL library}. 

For the FUGW method, we adapt the solver from \cite{flamary2021pot}. We test different marginal penalty parameters $\rho$ to ensure that the transported mass in each sampled pair is approximately 1. Additionally, we choose the smallest entropy regularization term $\epsilon$ that prevents NaN errors.

\begin{table*}[h!]
\captionsetup{skip=2pt}
\caption{Parameter setting of all methods in graph classification. We present the parameter setting in all the methods. $\sigma$ is the weight parameter in the SVM classifier. In FGW and FPGW, $\alpha$ parameter is the $\omega_2$ in formulations of FGW \eqref{eq:fgw} and FPGW \eqref{eq:fmpgw}. For WLK/FGW/FUGW/FPGW, the value $H$ is the Weisfeiler-Lehman labeling parameter. $\rho,\epsilon$ in $FUGW$ is the weight parameter for the marginal penalty and entropy regularization. 
Parameter $\kappa$ in the GK/ODD method is the graphlet size. ``n-samples'' in GK is the random draw sample size.}
\label{tab:parameter}
\begin{center}
\begin{small}
\begin{sc}
\begin{tabular}{lccccccc} 
\toprule
Data set &parameter&{Synthetic}&{Proteins}& {Mutag} &MSRC-9\\
\midrule 

SVM (Common) & $\sigma$ & 1 & 15 & 2 & 1\\
\hline
\multirow{2}{*}{RGW} &$\rho$ & $0.1$ &$0.1$ & $0.1$ & $0.1$  \\
&$\epsilon$ & $0.05$ & $0.05$ & $0.05$ & $0.05$\\
&t & $0.1$ & $0.1$ & $0.1$ & $0.1$\\
&$\tau$ & $0.1$ & $0.1$ & $0.1$ & $0.1$\\
\hline
\multirow{2}{*}{FGW} &$\alpha$ & $0.5$ &$0.5$ & $0.5$ & $0.5$  \\
&H & $-$ & $-$ & $2$ & $4$\\
\hline
\multirow{2}{*}{FPGW} & $\alpha$ & $0.5$ & $0.5$ & $0.5$ & $0.5$   \\
& H & $-$ & $-$ & $2$ & $4$\\
\hline
\multirow{3}{*}{FUGW} & $\alpha$ & $0.5$ & $0.5$ & $0.5$ & $0.5$   \\
& H & $-$ & $-$ & $2$ & $4$\\
& $\rho$ & $1$ & $1$ & $0.4$ & $0.4$\\
& $\epsilon$ & 0.05 & 0.1 & 0.02&0.02\\
\hline
WLK & H & $-$ & $-$ & $5$ & $5$ \\
\hline
\multirow{2}{*}{GK} & $\kappa$ & $-$ & $-$ & $3$ & $3$ \\
                   & n-samples &  $-$ & $-$ & $100$ & $100$ \\
\hline
Odd Sth & $\kappa$ & $-$ & $-$ & $3$ & $3$ \\
\bottomrule
\end{tabular}
\end{sc}
\end{small}
\end{center}
\end{table*}

\textbf{Wall-clock time analysis.}  

The wall-clock times are reported in Table \ref{tab:classification_run_time}. All graph data are formatted as NetworkX graphs using the \href{https://networkx.org/}{NetworkX library}. Continuous features are represented as 64-bit float NumPy vectors, while discrete features are stored as 64-bit integers.

For discrete feature graphs, WLK, GK, and Vertex Histogram are the fastest methods, while FGW and FPGW have similar wall-clock times. In contrast, Lovász Theta and Random Walk Kernel are the slowest. For continuous feature graphs, ``Propagation'' and ``GraphHopper'' are significantly faster than FGW and FPGW. 

Among the three fused-GW-based methods, FGW is the fastest overall, while FUGW is the slowest. A potential reason for this difference is that FGW and FPGW leverage a C++ linear programming solver for the OT/POT solving step. In addition, FUGW adapts the Sinkhorn solver, leading to an accuracy-efficiency trade-off. Specifically, a smaller $\epsilon$ can reduce the accuracy gap introduced by the entropic term; however, it increases the number of iterations required for convergence.

All experiments presented in this paper are conducted on a computational machine with an AMD EPYC 7713 64-Core Processor, 8 $\times$ 32GB DIMM DDR4, 3200 MHz, and an NVIDIA RTX A6000 GPU.

\section*{Impact Statement}
This paper aims to advance the theoretical foundations and potential applications of Optimal Transport in the field of Machine Learning.  There are many potential societal consequences 
of our work, none of which we feel must be specifically highlighted here.

\end{document}